\documentclass[10 pt, journal, twoside, twocolumn]{IEEEtran}
\IEEEoverridecommandlockouts                              

\usepackage{times} 

\usepackage{graphics} 
\usepackage{epsfig} 
\usepackage{subcaption} 
\usepackage{listings}

\usepackage{mathptmx} 
\usepackage{amsmath} 
\usepackage{amssymb}  
\usepackage{amsthm}
\usepackage{amsfonts}
\usepackage{bm} 

\usepackage{algorithm}
\usepackage{algorithmic}

\usepackage{siunitx}
\usepackage[normalem]{ulem}

\usepackage{xcolor}
\usepackage{svg}
\usepackage{hyperref}

\usepackage[font=small,labelfont=bf]{caption} 

\usepackage[
backend=biber,
style=ieee,
sorting=none,
]{biblatex}
\title{
  Avoidance of Concave Obstacles \\
  through Rotation of Nonlinear Dynamics
}
\author{Lukas Huber$^{1}$, Jean-Jacques Slotine$^2$,  Aude Billard$^1$ \vspace{-2.3ex}
\thanks{This work was supported by EU ERC grant SAHR.} 
\thanks{$^{1}$ LASA Laboratory, Swiss Federal School of Technology in Lausanne - EPFL, Switzerland. \tt $\{$lukas.huber;aude.billard$\}$@epfl.ch }
\thanks{$^{3}$ Nonlinear Systems Laboratory,  Massachusetts Institute of Technology, USA. \tt jjs@mit.edu}   
}

\newtheorem{theorem}{Theorem}[section]

\newtheorem{lemma}[theorem]{Lemma}

\theoremstyle{definition}
\newtheorem{definition}{Definition}[section]

\begin{document}
\newcommand{\vect}[1]{\bm{#1}}
\newcommand{\vecs}[1]{\bm{#1}}
\newcommand{\matr}[1]{\bm{#1}}

\newcommand{\angspace}[2]{\bm{k} \left( #1, \, #2 \right) }

\newcommand{\angs}[1]{\hat{\mathbf{#1}}}

\newcommand{\dotprod}[2]{\left\langle {#1}, \, {#2} \right\rangle}
\newcommand{\normdotprod}[2]{\frac{\left\langle #1, \, #2 \right\rangle}{\| #1 \| \, \| #2 \|}}

\def\sminus{-}
\def\const{\text{const.}}
\def\ker{\text{ker}}

\maketitle
\begin{abstract}
Controlling complex tasks in robotic systems, such as circular motion for cleaning or following curvy lines, can be dealt with using nonlinear vector fields. 
In this paper, we introduce a novel approach called rotational obstacle avoidance method (ROAM) for adapting the initial dynamics when the workspace is partially occluded by obstacles.
ROAM presents a closed-form solution that effectively avoids star-shaped obstacles in spaces of arbitrary dimensions by rotating the initial dynamics towards the tangent space.
The algorithm enables navigation within obstacle hulls and can be customized to actively move away from surfaces, while guaranteeing the presence of only a single saddle point on the boundary of each obstacle. 
We introduce a sequence of mappings to extend the approach for general nonlinear dynamics. 
Moreover, ROAM extends its capabilities to handle multi-obstacle environments and provides the ability to constrain dynamics within a safe tube. 
By utilizing weighted vector-tree summation, we successfully navigate around general concave obstacles represented as a tree-of-stars. 
Through experimental evaluation, ROAM demonstrates superior performance in terms of minimizing occurrences of local minima and maintaining similarity to the initial dynamics, outperforming existing approaches in multi-obstacle simulations. 
The proposed method is highly reactive, owing to its simplicity, and can be applied effectively in dynamic environments. 
This was demonstrated during the collision-free navigation of a 7 degree-of-freedom robot arm around dynamic obstacles.
\end{abstract}


\section{Introduction}  \label{sec:introduction}
Reactive motion plays a crucial role in numerous real-world robotics applications. When operating outside the controlled environments of factory floors, robots are exposed to unpredictable and dynamic surroundings, making precise estimation challenging. As a result, real-time adaptive controllers are essential to enable robots to adapt and reevaluate their actions in response to changing conditions.

A primary constraint when navigating in dynamic and cluttered environments is to ensure the safety of individuals moving around the robot. This requires the robot to constantly and rapidly replan its path to avoid collisions while making every effort to maintain its intended task and adhere to the originally intended movement dynamics. Furthermore, to protect physical hardware from potential damage caused by high accelerations, it is crucial to design a smooth system. Additionally, collision avoidance needs to seamlessly integrate with reactive control techniques (Figure~\ref{fig:qolo_in_maze}).

\begin{figure}[t]
  \centering
    \begin{subfigure}{.8\columnwidth}
    \centering
    \includegraphics[width=\textwidth]{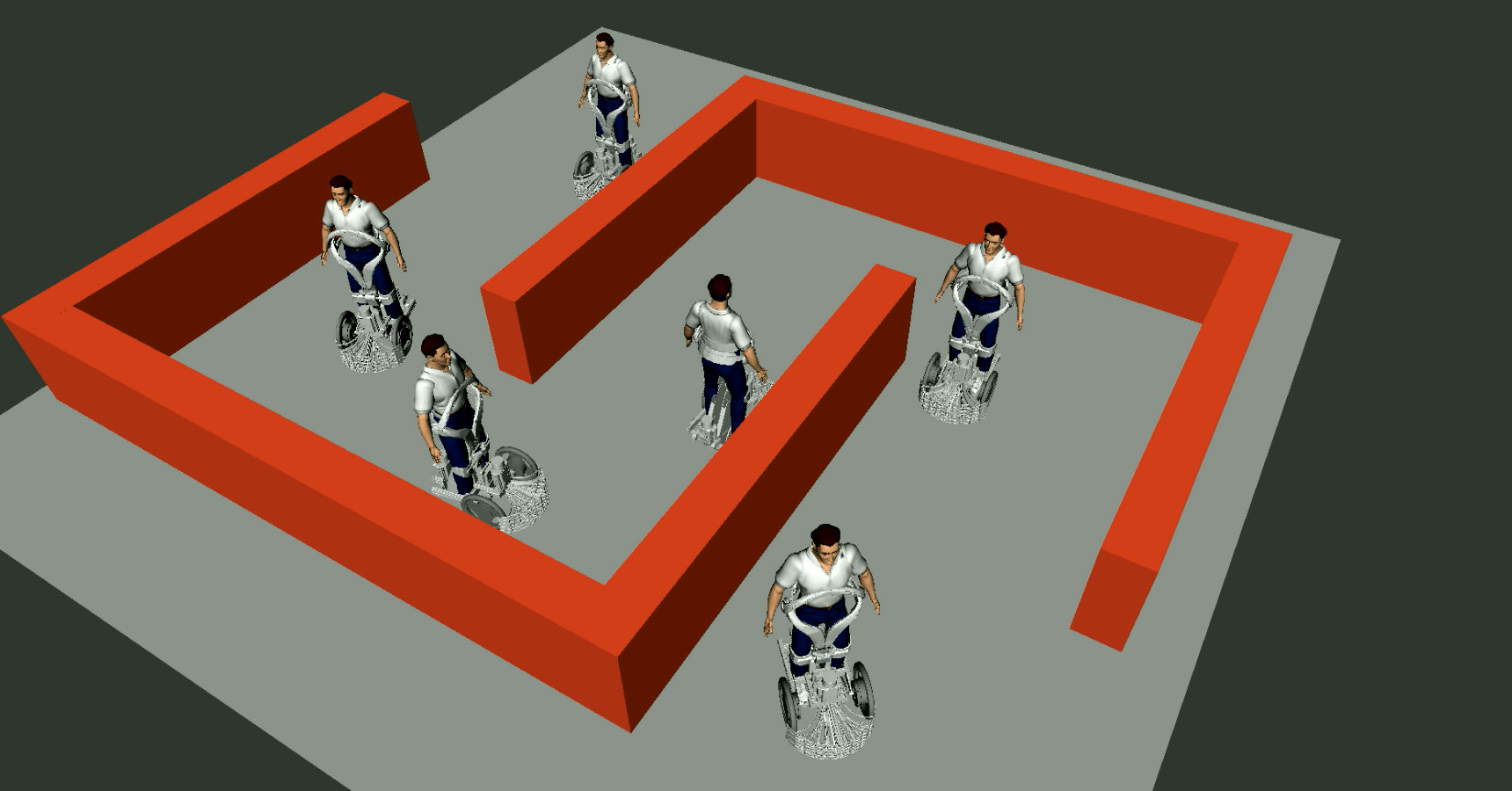}
    \caption{QOLO-robot navigating within small labyrinth}
    \label{fig:qolo_in_maze}
  \end{subfigure}
  \begin{subfigure}{.48\columnwidth}
    \centering
    \includegraphics[width=\textwidth]{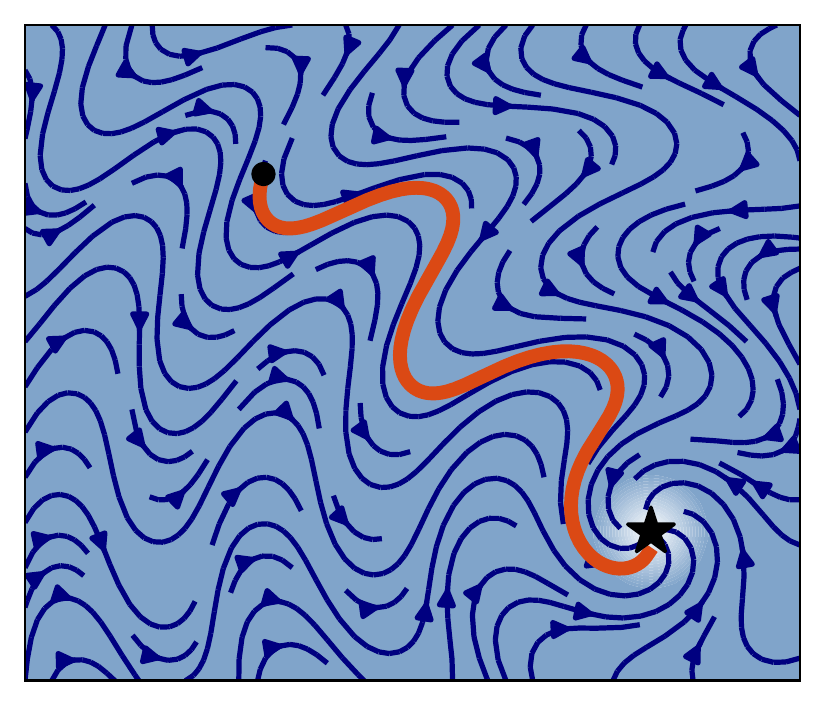}
    \caption{Initial dynamics}
    \label{fig:nonlinear_vectorfield_rotated}
  \end{subfigure}\hfill%
  \begin{subfigure}{.48\columnwidth}
    \centering
    \includegraphics[width=\textwidth]{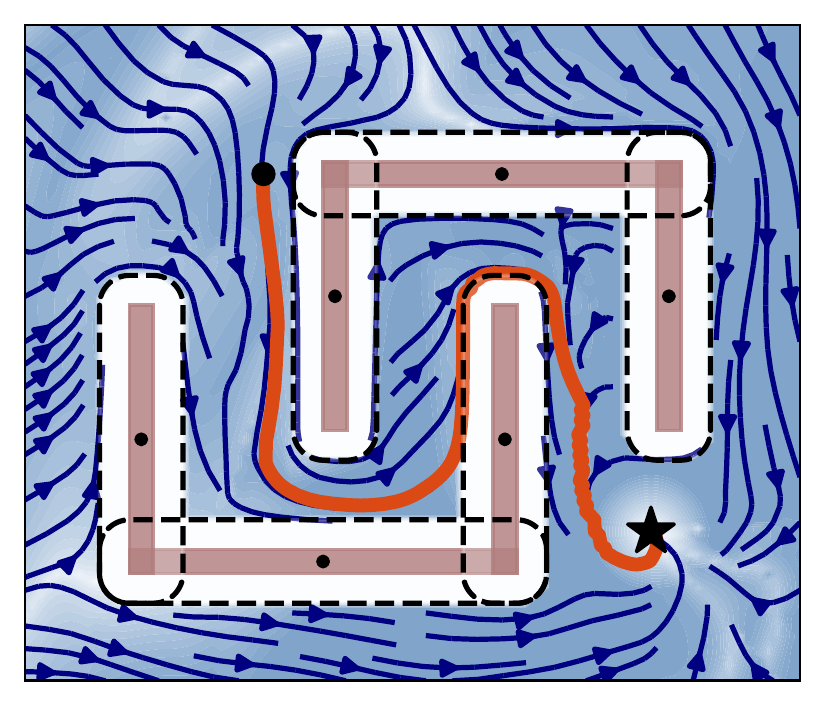}
    \caption{Rotated dynamics}
    \label{fig:nonlinear_vectorfield_rotated}
  \end{subfigure}
  \centering
  \caption{An autonomous wheelchair is guided by obstacle avoidance to navigate in an environment of complex obstacles (c). The intensity of the shading in (a) and (b) indicates the magnitudes of the velocity.}
    \label{fig:qolo_in_maze}
\end{figure}

Control methods based on vector fields \cite{goncalves2010vector} and dynamical systems \cite{huber2019avoidance} have proven to be well-suited for addressing the challenges posed by such scenarios. Rather than precomputing a trajectory for the robot, these methods generate a (nonlinear) control field that is evaluated in real-time at the robot's position. This allows the robot to react instantaneously to disturbances and perceive changes in its environment. In this work, we leverage the framework of dynamical systems and vector fields to develop an adaptive obstacle avoidance approach capable of modifying nonlinear motion in dynamic and complex environments (Fig. \ref{fig:qolo_in_maze}).


 Dynamical systems represent the evolution of a system's state without considering its history, effectively capturing the system's dynamics as a vector field. In robotics, dynamical systems have been employed for learning complex dynamics \cite{figueroa2018physically} and enforcing stability guarantees through techniques such as Lyapunov Stability \cite{slotine1991applied} or Contraction Theory \cite{huber2019avoidance}. Force control in robotics often utilizes second-order dynamical systems \cite{salehian2018dynamical}, while obstacle avoidance typically relies on first-order systems that output desired velocities based on the current position \cite{khansari2012dynamical}. Consequently, additional controllers are employed to ensure force- and torque-controlled robots follow the desired motion of the dynamical system \cite{kronander2015passive}.

Analogously, desired dynamics based on position can be represented as a vector field \cite{goncalves2010vector}. These nonlinear vector fields can be designed to ensure convergence and stability properties. Nonlinear vector fields designed for path following, known as "vector-field-guided path following," enable smooth convergence towards a desired path, reducing path-following errors \cite{kapitanyuk2017guiding}.

In this paper, we develop an adaptive obstacle avoidance framework capable of modifying nonlinear motion in response to dynamic and complex environments. The proposed approach aims to navigate the control system safely while maintaining the integrity of the intended motion dynamics. The effectiveness of our approach is demonstrated through experiments.

\subsection{Related Work} \label{sec:related_work}
Obstacle avoidance methods are often categorized into local and global strategies \cite{lynch2017modern}. Global methods focus on convergence to the goal, while local methods prioritize reactive collision avoidance in dynamic scenarios.

Global sampling-based methods, such as Rapidly-Exploring Random Trees (RRT), have been widely utilized to populate the search space by branching out a space-filling tree in order to find a feasible path \cite{lavalle1998rapidly}. Building upon RRT, the RRT* algorithm was introduced to automatically scale the sampling length during the search, resulting in improved performance and convergence rate \cite{karaman2011sampling}. Similarly, Probabilistic Roadmaps (PRM) employ space sampling to determine the collision-free status of points. By connecting the collision-free samples, local connectivity graphs are formed, enabling graph-search algorithms to find an optimal path from the initial position to the goal \cite{kavraki1996probabilistic}.
Although PRM and RRT are probabilistic-complete, meaning they can find a feasible path if one exists \cite{lavalle2004relationship}, the sampling-based approach incurs a significant computational cost. This renders it unsuitable for real-time recomputation, which is essential for dynamic environments. Sampling algorithms like RRT and PRM exhibit a complexity that is linear to the number of nodes $n$ in the sampled tree, represented as $\mathcal{O}(n) = \mathcal{O}(exp(N))$ \cite{noreen2016optimal}. Even in two-dimensional environments, the trees need to cover the entire workspace, leading to a large number of nodes.
To address the computational challenges and facilitate fast computation in dynamic environments, initial paths are often reshaped. One approach involves interpreting trajectories as \textit{elastic strips}, enabling efficient adaptation and obstacle avoidance \cite{brock2002elastic}. However, this comes at the loss of convergence guarantees.


Motion Optimization (MO) has been employed to dynamically adapt the originally sampled trajectories, analogous to stretching an elastic band around an object \cite{quinlan1993elastic}. However, MO tends to perform poorly in highly non-convex problems and often converges to suboptimal local minima. To alleviate this issue, Markov chains have been utilized for joint-space control \cite{zucker2013chomp}, while other approaches incorporate Riemann geometry and geometric constraints to improve convergence heuristics \cite{ratliff2015understanding, mainprice2020interior}. Nonetheless, MO remains local in nature and often fails to converge when dealing with non-convex motions.


Model Predictive Control (MPC) reduces the optimization problem to a finite time horizon, thereby reducing the convergence time \cite{sanchez2021nonlinear}. The increased computational power in recent years has facilitated the utilization of MPC for real-time collision avoidance \cite{williams2017model}. Sampling-based MPC has been employed for dynamic configuration-space collision avoidance in scenarios where an analytic cost function is absent \cite{bhardwaj2022storm, koptev2022neural}. However, the limited time horizon of MPC often limits the problem to local optimality, and global convergence cannot always be guaranteed. 


In recent years, the application of Machine Learning (ML) in collision avoidance for robotics has gained significant traction. End-to-end learning approaches have enabled collision-free multi-robot navigation \cite{riviere2020glas}, while fuzzy artificial potential fields augmented with neural networks have been utilized for motion planning of mobile robots \cite{wang2021path}. Conversely, RRT has served as a foundation for training reinforcement learning algorithms \cite{cai2023overcoming}. While learning-based methods demonstrate impressive performance, they often lack collision avoidance guarantees and cannot ensure convergence to a reliable solution.
Deep learning methods, such as multi-layer neural networks, incur a computational cost that scales with the number of parameters $p$, denoted as $\mathcal{O}(p)$ \cite{lecun2015deep}. Given that modern neural networks frequently comprise millions of parameters, the inference process becomes the most time-consuming component of the control loop. In numerous robotic implementations, the vision-based neural network constitutes the most time-intensive aspect of the control loop \cite{huber2022fast}. However, employing neural networks to output control commands introduces substantial computation time to the control loop \cite{koptev2022neural, cai2023overcoming}.

While most global methods lack the ability to adapt to dynamic environments, this limitation can be overcome by employing local methods. Many local methods can be interpreted as nonlinear feedback controllers.

One of the early collision avoidance methods in robotics is Artificial Potential Fields (APF). APF represents each obstacle as a repulsive field through which the agent navigates \cite{khatib1986real}. The reactive nature of APF makes it suitable for dynamic environments, such as tracking moving targets \cite{huang2009velocity}. APF has also been integrated with closed-feedback loops for real-time collision avoidance of robotic manipulators \cite{tulbure2020closing}. While APF is efficient and easy to implement, it is prone to local minima in multi-obstacle environments \cite{koren1991potential}.

Navigation functions (NFs) were developed to address the issue of multiple local minima by constructing a potential field with a single minimum \cite{koditschek1990robot}. While initially designed for sphere worlds, NFs utilize diffeomorphic mappings between \textit{star-worlds} and \textit{sphere-worlds} to extend their applicability to more general scenarios \cite{rimon1991construction}. These functions can also be employed to navigate around inverted obstacles that represent space boundaries, and their scope has been further expanded to include navigation around \textit{trees-of-stars} \cite{rimon1992exact}. However, achieving optimal performance with NFs requires comprehensive knowledge of the obstacle distribution, making the tuning of critical parameters a challenging task.
Efforts have been made to enhance NFs by eliminating critical parameters in sphere worlds \cite{paternain2017navigation}, and alternative approaches have been proposed to automate the tuning process \cite{loizou2017navigation}. NFs offer the advantage of quick re-evaluation, making them robust against disturbances. Nonetheless, due to the inherent difficulty in tuning them for general obstacle configurations, NFs pose challenges when applied in dynamic environments. Furthermore, as NFs rely on the gradient of the potential function to guide the system dynamics, they are unable to effectively track arbitrary nonlinear dynamics.

Vector fields (VF) employ highly nonlinear vector fields to navigate paths while ensuring collision avoidance. They have been successfully applied to evade multiple circular obstacles in two-dimensional space using local repulsive fields \cite{panagou2014motion}. VF techniques have also been utilized for controlling fixed-wing aircraft \cite{wilhelm2019vector} and extended to accommodate initially nonlinear dynamics, such as following a limit cycle \cite{yao2022guiding}. Nonetheless, it involves switching between different dynamics, which can result in high accelerations.
However, VF methods typically handle one obstacle at a time, imposing a conservative constraint where the influence regions of obstacles cannot overlap.

Harmonic potential functions (HPF) are intriguing because they guarantee the absence of topologically critical points in free space. While analytical HPFs are often elusive, numerical approximations have been employed \cite{connolly1990path}. Linear panel representations for obstacles enable the generation of closed-form HPFs \cite{kim1992real}. Although this linear approximation can be extended to concave obstacles, its application is limited to two-dimensional environments \cite{feder1997real}.

Dynamical system modulation (DSM) has emerged as an effective approach for collision avoidance in reactive environments by redirecting initial dynamics away from obstacles \cite{khansari2012dynamical}. To overcome challenges posed by intersecting, convex obstacles in three dimensions, DSM was extended to incorporate switching to surface following \cite{zheng2020dynamical}. In our previous research, we successfully imitated the behavior of harmonic potential functions to ensure the convergence of DSM around concave (star-shaped) obstacles \cite{huber2019avoidance}. Furthermore, we expanded the scope of DSM to encompass indoor environments, facilitating reactive avoidance based on sensor data \cite{huber2022avoiding, huber2022fast}. While DSM traditionally assumes a zero-dimensional point, we achieved collision-free rigid-body dynamics by representing mobile agents through multiple control points \cite{conzelmann2022dynamical}. However, it is important to note that DSM utilizes straight dynamics towards an attractor as an input, and the convergence guarantees are presently limited to star-shaped obstacles.

\subsection{Contributions}
This work significantly expands upon the reviewed DSM approach, enabling obstacle avoidance with nonlinear dynamical systems and for general concave obstacles. The paper's tehcnical contributions are outlined as follows:

\begin{itemize}
\item We introduce the Rotational Obstacle Avoidance Method (ROAM), which ensures local minima-free obstacle avoidance for nonlinear dynamics (Section~\ref{sec:rotational_avoidance}).
\item We present a diffeomorphic mapping that guarantees convergence while maintaining proximity to the general nonlinear dynamics, while trying to preserve the initial flow (Section~\ref{sec:convergence_dynamics}).
\item We extend the ROAM formulation to handle obstacle avoidance of multiple dynamic obstacles and enclosed spaces, utilizing the concept of inverted obstacles as introduced in \cite{koditschek1990robot, huber2022avoiding} (Section~\ref{sec:multi_obstacle_environment}).
\item Finally, we demonstrate the application of ROAM for avoiding general obstacles represented by trees-of-stars in the presence of nonlinear dynamics (Section~\ref{sec:concave_obstacles}).
\item Additionally, we develop a method for vector rotation in general dimensions and the weighted summing of trees of vector rotations (Appendix~\ref{sec:perpendicular_rotation}).
\end{itemize}

To validate the properties of ROAM, we conducted following evaluations:
\begin{itemize}
   \item We performed a quantitative comparison of ROAM against two recent analytical obstacle avoidance methods through simulations, assessing convergence rate, similarity to unperturbed motion, and acceleration along the trajectories (Section~\ref{sec:results}).
   \item Furthermore, we conducted a qualitative evaluation of ROAM's application in controlling a 7DoF robot arm to avoid trees-of-stars in three dimensions (Section~\ref{sec:robot_implementation}).
\end{itemize}

The proposed rotation obstacle avoidance method (ROAM) ensures collision avoidance in a local minima-free vector field in the presence of initial nonlinear dynamics. Moreover, the influence regions of the obstacle can overlap, and the method has been extended to inverted obstacles, too, as well as multiple obstacles with overlapping regions of influence, as can be seen in the comparison with similar algorithms in Table~\ref{tab:comparison_convergence}.

\begin{table}[tbh]
    \centering
	\scalebox{0.93}{
    \begin{tabular}{|l|c|c|c|c|c|} \hline 
      & RF & NF & MuMo  & VF-CAPF  & ROAM \\ 
       & \cite{khatib1986real} & \cite{koditschek1990robot} & \cite{huber2022avoiding} & \cite{yao2022guiding} &  \\ \hline \hline
      Local minima free & & \checkmark & \checkmark & (\checkmark) & \checkmark \\ \hline
      Switching free &\checkmark & \checkmark & \checkmark  &  & \checkmark\\ \hline
      Overlapping regions & \checkmark & \checkmark  & \checkmark  &  & \checkmark \\ \hline
      Nonlinear dynamics & \checkmark & & (\checkmark) & \checkmark & \checkmark \\ \hline
      Dynamic environments & \checkmark &  & \checkmark & \checkmark & \checkmark \\ \hline
      (Optional) repulsion & \checkmark &  &  &  & \checkmark \\ \hline
      Inverted obstacle & \checkmark & \checkmark & \checkmark &  & \checkmark \\ \hline
      Trees of Obstacles & \checkmark & \checkmark &  &  & \checkmark \\ \hline
    \end{tabular}
}
    \caption{
    Comparison of five different time-invariant algorithms across multiple criteria.
    All algorithms avoid solving an optimization problem.
    }
    \label{tab:comparison_convergence}
\end{table}

\section{Preliminairies} \label{sec:preliminairies}

\subsection{Notations}
In this section, we establish the notations used throughout this paper.
We consider a space of general dimension $N$ as the underlying framework for our developments. Vectors are denoted using bold symbols, and the state variable is represented by $\vecs \xi \in \mathbb{R}^N$. The time derivative of a vector $\vecs \xi$ is denoted as $\dot{\vecs{\xi}}$, assuming $\vecs \xi$ is a function that is differentiable with respect to time.
Unless explicitly stated otherwise, the paper employs the Euclidean norm, denoted as $|\cdot|$. The circular constant is denoted by $\pi$. The matrix $\matr I$ represents the identity matrix of appropriate dimensions.
The symbol $\circ$ signifies the iterative evaluation of a function. For example, when considering two functions $a(\cdot)$ and $b(\cdot)$, the expression $a \circ b(\vecs \xi)$ denotes the composition $a(b(\vecs \xi))$.
In the context of this paper, the term \textit{vector field} (VF) refers to the velocity field $\dot{\vecs \xi}$, while \textit{dynamical system} (DS) specifically denotes the continuous-time feedback system described in Section~\ref{sec:prelimniairies_ds}.

\subsection{Dynamical Systems} \label{sec:prelimniairies_ds}
Consider a vector field that governs the evolution of the position $\vecs{\xi} \in \mathbb{R}^N$ of a continuous-time system, described by the state derivative $\dot{\vecs{\xi}} \in \mathbb{R}^N$:
\begin{equation}
\dot{\vecs{\xi}} = \vecs{f}(\vecs{\xi}) \label{eq:ideal_initial_dynamics}
\end{equation}

Let us define {\em straight dynamics}, that do maintain the direction along a trajectory:
\begin{definition}[Straight Dynamics] \label{def:straight_dynamics}
A dynamical system of the form \eqref{eq:ideal_initial_dynamics} is referred to as \textit{straight dynamics} if, for all initial conditions $\vecs{\xi}_0$ and $\dot{\vecs{\xi}}_0$, the flow remains collinear with the initial velocity, satisfying $\dot{\vecs{\xi}}_0^T \dot{\vecs{\xi}}_t = \|\dot{\vecs{\xi}}_0\| \|\dot{\vecs{\xi}}_t\|$, for all $t \geq 0$. The flow is said to be \textit{locally straight} if it is straight within a subdomain $\mathcal{X}^s$. In the special case where the attractor is positioned infinitely far away, hence $\exists \vect v^a \in \mathbb{R}^N : \lim_{ \|\vecs{\xi} - \vecs{\xi}^a \| \rightarrow \infty} \dot{\vecs{\xi}}_t^T \vect v^a = \| \dot{\vecs{\xi}}_t^T \|$, hence the dynamics are called \textit{(locally) collinear}.
\end{definition}

Multiple straight dynamics are visualized in Figure~\ref{fig:straight_dynamcis}.
In the subsequent sections, we adopt the following parameterization for globally straight dynamics:
\begin{equation}
\vecs{f}^s(\vecs{\xi}) = q(\vecs{\xi}) \left(\vecs{\xi}^a - \vecs{\xi}\right)
\label{eq:straight_system}
\end{equation}
where $q: \mathbb{R}^N \rightarrow \mathbb{R} \setminus 0$ represents a continuous state-dependent scaling function that modulates the speed of the linear system $(\vecs{\xi}^a - \vecs{\xi})$ towards a fixed point $\vecs{\xi}^a$. This scaling allows for the adjustment of the speed of the linear dynamics as the system approaches or moves away from the attractor $\vecs{\xi}^a \in \mathbb{R}^N$, which serves as the unique stable fixed point or \textit{attractor} of the system (Fig.~\ref{fig:globally_straight}). Importantly, the condition $q(\vecs{\xi}) \neq 0$, combined with the continuity of $q$, preserves the directionality of the flow. Consequently, the vector field consistently points towards or away from the attractor across the state space.


\begin{figure}[tbh]\centering
\hfill
\begin{subfigure}{.33\columnwidth}
\centering
\includegraphics[width=\textwidth]{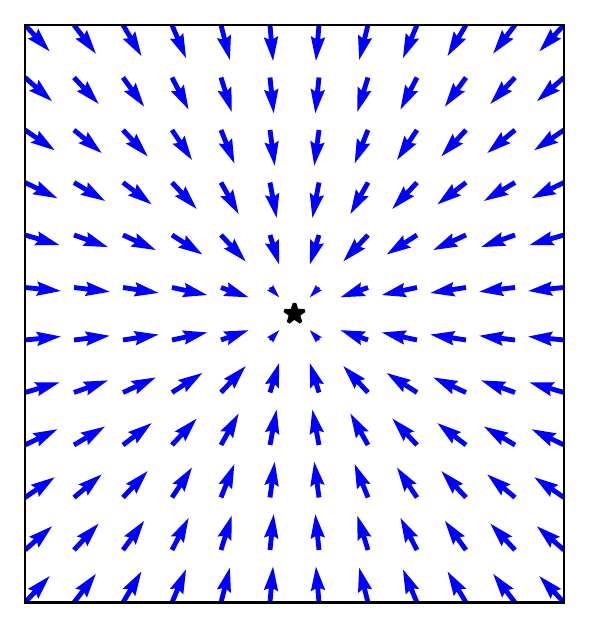}
\caption{Globally straight}
\label{fig:globally_straight}
\end{subfigure}\hfill%
\begin{subfigure}{.33\columnwidth}
\centering
\includegraphics[width=\textwidth]{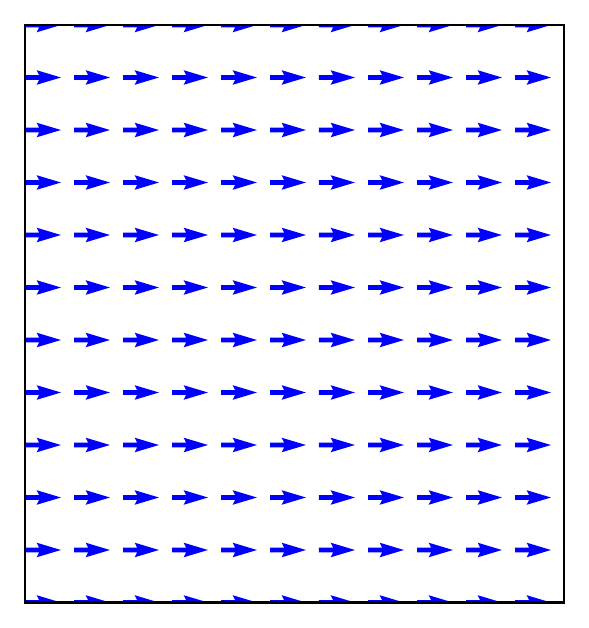}
\caption{Globally collinear}
\label{fig:globally_collinear}
\end{subfigure}\hfill%
\begin{subfigure}{.33\columnwidth}
\centering
\includegraphics[width=\textwidth]{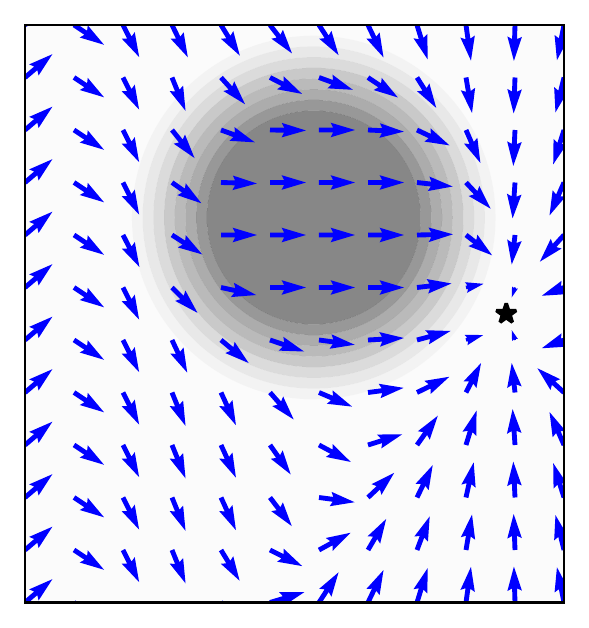}
\caption{Locally straight}
\label{fig:locally_straight}
\end{subfigure}\hfill%
\hfill
\caption{
Dynamical systems can exhibit different straightness characteristics.
Such as globally straight with a single, stable attractor $\vecs \xi^a$ visualized as a star (a). 
They can be globally collinear when the attractor is infinitely far away (b). 
Conversely, dynamics might be globally defined but only locally straight in a subdomain, visualized as the dark gray region (c).}
\label{fig:straight_dynamcis}
\end{figure}

\subsection{Obstacle Description}
In accordance with \cite{huber2019avoidance}, each obstacle is characterized by a continuous distance function $\Gamma(\vecs \xi): \mathbb{R}^N  \mapsto \mathbb{R}_{\geq 0}$.
This function enables the differentiation of regions outside, on the boundary, or inside the obstacle based on the isolines of $\Gamma$.
\begin{align}
  &\text{Free space:}&  \qquad & \mathcal{X}^f = \{\vecs \xi \in \mathbb{R}^N: \Gamma(\vecs \xi) > 1 \} \nonumber \\ 
  &\text{Boundary:}&  \qquad & \mathcal{X}^b = \{\vecs \xi \in \mathbb{R}^N: \Gamma(\vecs \xi) = 1 \} \label{eq:levelFunc} \\
  &\text{Interior set::}&  \qquad & \mathcal{X}^i = \{ \vecs \xi \in \mathbb{R}^N \setminus (  \mathcal{X}^f \cup \mathcal{X}^b ) \} \nonumber
\end{align}

\subsubsection{Star-Shaped Obstacles}\label{sec:kernel_space}
For representing star-shaped obstacles, we adopt the notation introduced in \cite{brunn1913kerneigebiete, hansen2020starshaped}. An obstacle is considered star-shaped if there exists a point $\vecs \xi^r$ with $\Gamma(\vecs \xi^r)<1$ such that every point on the boundary or in the interior set of the obstacle, i.e., $\vecs \xi \in \mathcal{X}^i$ and $\mathcal{X}^i:{\vecs \xi \in \mathbb{R}^N, \Gamma(\vecs \xi) \leq 1}$, is connected to $\vecs \xi^r$ by a line segment $l(\vecs \xi^r, \vecs \xi)$ contained within $\mathcal{X}^i$.

The set of all such points $\vecs \xi^r$ is referred to as the kernel of $\mathcal{X}^i$ and denoted as:
\begin{equation}
  \ker(\mathcal{X}^i) = \left\{
  \vecs \xi^r \in \mathcal{X}^i
  : l(\vecs \xi^r, \vecs \xi) \subset \mathcal{X}^i, \;
  \forall \, \vecs \xi \in \mathcal{X}^i \right\}  \label{eq:kernel_space}
\end{equation}
Throughout this paper, a specific choice of $\vecs \xi^r$ is referred to as the {\em reference point}, as commonly used in \cite{huber2019avoidance, huber2022avoiding}, and represented by a cross symbol $+$. \footnote{The reference point is in literature sometimes denoted \textit{kernel point}.}

\subsubsection{Distance Function} \label{sec:distance_function}
The distance function $\Gamma(\vecs \xi)$ increases monotonically in the radial direction, i.e., along the vector $\vecs \xi - \vecs \xi^r$, and has a continuous first-order partial derivative ($C^1$ smoothness). In this paper, we compute a distance function relative to one boundary point $\vecs \xi^b \in \mathcal{X}^{b}$: 
\begin{equation}
  \Gamma( \vecs \xi) = \| \vecs \xi-\vecs \xi^b\| / d_0  + 1 \quad \forall \vecs \xi \in \mathcal{X}^e \label{eq:distance_function}
\end{equation}
where $d_0 \in \mathbb{R}_{>0}$ is a scaling factor, which modulates the obstacle influence. 
The boundary point $\vecs \xi^ b$ lies at the intersection between a ball of radius $R(\vecs \xi)$ around the reference point $\vecs \xi^r$:
\begin{equation}
    R(\vecs \xi) = \| \vecs \xi^b - \vecs \xi^r \| \quad
  \vecs \xi^b = b \vect r (\vecs \xi) + \vecs \xi^r  
  \; , \; b>0 \, , \;\; \vecs \xi^b \in \mathcal{X}^{b} \label{eq:boundary_point}
\end{equation}
and the {\em reference direction} vector $\vect r: \mathbb{R}^N \setminus \vecs \xi^r \rightarrow \{ \vecs r \in \mathbb{R}^N : \|\vecs r \| = 1\}$:
\begin{equation}
{\vect r}(\vecs \xi) = {\left( \vecs \xi - \vecs \xi^r \right)}/{\|\vecs \xi - \vecs \xi^r \|}.   
\label{eq:reference_direction}
\end{equation}
The boundary point is hence a function of the state, $\vecs \xi^ b(\vecs \xi)$. 

Further, the normal direction can be defined as the derivative of the distance across space:
\begin{equation}
  \vect n(\vecs \xi) = d \Gamma(\vecs \xi) / d \vecs \xi \label{eq:normal_direction}
\end{equation}

\subsection{Rotated Vector Field} \label{sec:rotated_vector_field}
We present a construct that is fundamental to the obstacle avoidance approach proposed here, inspired by the concept of stereographic projection \cite{huber2022avoiding}. \footnote{A stereographic projection maps points located on a sphere onto a plane perpendicular to the sphere surface.}  
Consider a unit vector with respect to a basis vector $\vect b$, denoted as $ \vect k(\vect b, \vecs \xi): \mathbb{R}^N_I \times \mathbb{R}^N_I\rightarrow \mathbb{R}^{N-1}$, under the condition that $\vecs \xi:\dotprod{- \vect b}{\vecs \xi} \neq -1$. This vector extraction captures the rotation of $\vecs \xi$ with respect to $\vect b$ and has proven to be useful for minima-free vector-summing. The relative rotation satisfies the following properties:
\begin{equation}
  \cos \left(\| \vect k(\vect b, \vecs \xi) \| \right)  = \dotprod{\vect b}{\vecs \xi}
  \quad \text{and} \quad
  \| \vect k(\vect b, \vecs \xi) \| < \pi
  \label{eq:direction_space_constraints}
\end{equation}

The basis vector $\vect{b}$ is utilized as the first column to construct the orthonormal transformation matrix $\matr B$, enabling the transformation into a new basis: $\hat{\vecs \xi}= \matr B ^T \vecs \xi$.
In the direction space, the magnitude corresponds to the angle between the original vector and the reference vector. The transformation of the initial vector $\vecs \xi$ in the direction space is given by:
\begin{equation}
  \vect k(\vect b, \vecs \xi ) =
  \begin{cases}
    \arccos \left(\hat{ \vecs{\xi}}_{[1]} \right)
    \hat{\vecs \xi}_{[2:]} / \|  \hat{\vecs \xi}_{[2:]} \|  & \text{if} \;\; \hat{ \vecs{\xi}}_{[1]} \neq 1 \\
    \vect 0  & \text{otherwise}
\end{cases}
\label{eq:kappa_trafo}
\end{equation}

Note that a vector $\vecs \xi$, which is anti-collinear to $\vect b$, does not have a unique transformation and is hence excluded; see \cite{huber2022avoiding} for further discussion.
Correspondingly, we use $\bar{\vect k}(\vect b, \bar{\vecs \xi}): \mathbb{R}^N_I \times \mathbb{R}^{N-1} \rightarrow \mathbb{R}^{N}_I$ to describe the inverse mapping, such that:
\begin{equation}
  \vecs \xi = \bar{\vect k}(\vect b, \bar {\vecs \xi})
  \quad \text{with} \quad
  \bar {\vecs \xi} = \vect k(\vect b, \vecs \xi) 
  \label{eq:inverse_angular_mapping}
\end{equation}

The mapping to the original space is evaluated as follows:
\begin{equation}
  \bar{\vect k}(\vect b, \bar{\vecs \xi} ) =
  \begin{cases}
    \matr B \begin{bmatrix} 1 & 0 & .. & 0  \end{bmatrix}^T
    & \text{if} \;  \| \bar{\vecs \xi} \| = 0  \\
    \matr B
    \begin{bmatrix}
      \cos (\| \bar{\vecs \xi}  \| ) &
      \sin (\| \bar{\vecs \xi} \|) \bar{\vecs \xi} / \| \bar{\vecs \xi}\|
    \end{bmatrix}^T
    &  \text{otherwise}
  \end{cases}
\label{eq:kappa_trafo_inv}
\end{equation}%
It is important to note that the reverse mapping $\bar{\vect k}(\cdot)$ is defined for angles larger than $\pi$, but the function is not bijective for these values.

\subsection{Problem Statement} \label{sec:problem_statement}
We establish the following requirements for our obstacle avoidance controller:

\begin{itemize}
\item \textbf{Collision free:}
 The flow must remain outside the obstacles at all times, i.e., $\{\vecs \xi\}_t \in \mathcal{X}^e \; \forall t$ if $\{\vecs \xi\}_0 \in \mathcal{X}^e$. 
  
\item \textbf{State dependent:} The dynamics are history-invariant and depend solely on the current state, i.e. $\dot{\vecs \xi}_{t}=\vect f(\vecs \xi_t, \vecs \xi_{t-1}, .., \vecs \xi_{0})= \vect f( \vecs \xi_t)$. 
  
\item \textbf{Local minima free:}
    As demonstrated in \cite{yao2022guiding}, each $C^1$-smooth vector field obstacle introduces at least one stationary point on the surface. However, it must be ensured that (1) this point is a saddle point and not a minimum and (2) there are no additional stationary points in space.
  
\item \textbf{Limited and smooth dynamics;}
  The  magnitude of the vector field is upper bounded, i..e., $\exists \vect v^{\mathrm{max}} = \const: \| \dot{\vecs \xi} \| < \vect v^{\mathrm{max}}$.
  Additionally, the vector field is smooth, meaning that small displacements result in proportionally small velocities: $\lim_{\vecs \xi_1 \rightarrow \vecs \xi_2} \| \dot{\vecs \xi}_2 - \dot{\vecs \xi}_1 \| = 0$

\item \textbf{General dimensions:}
  The obstacle avoidance algorithm is applicable in a space of dimensions  $N \geq 2$.
\end{itemize}

Furthermore, we assume that the environment and initial dynamics $\vect f(\vecs \xi)$, as described in \eqref{eq:ideal_initial_dynamics}, possess the following properties:
\begin{itemize}
\item \textbf{Star-shaped:} All obstacles are star-shaped as defined in Sec.~\ref{sec:kernel_space}, or are composed of obstacles (trees-of-stars) with a nonzero intersection regions among the components of a tree, as discussed in Section~\ref{sec:kernel_space}.
  
\item \textbf{Limited and smooth dynamics:} Analogously to the output, the magnitude of the initial dynamics $\vecs f(\vecs \xi)$ is required to be $C^1$-smooth and limited, i.e., $\| \vecs f(\vecs \xi) \| < v^{\mathrm{max}}, \;\; \forall \vecs \xi \in \mathbb{R}^N, v^{\mathrm{max}} \in \mathbb{R}_{>0}$. 

\item \textbf{Dynamics as a vector rotation:} 
The initial dynamics $\vect f(\vecs \xi)$ can be evaluated as a local rotation (Section~\ref{sec:rotated_vector_field}) or a sequence of rotations (Appendix~\ref{sec:perpendicular_rotation}) of globally straight dynamics with respect to a fixed point $\vecs \xi^a$ (Definition~\ref{def:straight_dynamics}).
\end{itemize}


\section{Obstacle Avoidance through Rotation} \label{sec:rotational_avoidance}


A smooth vector field that effectively avoids an obstacle should have the velocity $\dot{\vecs \xi}$ directed away from or perpendicular to the normal $\vect n(\vecs \xi) \in \mathbb{R}^N$ as defined in \eqref{eq:normal_direction}. This principle is commonly expressed as follows \cite{khansari2012dynamical, feder1997real}:
\begin{equation}
\dotprod{\vect n(\vecs \xi)}{\dot{\vect \xi}} \geq 0
\quad \forall \, \vecs \xi \in \mathcal{X}^b
\label{eq:boundary_condition}
\end{equation}

\subsection{Vector Rotation for Collision Avoidance}
We propose a method called Rotational Obstacle Avoidance Method (ROAM) to achieve collision-free motion. ROAM smoothly adjusts the initial velocity $\vect f(\vecs \xi)$ as given in Eq.~\eqref{eq:ideal_initial_dynamics}, rotating it towards a feasible half-space as the position approaches the obstacle (Fig.~\ref{fig:nonlinare_vectorfield_creation}). The steps involved in ROAM for avoiding a single obstacle are as follows:
\begin{enumerate}
    \item Evaluation of pseudo-tangent direction $\vecs e(\vecs \xi)$ using the convergence direction $\vecs c(\vecs \xi)$ (Sec.~\ref{eq:pseudo_tangent})
    \item Rotation of the initial dynamics $\vecs f(\vecs \xi)$ towards the pseudo-tangent $\vecs e(\vecs \xi)$ to obtain the collision-free direction $\dot{\vecs \xi}$ (Sec.~\ref{sec:rotation_of_dynamics}).
    \item Evaluation of the velocity magnitude $h(\vecs \xi)$ to ensure a smooth vector field (Sec.~\ref{sec:evaluation_of_speed})
\end{enumerate}

\begin{figure}[tbh]\centering
  \hfill
  \begin{subfigure}{.48\columnwidth}
\centering
\includegraphics[width=\textwidth]{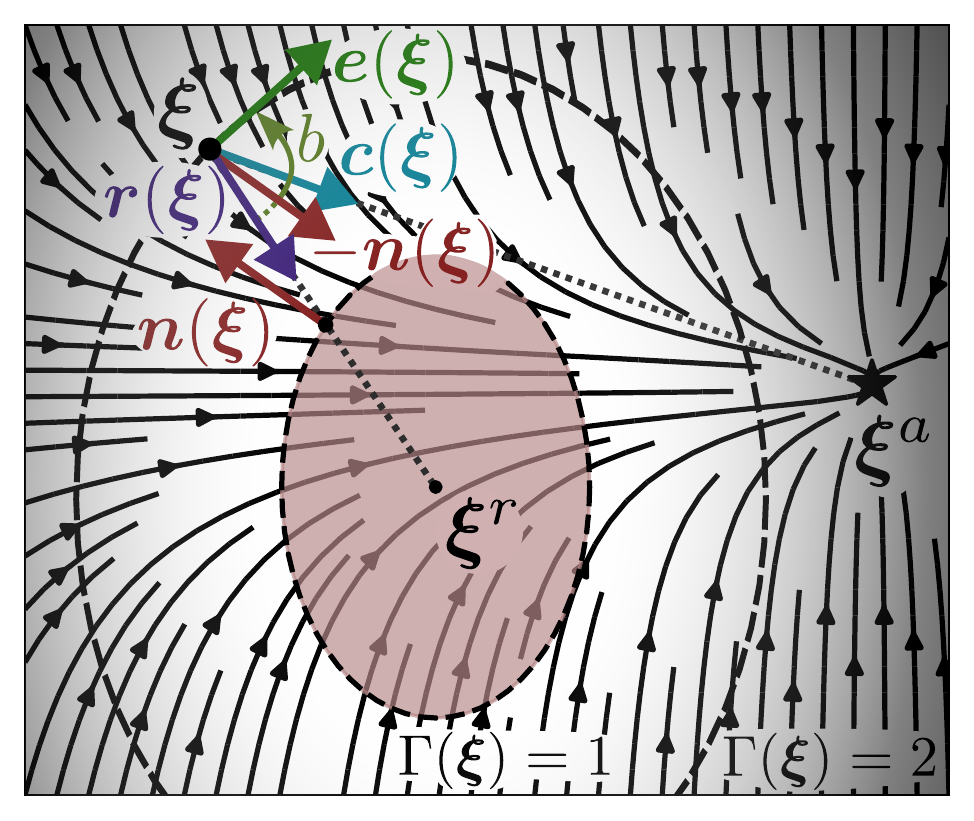}
\caption{Initial dynamics}
\label{fig:nonlinear_vectorfield_initial}
\end{subfigure}\hfill%
\begin{subfigure}{.48\columnwidth}
\centering
\includegraphics[width=\textwidth]{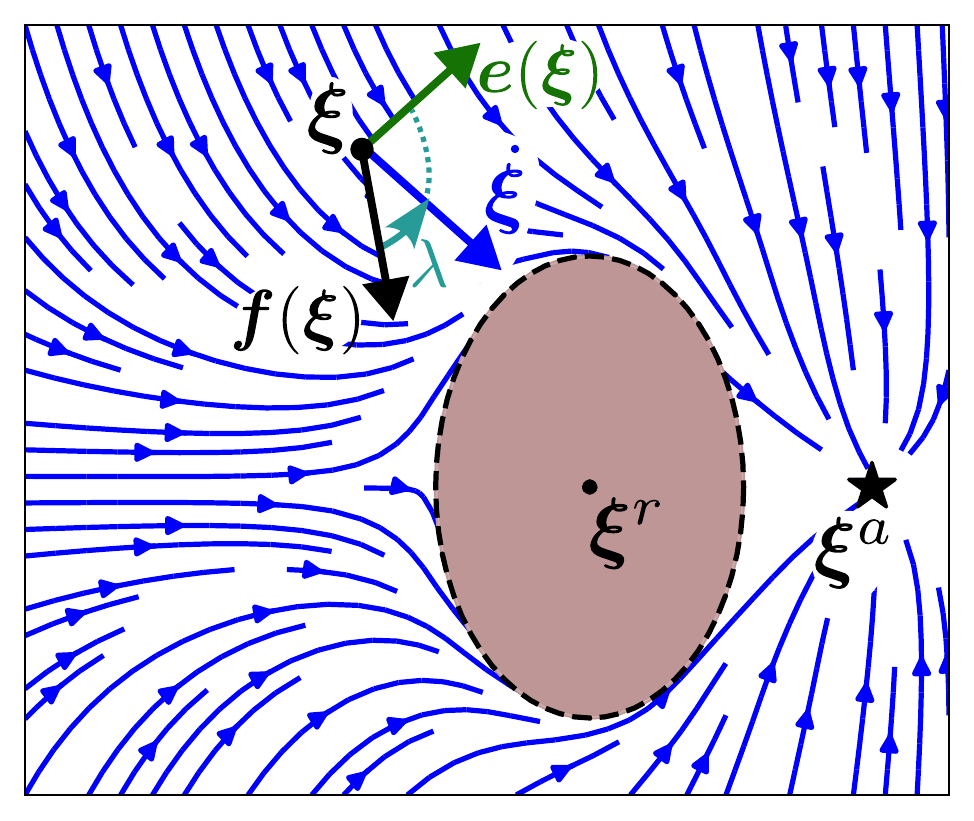}
\caption{Rotated system}
\label{fig:nonlinear_vectorfield_rotated}
\end{subfigure}\hfill%
\hfill
\caption{A nonlinear dynamical system 
$\vect f(\vecs \xi) = \text{diag}( 1 \;\; \Delta \vecs \xi_{[2]}  ) \| \Delta \vecs {\xi} \|$ with $\Delta {\vecs \xi} = \vecs \xi^a - \vecs \xi$
(a) is rotated using ROAM to ensure convergence towards the attractor $\vecs \xi^a$ with a single saddle point on the surface (b).}
\label{fig:nonlinare_vectorfield_creation}
\end{figure}

\subsubsection{Evaluation of Preferred Tangent Direction} \label{eq:pseudo_tangent}
At each position on the surface of an obstacle, there can exist multiple tangent directions. In the case of $d=2$, there are exactly two tangent directions, while for $d \geq 3$, there are infinitely many. To construct a smoothly defined pseudo-tangent $\vecs e(\xi)$, we introduce the concept of convergence dynamics, which are used to obtain the pseudo-tangent:

\begin{definition}[Convergence Dynamics] \label{def:convergence_dynamics}
    The $C^1$-smooth convergence-dynamics $\vect c: \mathbb{R}^N \rightarrow \mathbb{R}^N$ guides the initial dynamics $\vect f(\vecs \xi)$ around obstacles. 
    The convergence dynamics are locally straight according to Definition~\ref{def:straight_dynamics}, on the surface of the obstacle $\mathcal{X}^b$ as defined in \eqref{eq:levelFunc}
    Furthermore, the convergence dynamics $\vect c(\vecs \xi)$ is set to never be anti-collinear to the initial dynamics, i.e., $\dotprod{\vect f(\vecs \xi)}{\vect c(\vecs \xi)} \neq - \|\vect f(\vecs \xi)\| \, \| \vect c(\vecs \xi)\|
    \; \forall \vecs \xi \in \mathbb{R}^N$. \\ 
\end{definition}
In this section, we utilize convergence dynamics of the form $\vect c(\vecs \xi) = \vecs \xi - \vecs \xi^a$. For more general convergence dynamics, please refer to Section~\ref{sec:convergence_dynamics}.

The desired pseudo-tangent $\vect e(\vecs \xi)$ is obtained by rotating the convergence dynamics $\vecs c(\vecs \xi)$ away from the reference direction $\vecs r(\vecs \xi)$  until it lies in the tangent plane (Fig.~\ref{fig:nonlinare_vectorfield_creation}). 
Since, in the angular space $\vect k(\vecs n, \cdot)$, given in \eqref{eq:direction_space_constraints}, any tangent vector lies at a distance $R^e = \pi / 2$ to the normal direction, the tangent hyper-plane $\mathcal{T}$ forms a hyper-sphere in the direction space, as visualized in Figure~\ref{fig:visualization_direction_space}.
Hence, the rotation of the convergence dynamics $\vecs c(\vecs \xi)$ away from reference direction $\vect r(\vecs \xi)$ can be evaluated by intersecting the line connecting $\vect r(\vecs \xi)$ and $\vect c(\vecs \xi)$ with a circle of radius $R^e \in [\pi/2, \pi]$, see Fig.~\ref{fig:visualization_direction_space}. Thus, $\vect e(\vecs \xi)$ can be obtained through the following constraints\footnote{For conciseness dependency on $\vecs \xi$ is omitted.}:
\begin{gather}
  \text{if}
  \quad \| \vect k \left( \sminus \vect n, \vect c \right) \| \geq {R^e} 
  \quad \text{then} \quad {\vect e} = {\vect c}
  \quad \text{otherwise}   \nonumber \\
  \vect k \left(\sminus \vect n, {\vect e} \right) = (1 - b) \vect k \left( \sminus \vect n, {\vect r} \right) + b \vect k \left(\sminus \vect n,   \vect c \right)
  \label{eq:tangent_calculation} \\
    \text{such that} \quad \| \vect k \left(\sminus \vect n, \vect e \right) \| = R^e, \quad b \in \mathbb{R}_{>0} \;\; \forall \vecs \xi : \vect{c} \neq \vect{r} \nonumber
\end{gather}
The solution to the above equality constraints is obtained analytically by solving a quadratic equation with respect to the scalar $b$.

As a result, the pseudo-tangent $\vect e(\xi)$ can either lie in the tangent plane $\mathcal{T}$ or point away from the obstacle, with a distance to the normal direction contained within the green region depicted in Figure~\ref{fig:visualization_direction_space}.
In the special case where $\vect c(\vecs \xi) = \vect r(\vecs \xi)$, the intersection of the hyper-circle in Equation~\ref{eq:tangent_calculation} does yield a solution. However, it is ensured that the final vector field $\dot{\vecs \xi}$ is continuously defined, as discussed in Section~\ref{sec:rotation_of_dynamics}.

\begin{lemma} \label{lemma:pseudo_tangent}
 Let us assume the convergence dynamics $\vect c(\vecs \xi)$, as given in Definition~\ref{def:convergence_dynamics}.
  The pseudo tangent $\vect e(\vecs \xi) \in \mathbb{R}^N$ obtained through \eqref{eq:tangent_calculation} is smooth and satisfies the boundary inequality $\dotprod{\vect n(\vecs \xi)}{\vect e(\vecs \xi)} \geq 0$, stated in \eqref{eq:boundary_condition}, at any position on the surface of the obstacle 
  but the saddle point, i.e., $\vecs \xi \in \mathcal{X}^b : \dotprod{\vect c(\vecs \xi)}{\vect n(\vecs \xi)} \neq -1 \}$ 
\end{lemma}

\begin{proof}
  For a starshaped obstacle, we have by definition of the starshaped kernel space in \eqref{eq:kernel_space}, that $\dotprod{\vect n(\vecs \xi)}{\vect r(\vecs \xi)} < 0$. Hence, $\| \vect k (-\vect n, \vect r) \| < \pi / 2$, i.e.,  $\vect k (-\vect n, \vect r)$ is strictly inside the hyper-sphere $\mathcal{T}$.
  Furthermore, as long as $\vect c(\vecs \xi) \neq \vect r(\vecs \xi)$, the equality of finding the intersection of a line and a hyper-sphere from \eqref{eq:tangent_calculation} has exactly one solution with $b > 0$.

  Concerning the boundary condition; looking at the case of $\| \vect k \left( \sminus \vect n, \vect c \right) \| > {R^e}$. From the definition of the direction-space in \eqref{eq:direction_space_constraints}, it follows that:
  \begin{equation}
    \dotprod{\vect c(\vecs \xi)}{\vect n (\vecs \xi)} > 0
      \underset{\vect c(\vecs \xi) = \vect e(\vecs \xi)}{\Rightarrow}
      \dotprod{\vect e(\vecs \xi)}{\vect n (\vecs \xi)} > 0
\end{equation}
For the case that $\dotprod{\vect c(\vecs \xi)}{\vect n (\vecs \xi)} < 0$, from \eqref{eq:direction_space_constraints} and \eqref{eq:tangent_calculation} we can conclude:
\begin{equation}
\begin{split}
\dotprod{\vect n}{\vect e} & = \cos \left( \| \vect k (\vect n,  \vect e )\| \right) = - \cos \left( \| \vect k ( - \vect n,  \vect e) \| \right) \\
& \geq - \cos \left( R^e \right) \geq - \cos \left( \pi/2 \right) = 0
\end{split}
\end{equation}
Hence, we have a uniquely defined pseudo tangent $\vect e(\vecs \xi)$, which satisfies the boundary condition given in \eqref{eq:boundary_condition}. 
\end{proof}

\begin{figure}[tbh]\centering
\begin{subfigure}{.49\columnwidth}
\centering
\includegraphics[width=\textwidth]{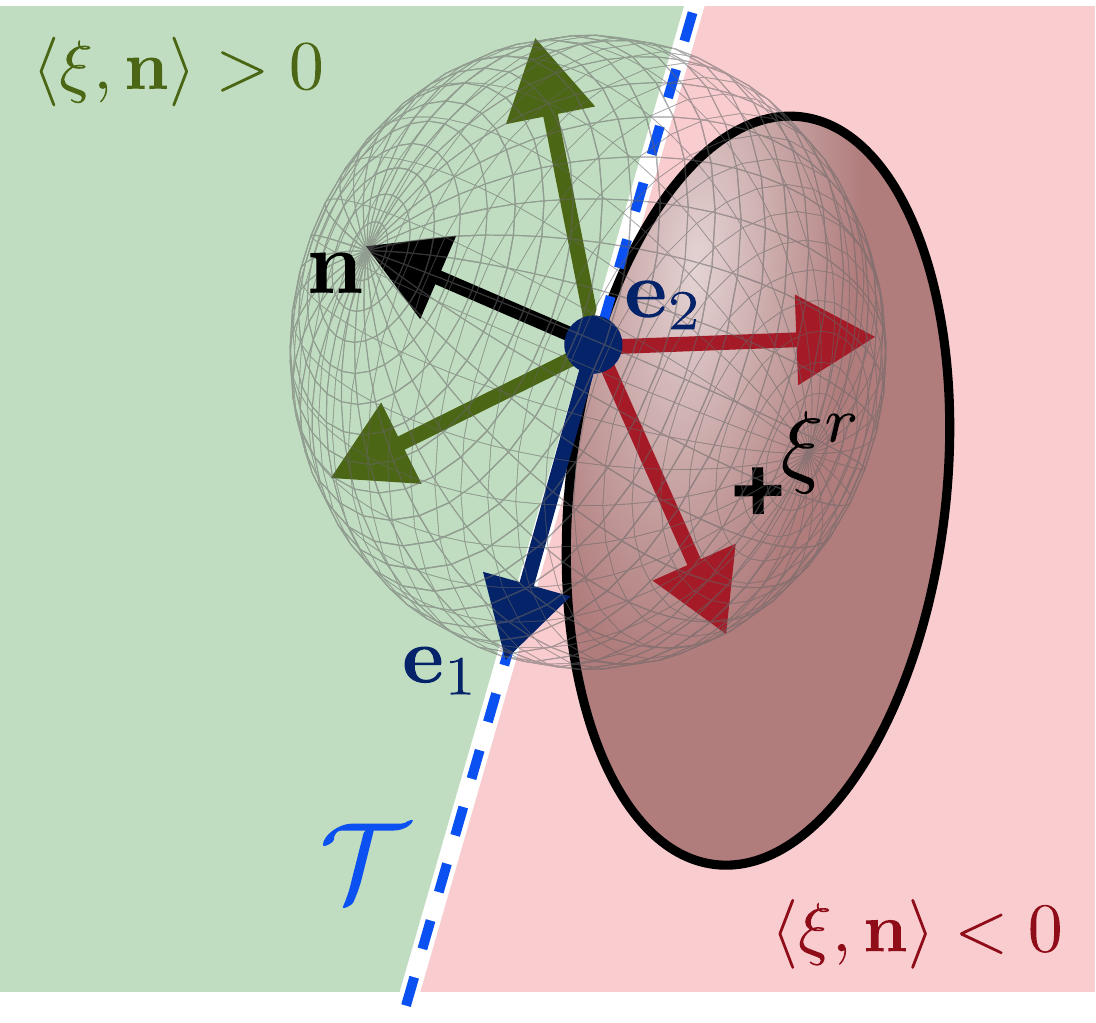}
\caption{Euclidean space}
\label{fig:collision_vectors}
\end{subfigure}%
\hfill
\begin{subfigure}{.48\columnwidth}
\centering
\includegraphics[width=\textwidth]{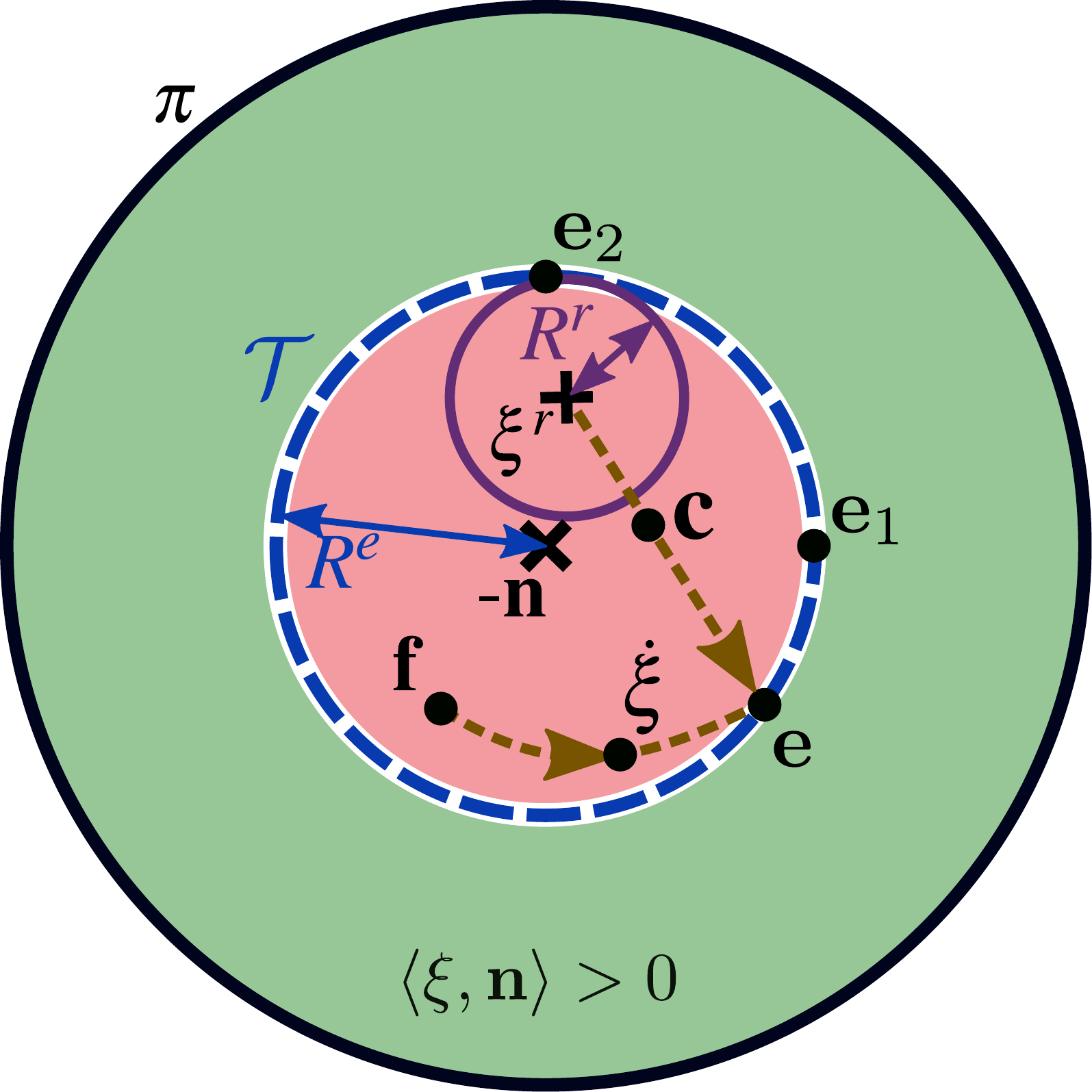}
\caption{Directional space $\vect k(-\vect n, \cdot)$}
\label{fig:directional_space_tangent}
\end{subfigure}%
\caption{When approaching the surface of the obstacle, the collision-free vectors are located on one side of the tangent plane $\mathcal{T}$ formed by the linearly independent tangent vectors $\vect e_{(\cdot)}$ (a). These unit vectors can be projected onto a circular hypersphere of dimension $N-1$, where the collision-free pseudo tangent $\vecs e$ is calculated (b). Any vector $\vecs \xi$ pointing towards the obstacle, i.e., $\dotprod{\vect n}{\vecs \xi} < 0$ (a), is projected to lie inside the circle of radius $\pi/2$ (b), while a vector pointing away from the obstacle is projected to the outer ring.}
\label{fig:visualization_direction_space}
\end{figure}

\subsubsection{Rotation Towards Tangent Direction} \label{sec:rotation_of_dynamics}
In a second step, the initial dynamics $\vect f(\vecs \xi)$ is \textit{rotated} towards the pseudo tangent $\vect e(\vecs \xi)$. 
This rotation operation is performed in the direction space, as described in Section~\ref{sec:rotated_vector_field}, but this time with respect to the convergence dynamics $\vecs c(\vecs \xi)$:
\begin{gather}
   \dot{\vecs \xi} = \bar{\vect k} \Bigl( \vect c, \,\bigl(1-\lambda(\vecs \xi)\bigr) \angspace{\vect c}{\vect f} +  \lambda(\vecs \xi) \, \angspace{\vect c}{\vect e} \Bigr) \quad \text{with} \nonumber \\ 
   \lambda(\vecs \xi) = \left(\frac{1}{\Gamma(\vecs \xi)}\right)^q , \;\;
    R^r = \min \left( R^e -  \| \angspace{\sminus \vect n}{\vect r} \|, \frac{\pi}{2}\right) \label{eq:pulling_weight} \\ 
      q  = \max \left( 1 , \; \frac{ R^r }{\Delta \vect k(\vect c)} \right)^{s},
     \quad \Delta \vect k(\vect c) = \| \angspace{\sminus \vect n}{\vect r} - \angspace{\sminus \vect n}{\vect c} \| \nonumber  
 \end{gather}
The rotation weight $\lambda(\vecs \xi) \in [0, 1]$ is determined based on the inverse of the distance $\Gamma$ to ensure that the rotation has a decreasing influence as the distance increases. This allows for a smooth transition in the avoidance behavior. Additionally, the smoothing factor $q$ is introduced, which gradually decreases to zero when the convergence dynamics point towards the robot. This effectively cancels the rotation avoidance effect in regions where the tangent $\vect e(\vecs \xi)$ is not defined, resulting in a smooth avoidance behavior across those positions. The impact of the smoothness constant $s \in \mathbb{R}_{>0}$ can be observed in Figure~\ref{fig:comparison_smoothing_constant}, and unless otherwise specified, we use $s=0.3$.
 
\begin{lemma} \label{lemma:velocity_rotation}
The vector field $\dot{\vecs \xi}$ obtained through rotation as given in \eqref{eq:pulling_weight}  satisfies the boundary condition as defined in \eqref{eq:boundary_condition} if the pseudo tangent $\vect e(\vecs \xi)$ is tangent or pointing away from the obstacle, i.e., $\dotprod{\vect e}{\vect n} \geq 0$, $\forall \vecs \xi: \Delta \vect k(\vect c) \neq 0$.
\end{lemma}

\begin{proof}
	When approaching the obstacle, the rotated velocity $\dot{\vecs \xi}$ is evaluated as:
	\begin{equation}
		\lim_{\Gamma(\vecs \xi) \rightarrow 1, \; } \lambda(\vecs \xi) = 1 
		\;\; 
			\underset{\text{with} \; \eqref{eq:inverse_angular_mapping}} {\Rightarrow} \;\;
	\lim_{\lambda \rightarrow 1} \dotprod{\vect n}{\dot{\vect \xi}}
	= \dotprod{\vect n}{\vect e} \geq 0
	\end{equation}
    using Lemma~\ref{lemma:pseudo_tangent} and that $\Delta \vect k(\vect c) \neq 0 \Rightarrow q > 0$ with the smoothness factor $q$ given in \eqref{eq:pulling_weight}.
    
Following the same logic, we can also analyze the behavior far away from obstacles:
\begin{equation}
	\lim_{\Gamma(\vecs \xi) \rightarrow \infty} \lambda(\vecs \xi) = 0
	\;\;
	\Rightarrow
	\;\;
	\lim_{\lambda \rightarrow 0} \dot{\vecs \xi} = \vect f(\vecs \xi)
\end{equation}
\end{proof}

\begin{figure}[tbh]\centering
\begin{subfigure}{.33\columnwidth}
\centering
\includegraphics[width=\textwidth]{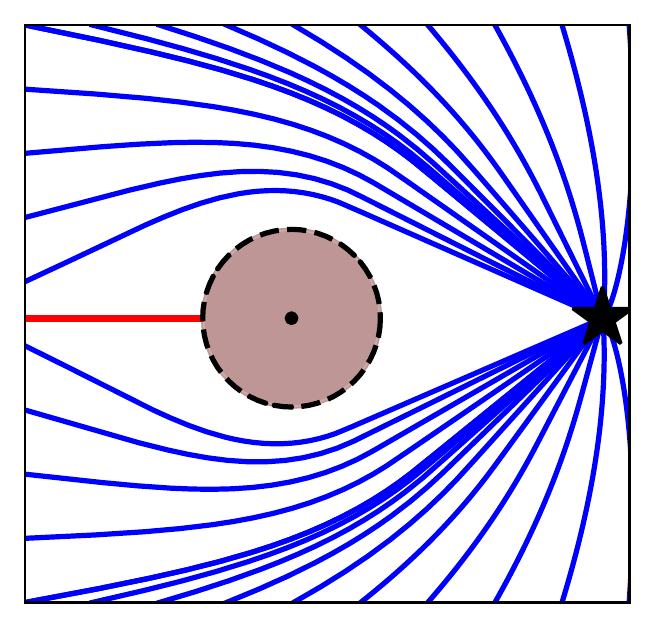}
\caption{$s = $1e-6}
\label{fig:comparison_linear_integration_00}
\end{subfigure}%
\begin{subfigure}{.33\columnwidth}
\centering
\includegraphics[width=\textwidth]{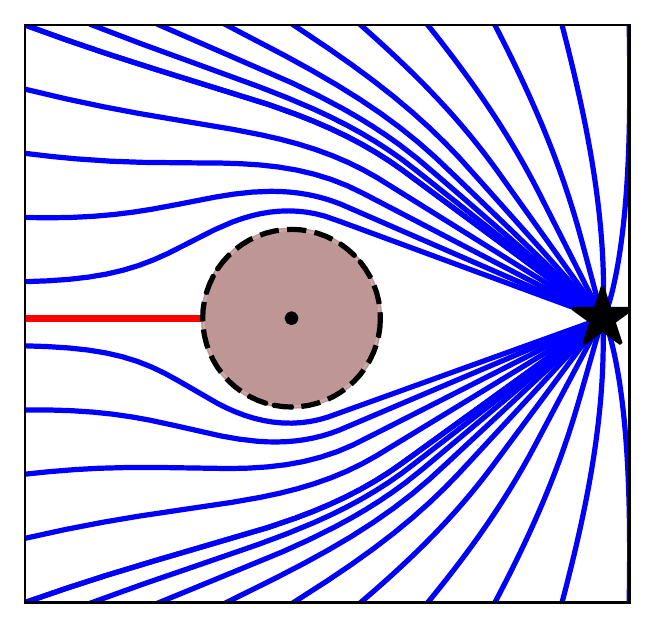}
\caption{$s = 0.3$}
\label{fig:comparison_linear_integration_03}
\end{subfigure}%
\begin{subfigure}{.33\columnwidth}
\centering
\includegraphics[width=\textwidth]{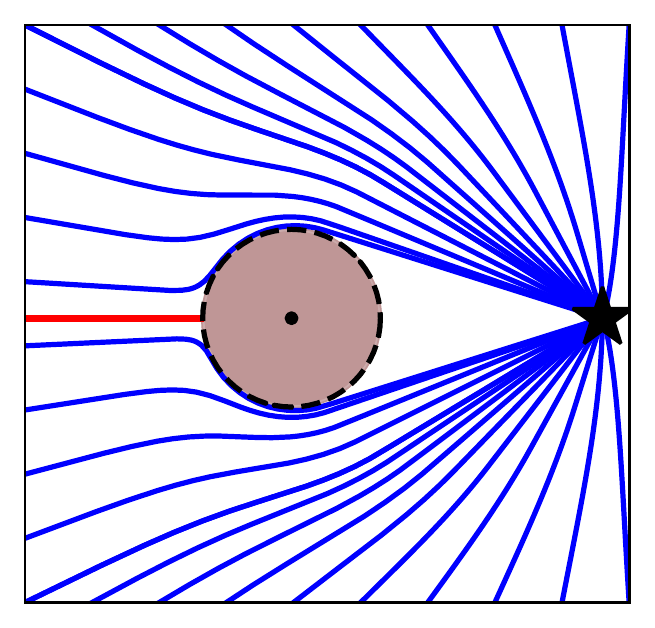}
\caption{$s = 1.0$}
\label{fig:comparison_linear_integration_10}
\end{subfigure}%
\caption{A smaller smoothness-constant $s$ increases the reactivity when approaching an obstacle (a), while a larger value ensures a smoother transition across the (red) saddle point trajectory (b, c).}
\label{fig:comparison_smoothing_constant}
\end{figure}


\subsubsection{Evaluation of Speed} \label{sec:evaluation_of_speed}
The directional space mapping $\vect k(\cdot)$ and its inverse mapping $\bar{\vect k}(\cdot)$ operate on unit vectors in the original space. As a result, the algorithm in \eqref{eq:pulling_weight} modifies the direction of the initial dynamics, rather than its magnitude. To incorporate magnitude control in a decoupled manner, we introduce an additional component using $h(\vecs \xi) : \mathbb{R}^N \rightarrow [0, 1]$:
\begin{equation}
	\| \dot {\vecs \xi} \| = h(\vecs \xi) \| \vect f(\vecs \xi) \| 
\end{equation}
The stretching factor $h(\vecs \xi)$ is designed to slow down when pointing towards an obstacle, and the effect decreases with increasing distance from the obstacle:
\begin{equation}
	h(\vecs \xi) = \min \left(1, \, \left({\frac{\| \Delta \vect k \left(\vect c\right) \|}{R^r}}\right) ^2 + \left(1 - \frac{1}{\Gamma(\vecs \xi)}\right) ^2 \right)
	\label{eq:magnitude_scaling}
\end{equation}
where the reference radius $R^r$ and the rotation space distance $\Delta \vect k(\vecs c)$ are given in \eqref{eq:pulling_weight}. 
In Figure~\ref{fig:comparison_nonlinear_vectorfield}, the velocity scaling can be observed to decrease the magnitude of the vectors pointing toward the obstacle.

\begin{theorem} \label{theorem:collision_avoidance}
Consider a vector field $\dot{\vecs \xi} \in \mathbb{R}^N$ obtained after a local rotation as defined in \eqref{eq:pulling_weight} of an initial dynamics $\vect f(\vecs \xi)$ and fixed point at $\vecs \xi^a$, with pseudo tangent $\vect e(\vecs \xi)$ defined in \eqref{eq:tangent_calculation} with respect to initial dynamics $\vect f(\vecs \xi)$, is locally straight on the surface of the obstacle $\mathcal{X}^b$ according to Definition~\ref{def:straight_dynamics}, and motion scaling $h(\vecs \xi)$ according to \eqref{eq:magnitude_scaling}. Any motion starting in free space $\{\vecs \xi\}_0 \in \mathcal{X}^e$ which evolves according to $\dot{\vecs \xi}$ will stay in free space for finite time $\{\vecs \xi\}_t \in \mathcal{X}^e$ with $t \in \mathbb{N}_{>0}$ and maintains the stationary point, i.e., $ \dot{\vecs \xi} = \vect 0$ if $\vect f(\vecs \xi) = \vect 0$ 
\end{theorem}

\begin{proof}
  From Lemma~\ref{lemma:pseudo_tangent}, we know 
  $
  \dotprod{\vect n(\vecs \xi)}{\vect e(\vecs \xi)} \geq 0
  $
  and Lemma~\ref{lemma:velocity_rotation} states that
  $
      \dotprod{\vect e}{\vect n} \geq 0
      \;\;
      \forall \vecs \xi: \Delta \vect k(\vect c) \neq 0
  $.
Additionally, in the latter case, the magnitude scaling from \eqref{eq:magnitude_scaling} evaluates as
\begin{equation}
	\lim_{\| \Delta \vect k(\vect c)\| \rightarrow 0, \Gamma(\vecs \xi) \rightarrow 1} h(\vecs \xi) = 0
\end{equation}
Hence, the speed reaches smoothly zero as we approach the saddle point, and its direction does not matter to fulfill smoothness and the boundary condition given in \eqref{eq:boundary_condition}.

Furthermore, far away from the obstacle, we have:
\begin{equation}
	\lim_{\Gamma(\vecs \xi) \rightarrow \infty} h(\vecs \xi) = 1
\end{equation}
Hence, the magnitude is equal to the original magnitude and only vanishes around existing stationary points, i.e., $\vect f(\vecs \xi) = 0$.
\end{proof}

The spurious stationary point on the surface is given by $\{ \vecs \xi : \vecs\xi \in \mathcal{X}^b, \, \|\vect k(\vect c)\| = 0 \}$. Furthermore, by construction, the pseudo tangent $\vect e(\vecs \xi)$ defined in \eqref{eq:tangent_calculation} points away from the reference direction and hence points away from the stationary point on the surface of the obstacle. Hence, it is a saddle point with a single trajectory converging to it $\mathcal{X}^s$. 

\subsection{Surface Repulsion} \label{sec:surface_repulsion}
n the vicinity of critical obstacles, it may be desirable to incorporate active repulsion from them without significantly altering their shape. This can be achieved by introducing a behavior similar to artificial potential fields \cite{rimon1992exact}, but without introducing spurious attractors in free space.
Building upon the developments presented in this section, we can further refine the tangential radius. By increasing the tangential radius such that $R^e \in ] \pi/2, \pi]$, values larger than $\pi / 2$, leads to a repulsive behavior while maintaining the avoidance properties (Fig.~\ref{fig:comparison_nonlinear_vectorfield}).

\begin{figure}[tbh]
    \centering
    \includegraphics[width=1.0\columnwidth]{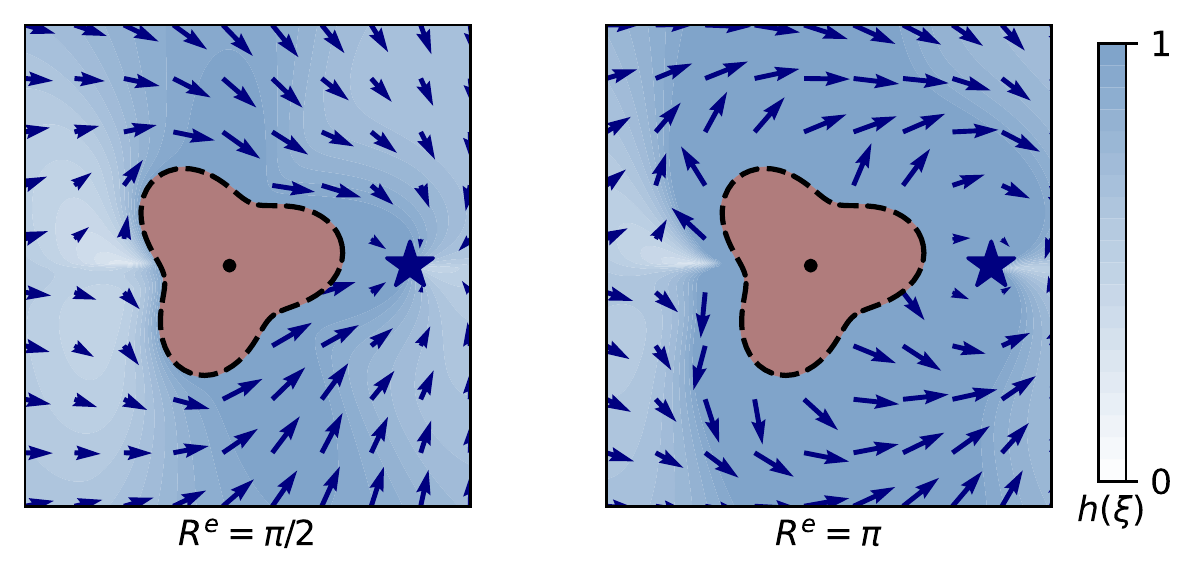}
    \caption{While a surface repulsion radius of  $R^e = \pi / 2$  ensures collision avoidance, an increased radius, i.e., $R^e \in ] \pi / 2, \pi]$ leads to increased repulsion around the obstacle. The initial dynamics are given by $\vect f(\vecs \xi) = (\vecs \xi - \vecs \xi^a)$.}
    \label{fig:comparison_nonlinear_vectorfield}
\end{figure}


\subsection{Inverted Obstacles} \label{sec:inverted_obstacles} 
Initial dynamics might be contained within a boundary, such as an agent moving in a room or a robot staying within its joint limits.
These constraints can be incorporated by inverting an obstacle and ensuring that the dynamics remain inside a boundary hull (Fig.~\ref{fig:inverted_nonlinear_avoidance}). 
Analogously to \cite{huber2022avoiding}, this is achieved by inverting the distance function $\Gamma(\vecs \xi)$, the reference direction $\vect r(\vecs \xi)$ and normal vector $\vecs n(\vecs \xi)$, defined in Sec.~\ref{sec:preliminairies}.
The distance function divides the space into free, boundary, and interior points, as described in \eqref{eq:levelFunc}. Consequently, the distance for the boundary obstacle can be defined as follows:
\begin{equation}
  \Gamma(\vecs \xi) = \left( R(\vecs \xi)/ \|\vecs \xi - \vecs \xi^r \| \right)^{2p} \quad \forall \vecs \xi \in  \mathbb{R}^N \setminus \vecs \xi^r \label{eq:inverse_gamma}
\end{equation}
with power weight $p \in \mathbb{R}_{>0}$, we choose $p=1$.

As the normal direction points away from the surface, it naturally points towards the interior of an inverted obstacle, in contrast to the normal direction of a regular obstacle. Consequently, in order to utilize the star-shaped constraint from \eqref{eq:pulling_weight} of $\dotprod{- \vect r(\vecs \xi)}{\vect n(\vecs \xi)} > 0$ for rotational obstacle avoidance, we need to flip the reference direction: 
\begin{equation}    
  \vect r(\vecs \xi)  = {\left( \vecs \xi^r - \vecs \xi \right)}/{\|\vecs \xi^r - \vecs \xi  \|}
  \qquad
  \forall \vecs \xi \in \mathbb{R}^N \setminus \vecs \xi^r 
	\label{eq:inverted_normal}
\end{equation}
Using the \textit{inverted} values, the ROAM can be applied on as described  throughout this section (Fig.~\ref{fig:inverted_nonlinear_avoidance}).

\begin{figure}[tbh]\centering
  \hfill
  \begin{subfigure}{.45\columnwidth}
    \centering
    \includegraphics[width=\textwidth]{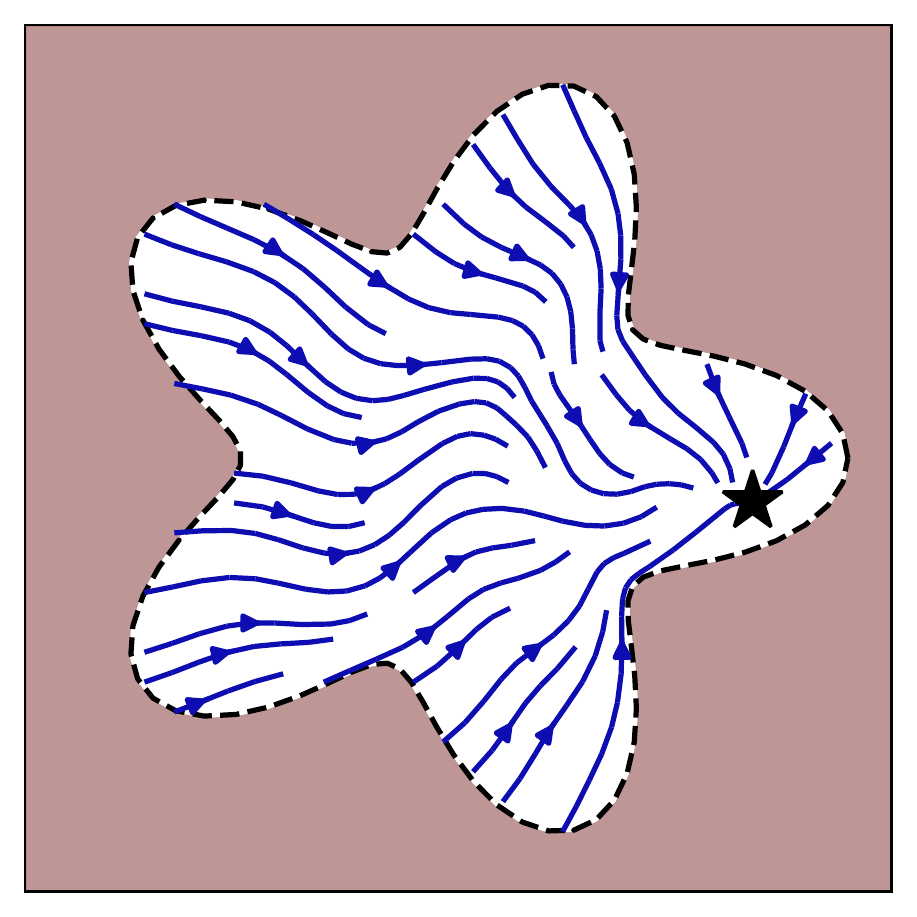}
    \caption{Wavy DS in starshape}
    \label{fig:starshaped_wall}
  \end{subfigure}\hfill%
  \begin{subfigure}{.45\columnwidth}
    \centering
\includegraphics[width=\textwidth]{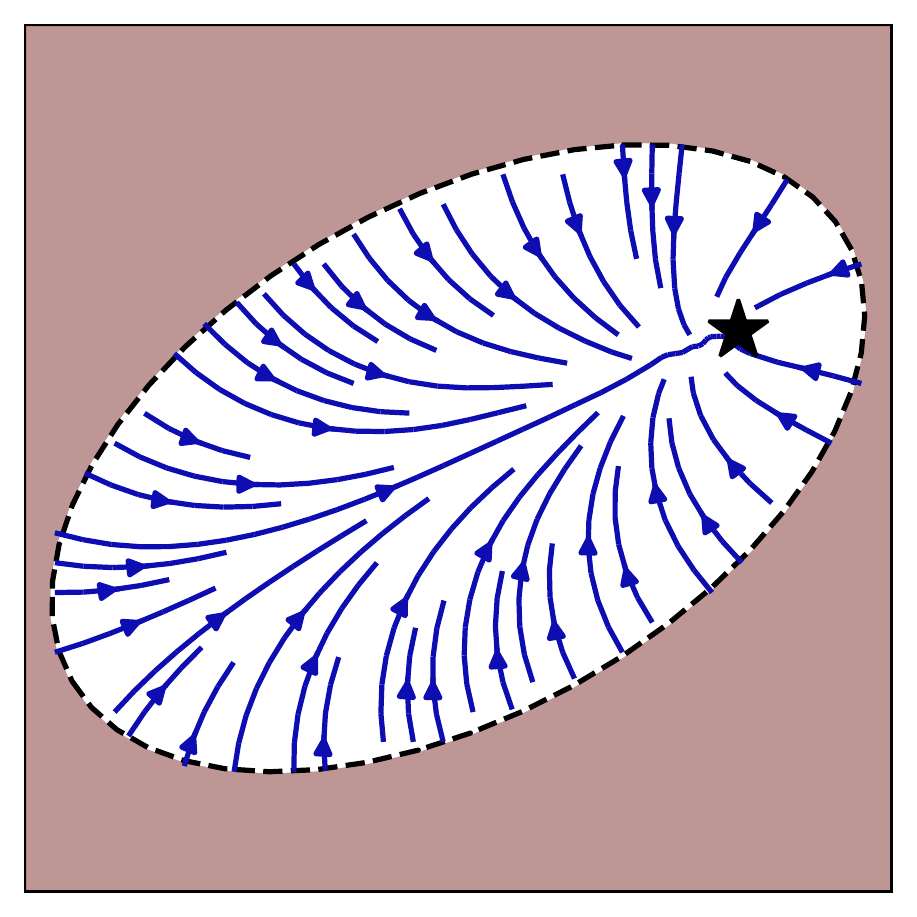}
    \caption{Linear DS in ellipse}
    \label{fig:ellipsoid_wall}
  \end{subfigure} \hfill%
  \caption{The inverted obstacle description is used to keep nonlinear dynamics $\vect f(\vecs \xi)$ within star-shaped boundaries, in (a) with $\vect f(\vecs \xi) = \matr R(\sin(\| \vecs\xi^a - \vecs\xi \|) (\vecs\xi^a - \vecs\xi)$ where $\matr R(\cdot)$ is a two dimensional rotation matrix and $\vecs \xi^a$ the attractor. Further, it can be used to create active repulsion from walls, as in (b) for globall straight dynamics with attractor $\vecs \xi^a$.}
  \label{fig:inverted_nonlinear_avoidance}
\end{figure}

\begin{lemma} \label{lemma:rotation_pulling}
 The rotated dynamics $\dot{\vect \xi}$  evaluated according to \eqref{eq:pulling_weight} are $C^1$-smooth and collision-free as shown in Theorem~\ref{theorem:collision_avoidance} when navigating within a boundary described as an inverted obstacle with distance function $\Gamma(\vecs \xi)$ as defined in \eqref{eq:inverse_gamma}, and reference vector $\vect r(\vecs \xi)$ as given in \eqref{eq:inverted_normal}.
\end{lemma}

\begin{proof}
    Since the distance function $\Gamma(\vecs \xi)$, average vector $\vect n(\vecs \xi)$, and reference direction $\vect r(\vecs \xi)$ possess all the necessary properties outlined in Theorem~\ref{theorem:collision_avoidance}, the collision avoidance properties directly carry over.

  The distance function $\Gamma(\vecs \xi)$ from \eqref{eq:inverse_gamma} and normal direction \eqref{eq:inverted_normal} are not defined at the reference point $\vecs \xi^r$. 
  However, at this point the rotation weight $\lambda (\vecs \xi)$ reaches zeros:
  \begin{equation}
    \lim_{\xi \rightarrow \xi^r} \Gamma(\vecs \xi) \rightarrow \infty
    \;\; \underset{\text{using} \; \eqref{eq:pulling_weight}}{\Rightarrow} \;\;
    \lim_{\xi \rightarrow \xi^r} \lambda(\vecs \xi) = 0
    \;\; \Rightarrow \;\;
    \lim_{\xi \rightarrow \xi^r} \dot{\vecs \xi} = \vect f(\vecs \xi)
  \end{equation}
  Thus, the rotation has no effect, and the system is smoothly defined.
\end{proof}
A more detailed analysis of inverted obstacles can be found in \cite{huber2022avoiding}.
Further development applies to both standard and inverted obstacles if not stated otherwise.

\section{General Nonlinear Motion} \label{sec:convergence_dynamics}
For systems characterized by small nonlinearities and a single attractor, the convergence dynamics $\vect c(\vecs \xi)$ exhibit global straightness, as defined in Definition~\ref{def:straight_dynamics}, resulting in desirable behavior. However, in systems with high nonlinearities, this can lead to avoidance patterns that do not accurately reflect the global dynamics, as illustrated by the example shown in Figure~\ref{fig:global_convergence_direction}.
To address this issue, this section introduces modified convergence dynamics that aim to maintain similarity with the original dynamics while being locally straight, as depicted in Figure~\ref{fig:local_convergence_direction}.

\begin{figure}[tbh]\centering
  \hfill
  \begin{subfigure}{.49\columnwidth}
    \centering
    \includegraphics[width=\textwidth]{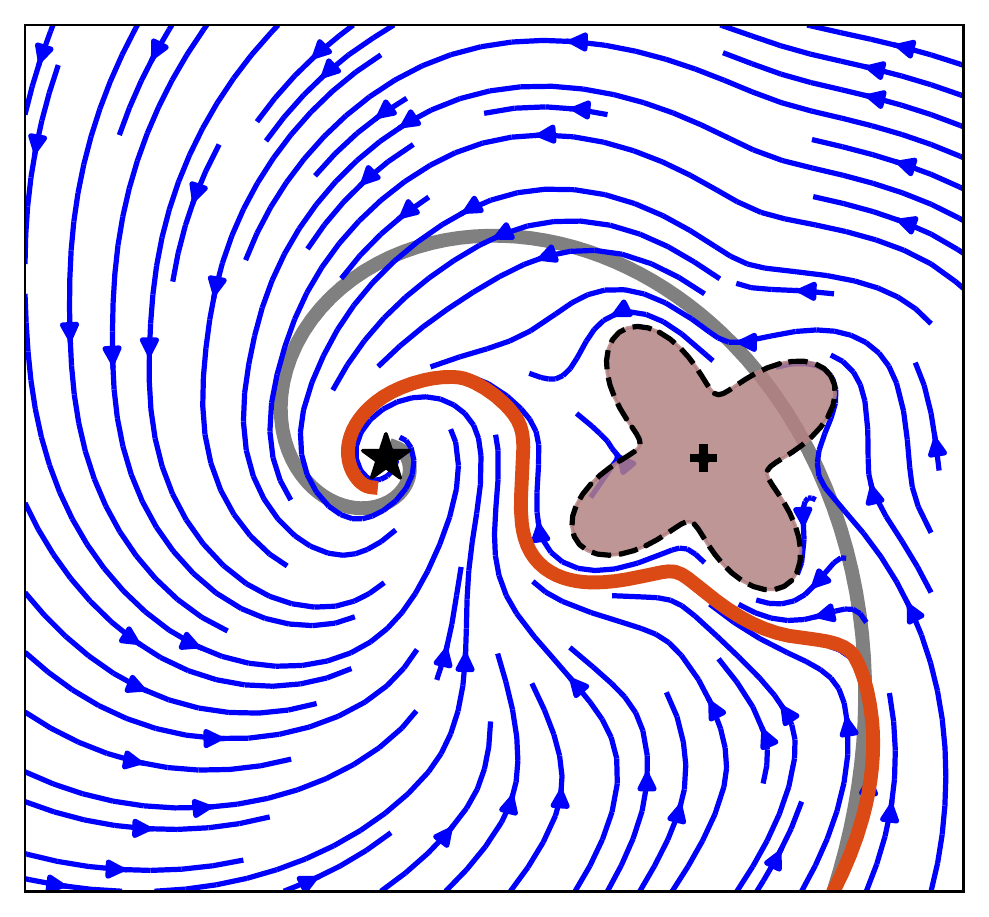}
    \caption{Globally straight $\vecs c(\vecs \xi)$}
    \label{fig:global_convergence_direction}
  \end{subfigure}\hfill%
  \begin{subfigure}{.49\columnwidth}
    \centering
    \includegraphics[width=\textwidth]{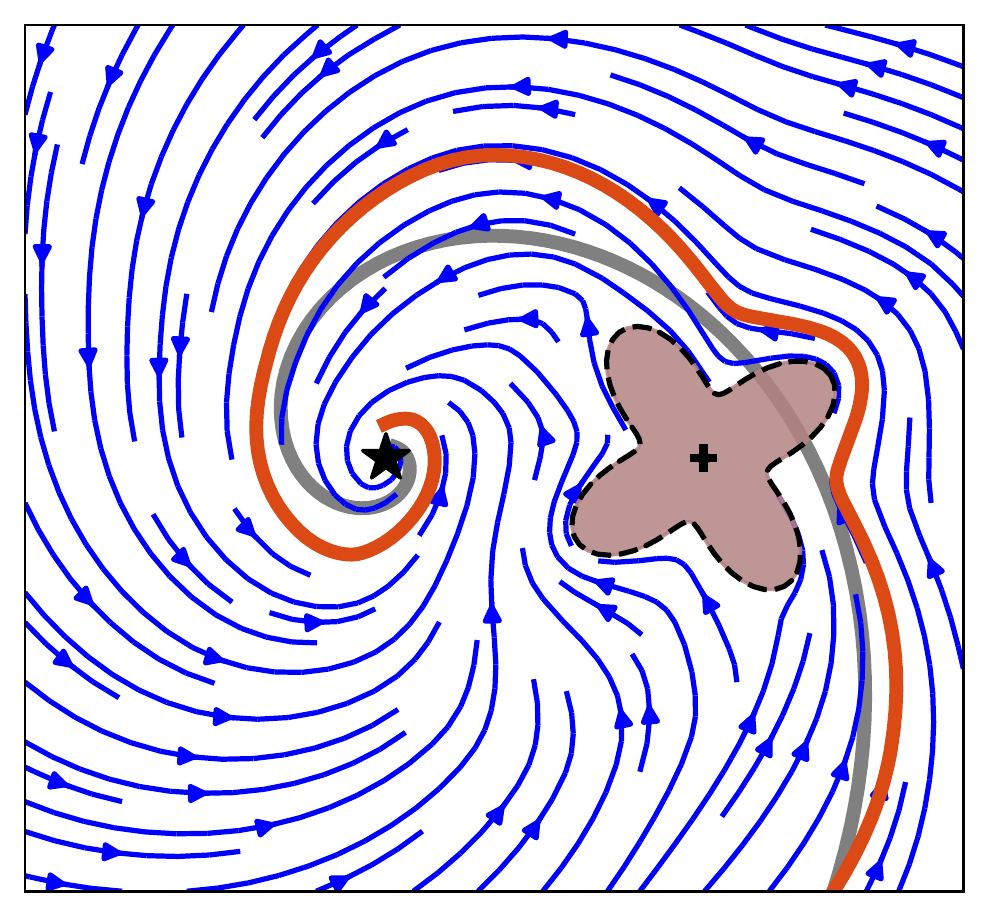}
    \caption{Locally straight $\vecs c(\vecs \xi)$}
    \label{fig:local_convergence_direction}
  \end{subfigure} \hfill%
  \caption{
    For motion with initial dynamics of $\vecs f(\vecs \xi) = \begin{bmatrix} -1 & 2 \\ -2 & -1 \end{bmatrix} \vecs \xi$ (gray) and a single stationary point (black star), employing a globally straight convergence dynamics $\vect c(\vecs \xi)$ results in a trajectory (orange) that deviates significantly from the original motion (a).
    In contrast, by utilizing locally straight convergence dynamics, the trajectory (b) maintains similarity to the original motion throughout its course.
  }
  \label{fig:local_vs_global_convergence_direction}
\end{figure}

\subsection{Nonlinear Motion without Stationary Point} \label{sec:nonlinear_dynamics}
Let us first focus on initial dynamics $\vect f(\vecs \xi)$ which do not have any directional singularity point, i.e., $\nexists \vecs \xi: \| \vect f(\vecs \xi) \| = 0, \nabla \vect f(\vecs \xi) \neq 0$.
The convergence dynamics $\vect c(\vecs \xi)$ of such initial dynamics $\vect f(\vecs \xi)$ is constructed by evaluating as a weighted sum of the initial velocity at the reference point of the obstacle $\vecs \xi^r$ and at position $\vecs \xi$:
\begin{equation}
	\vect c(\vecs \xi) = w^c(\vecs \xi) \vect f(\vecs \xi^r)  \hat{+} (1 - w^c(\vecs \xi)) \vect f(\vecs \xi) \label{eq:unbounded_convergence_dynamics}
\end{equation}
The convergence weight is chosen such that the convergence dynamics $\vect c(\vecs \xi)$ is straight on the surface of the obstacle:
\begin{equation}
  w^c(\vecs \xi) =
  \begin{cases}
    1 & \text{if} \;\; \Gamma(\vecs \xi) < 1 \\
    1 / \Gamma(\vecs \xi) & \text{otherwise}
  \end{cases}
  \label{eq:convergence_weight}
\end{equation}
where $\hat{+}$ describes the rotational summing described in Appendix~\ref{sec:perpendicular_rotation}. Note that the rotation summing from \eqref{eq:vector_basis} is not defined for two anti-collinear vectors. Hence, we cannot apply this summing at a directional singularity point, as will be further discussed in the next subsection.

\begin{lemma} \label{lemma:unbounded_dynamics}
  The convergence dynamics $\vect c(\vecs \xi) : \mathbb{R}^N \rightarrow \mathbb{R}^N$ as proposed in \eqref{eq:unbounded_convergence_dynamics} for initial dynamics without directional singularity points, i.e., $\left\{ \vecs \xi : \| \vect f(\vecs \xi) \| = 0 , \nabla \vect f(\vecs \xi) \neq 0 \right\} = \emptyset$, are straight according to Definition~\ref{def:straight_dynamics} on the surface and inside the obstacle, i.e, $\vect c(\vecs \xi) \in \mathcal{F}^s, \; \vecs \xi \in \mathcal{X}^b \cup \mathcal{X}^i$.
\end{lemma}

\begin{proof}
  Inside and on the surface of the obstacle, we have:
  \begin{equation}
    \vecs \xi \in \mathcal{X}^b \cup \mathcal{X}^i
    \quad \Rightarrow \quad
    \Gamma(\vecs \xi) \leq 1
    \quad \Rightarrow \quad
    w^c(\vecs \xi) = 1
  \end{equation}
 Hence the dynamics in this region are given as:
 \begin{equation}
   \vect c(\vecs \xi) =  1 \vect f(\vecs \xi^r) \hat{+} 0 \vect f(\vecs \xi) = \vect f(\vecs \xi^r)
   \qquad
   \vecs \xi \in \mathcal{X}^b \cup \mathcal{X}^i
 \end{equation}
 Thus, we have locally collinear dynamics.
\end{proof}
As the convergence dynamics $\vect c(\vecs \xi)$ are straight on the surface of the obstacle, they can be used in the rotational obstacle avoidance method defined in Section~\ref{sec:rotational_avoidance}, see Figure~\ref{fig:nonlinear_infinite_dynamics}.

\begin{figure}[tbh]\centering
  \begin{subfigure}{.9\columnwidth}
    \centering
    \includegraphics[width=\textwidth]{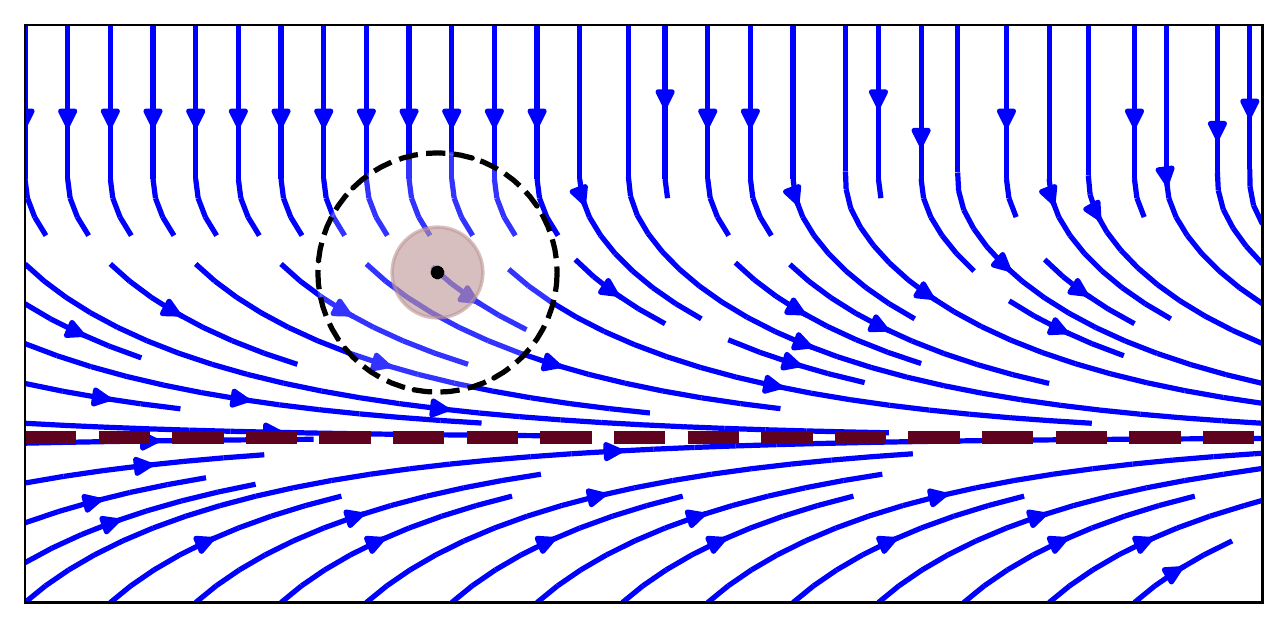}
    \caption{Nonlinear dynamics to follow the line at $y=0$ in red}
    \label{fig:nonlinear_infinite_dynamics_initial}
  \end{subfigure}
  \begin{subfigure}{.9\columnwidth}
    \centering
    \includegraphics[width=\textwidth]{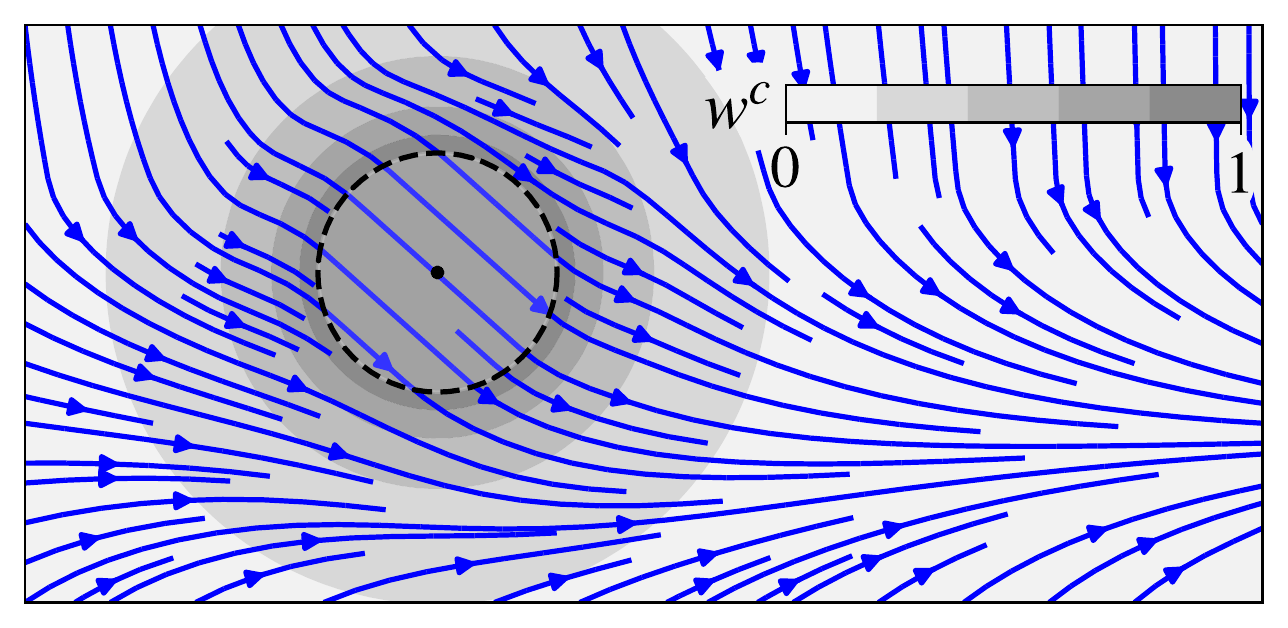}
    \caption{Convergence dynamics $\vect c(\vecs \xi)$ are locally straight in the gray subset}
    \label{fig:nonlinear_infinite_dynamics_convergence}
  \end{subfigure}
  \begin{subfigure}{.9\columnwidth}
    \centering
    \includegraphics[width=\textwidth]{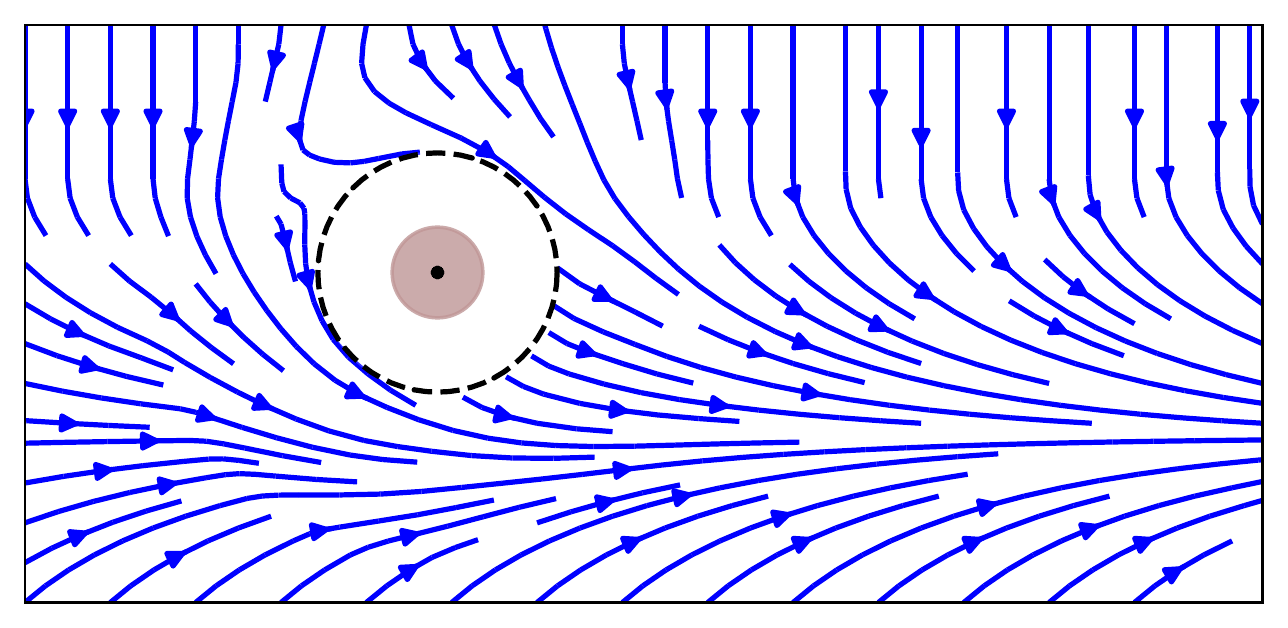}
    \caption{Rotated avoidance without local minima}
    \label{fig:nonlinear_infinite_dynamics_avoidance}
  \end{subfigure}
  \caption{The line following dynamics $\vect f(\vecs \xi) = [1 \;\; -\vecs\xi_2]^T$ shown in (a) uses locally straight convergence dynamics (b) to ensure the absence of local minima when avoiding the obstacle (c).}
  \label{fig:nonlinear_infinite_dynamics}
\end{figure}

\subsection{Nonlinear Motion in the Presence of a Stationary Point} 
Many initial dynamics can be characterized by a motion with a single directional singularity point, i.e., $\left\{ \vecs \xi: \| \vect f(\vecs \xi) \| = 0, \nabla \vect f(\vecs \xi) \neq 0  \right\} = \{ \vecs \xi^a \}$.
 This singularity point may correspond to a desired attractor point or arise from a limit cycle where an unstable stationary point resides at its center.
In the presence of such a singularity point, the directional summing approach employed in \eqref{eq:unbounded_convergence_dynamics} cannot be directly applied. This is because we cannot smoothly define local rotations around $\vecs \xi^a$. However, we can overcome this limitation by \textit{unfolding} the space using a sequence of mappings (Fig.~\ref{fig:convergence_direction}).

\subsubsection{Shrinking and Inverse Mapping}
The first step is to shrink the obstacles to a single point, i.e., all boundary points of the obstacle $\mathcal{X}^b$ are mapped to the reference point $\vecs{\xi}^r$. 

\begin{definition}[Shrinking Mapping] \label{def:shrinking_mapping}
The shrinking mapping $\vect m^s({\vecs \xi}) : \mathcal{X}^e \rightarrow \mathbb{R}^N \setminus \vecs \xi^r$ defined as 
\begin{equation}
  \vect m^s({\vecs \xi}) = {\vecs \xi}^r +  \vect r({\vecs \xi}) \left(\|{\vecs \xi} - {\vecs \xi}^r \| - 
  \| \vecs \xi^b - \vecs \xi^r\|
  \right)
  \;\; \forall \vecs \xi \in \mathcal{X}^e
	\label{eq:shrinking_mapping}
\end{equation}
is a bijection, and it maps the point on the obstacle's surface to the obstacle's reference direction, i.e., $\lim_{\Gamma(\vecs \xi) \rightarrow 1} \vect m^s({\vecs \xi}) \rightarrow {\vecs \xi}^r$.
\end{definition}

Analogously, we define the inverse of this mapping:
\begin{definition}[Inflating Mapping] \label{def:inflating_mapping}
The inflating mapping $\vect m^i({\vecs \xi}): \mathbb{R}^N \setminus {\vecs \xi}^r \rightarrow \mathcal{X}^e$ defined as:
\begin{equation}
	\vect m^i({\vecs \xi}) = {\vecs \xi}^r +  \vect r({\vecs \xi}) \left( \|{\vecs \xi} - {\vecs \xi}^r \| +
        \| \vecs \xi^b - \vecs \xi^r\|
    \right)
	\;\; 
	\forall {\vecs \xi} \neq {\vecs \xi}^r
	\label{eq:inflating_mapping}
\end{equation}
is a bijection and the inverse function of the shrinking mapping defined in \eqref{eq:shrinking_mapping}, i.e., $\vecs m^i \circ \vecs m^s(\vecs \xi) = \vecs \xi \;\; \forall \vecs \xi \in \mathcal{X}^e$
\end{definition}
The effect of shrinking and inflation mapping can be observed in Figure~\ref{fig:convergence_direction}.

\subsubsection{Folding Mapping}
Let us introduce a \textit{folding} mapping $\vect m^f(\vecs \xi)$, which moves the stationary point infinitely far away. Hence, the dynamics in the mapped space are without directional singularity point, which can be treated with the method introduced in Section~\ref{sec:nonlinear_dynamics}. 
Conversely, the surface of the obstacle should not be affected by the mapping.
The desired properties of this folding mapping are given as follows:
\begin{enumerate}
\item the stationary point gets mapped infinitely far away
  \begin{equation}
    \lim_{\vecs \xi \rightarrow \vecs \xi^a} \Gamma\left( \vect m^f({\vecs \xi}) \right) \rightarrow \infty
    \label{eq:opposite_region}
  \end{equation}
\item at the reference point (which represents the surface of the obstacle after the shrinking mapping), the effect of the mapping vanishes
  \begin{equation}
    \vecs \xi^r = \vect m^f({\vecs \xi^r}) \label{eq:stationary_reference}
  \end{equation}
 \end{enumerate}

Let us first define the unit directions in the coordinate system, which has its center at the stationary point and the first axis points towards the reference direction of the obstacle:
\begin{equation}
  \hat {\vecs \xi} = \matr{B}^T (\vecs \xi - \vecs \xi^a) / \| \vecs \xi - \vecs \xi^a \|
  \label{eq:unit_directions}
\end{equation}
where $\matr{B}$ is the orthonormal matrix of which the first row aligns with ${\vecs \xi}^r - {\vecs \xi}^a$

The desired constraints of the mapping from \eqref{eq:opposite_region} and \eqref{eq:stationary_reference} can be achieved by uniformly stretching along dimensions $ i \in [2 .. N]$ of $\hat{\vecs \xi}$ as:
\begin{equation}
  \begin{split}
  \vect s_{i}  = \left( 2 / (1 + p) - 1 \right)^g \frac{\hat{\vecs \xi}_i}{\|\hat{\vecs \xi}_{[2:N]} \|}
  \quad 
  p \in ] -1, \, 1] , \; \vecs \xi \neq \vecs \xi^a \\
  \text{with} \qquad
  p = \normdotprod{\vecs \xi - \vecs \xi^a}{\vecs \xi^r - \vecs \xi^a}
  \label{eq:folding_perpendicular}
  \end{split}
\end{equation}
where $g \in \mathbb{R}_{>0}$ is the power factor, we choose $g = 2$.
The above equation pushes points opposite the attractor relative to the obstacle infinitely far away, i.e., $\lim_{\vect{\vecs \xi}_1 \rightarrow -1} \vecs s_i \rightarrow \infty$, see the green line in Figure~\ref{fig:convergence_direction}.

Finally, the stretching of the \textit{folding} mapping along dimension $i = 1$ is constructed as follows:
\begin{equation}
  \vect s_1 = \| {\vecs \xi}^r - {\vecs \xi}^a \| \left(1 + \ln  \left( \frac{\|{\vecs \xi} - {\vecs \xi}^a \|}{\| {\vecs \xi}^r - {\vecs \xi}^a\| } \right) \right)
  \label{eq:folding_tangent}
\end{equation}

Using this, the folding mapping $\vect m^f({\vecs \xi}) : \{{\vecs \xi} \in \mathbb{R}^N : \hat{\vecs \xi}_1 \neq -1, \vecs \xi \neq \vecs \xi^a \} \rightarrow \mathbb{R}^N$ is defined as:
\begin{equation}
  \vect m^f({\vecs \xi}) = \matr{B} \, \text{diag}(\vect s) \matr{B}^T (\vecs \xi - \vecs \xi^a) + \vecs \xi^a
  \quad \forall {\vecs \xi} \in \mathbb{R} \setminus {\vecs \xi}^a
  \label{eq:folding_mapping}
\end{equation}

\begin{figure}[tbh]
	\centering
	\includegraphics[width=1.0\columnwidth]{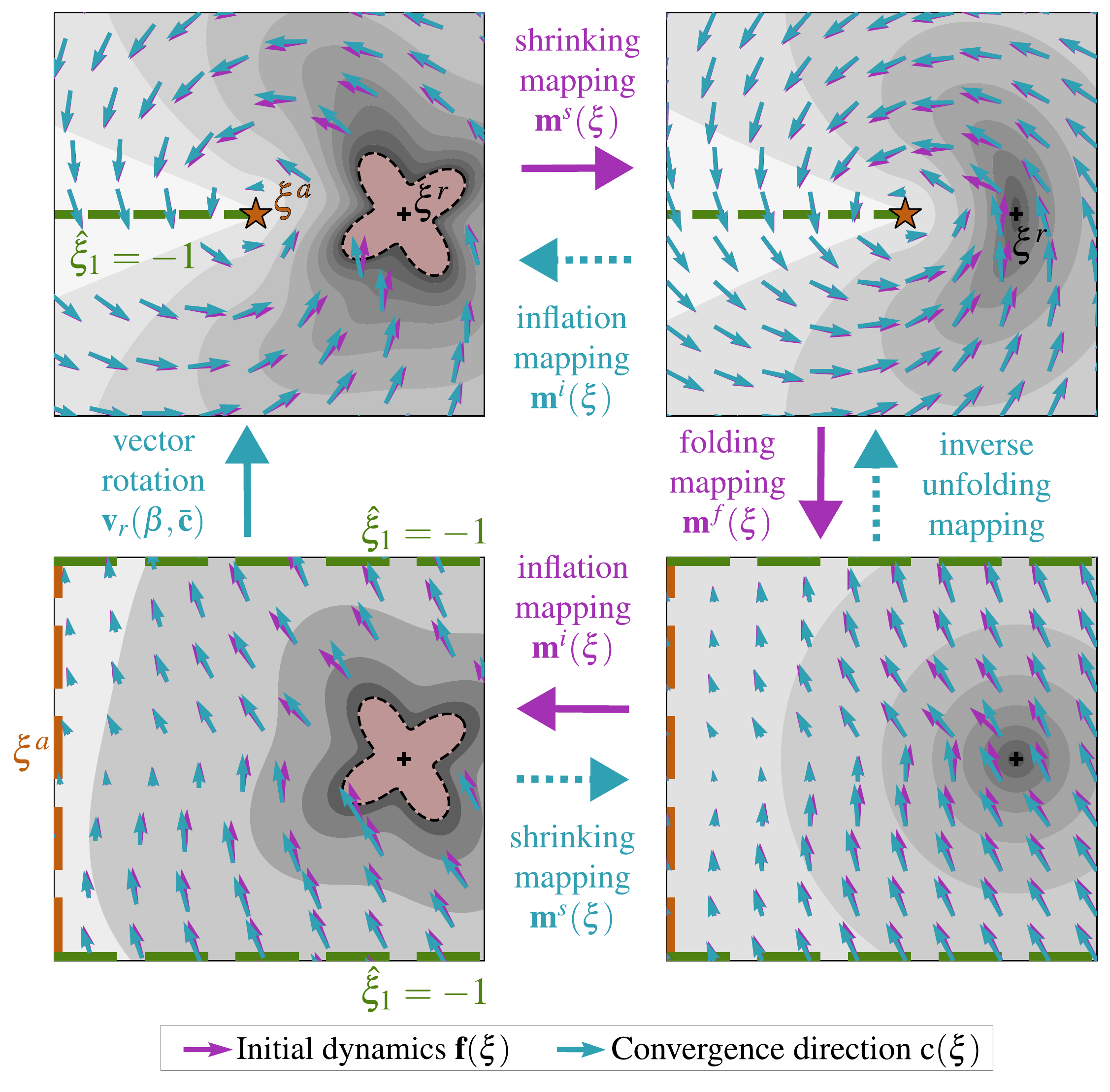}
	\caption{
     The initial dynamics $\vect f(\vecs \xi)$ undergo a series of three consecutive mappings, depicted in a clockwise manner starting from the top left.
      The convergence weight $w^c(\vecs \xi)$ is evaluated in the transformed space, where its value is represented by the gray shading. Notably, the weight diminishes to zero at the attractor point $\vecs \xi^a$.
      The directional summing operation from \eqref{eq:unbounded_convergence_dynamics} is performed to obtain the convergence dynamics $\vect c (\vecs \xi)$. This step is carried out in the mapped space.
        To obtain the reverse mapping, the vector rotation $\vect v_r(\cdot)$ is simply evaluated in the original space.
      The folding mapping $\vect m^f(\vecs \xi)$ maps the attractor infinitely far away in the $- \hat{\vecs \xi}_1$ direction, whereas the line with $\hat{\vecs \xi}_1 = -1$ is folded out to be infinitely far in the directions of $\pm \hat{\vecs \xi}_{[2 .. d]}$.
    }
	\label{fig:convergence_direction}
\end{figure}

\begin{lemma} \label{lem:folding_mapping}
  The mapping $\vect m^f(\vect \xi) : \vect m^f({\vecs \xi}) : \{{\vecs \xi} \in \mathbb{R}^N : \hat{\vecs \xi}_1 \neq -1, \vecs \xi \neq \vecs \xi^a \} \rightarrow \mathbb{R}^N$ as given in \eqref{eq:folding_mapping} is a bijection,
  and smoothly defined, i.e., $\lim_{{\vecs \xi} \rightarrow {\vecs \xi}_j} \vect m^f({\vecs \xi}_i) = \vect m^f({\vecs \xi}_j)$.
  Furthermore, the mapping has no effect at the reference point, i.e., $\vect m^f({\vecs \xi}^r) = {\vecs \xi}^r$, and the attractor is mapped infinitely far away,
  i.e., $\lim_{\vecs \xi \rightarrow \vecs \xi^a} \Gamma\left( \vect m^f({\vecs \xi}) \right) \rightarrow \infty$.
\end{lemma}

\begin{proof}
  The system is made up of the two transformations, one along the ${\vecs \xi}^r - {\vecs \xi}^a$ as given in \eqref{eq:folding_tangent}, and one perpendicular to, given in in \eqref{eq:folding_perpendicular}.
  Since all the underlying functions involved are smooth and bijective, the resulting transformation is also smooth and bijective.
  However, it is important to note that there are discontinuities at ${\vecs \xi} =  {\vecs \xi}^a$ and $\hat{\vecs \xi}_1=-1$. This points are excluded from the input set (in Eq.~\ref{eq:mapping_weight} we will additionally ensure that the corresponding weight goes to zero, for a smooth effect).
  
  Furthermore for an input value of ${\vecs \xi}^r$ in \eqref{eq:folding_tangent}, we get that $\vect s_1 = \| {\vecs \xi}^r - {\vecs \xi}^a\|$, and using \eqref{eq:folding_perpendicular} we get $\vect s_i = 0$ with $ i \in [2 .. d ]$. Hence, the transformation in \eqref{eq:folding_mapping} yields,
  \begin{equation}
    \begin{split}
    \vect m^f({\vecs \xi}^r) & = \matr{B} \; \text{diag}(\vect s) \matr{B}^T (\vecs \xi^r - \vecs \xi^a) + \vecs \xi^a  \\
    & = \matr{B} \; \text{diag}(\vect s) \begin{bmatrix} 1 & 0 & .. & 0\end{bmatrix}^T + \vecs \xi^a \\
    & = \matr{B} \; \begin{bmatrix} \| \vecs \xi^r - \vecs \xi^a \| & 0 & .. & 0\end{bmatrix}^T + \vecs \xi^a  \\
    &= \vecs \xi^r - \vecs \xi^a + \vecs \xi^a = \vecs \xi^r
    \end{split}
  \end{equation}
  For the region $\hat{\vecs \xi}_1 = -1$, we get for the stretching factors $s_i \rightarrow \infty$ for $d \in [2 .. d]$ from \eqref{eq:folding_perpendicular}.
\end{proof}

\subsubsection{Evaluation of the Relative Rotation}
The total mapping can be written as follows:
\begin{equation}
  \vect m(\vecs \xi) = \vect m^i \circ \vect m^f \circ \vect m^s (\vecs \xi)
  \label{eq:total_mapping}
\end{equation}

Since this mapping transforms the attractor infinitely far away, it produces a vector field without directional singularity, as Section~\ref{sec:nonlinear_dynamics} requires. Thereof, we can evaluate the convergence dynamics in the mapped space:
\begin{equation}
  \bar{\vect c}(\vecs \xi) =  w^m \circ \vect m(\vecs \xi) \vect f(\vecs \xi^r) \hat{+}  \bigl( 1 - w^m \circ \vect m(\vecs \xi) \bigr) \vect f({\vecs \xi})
  \label{eq:mapped_convergence_direction}
\end{equation}
the symbol $\hat +$ implies the directional summing as defined in Appendix~\ref{sec:perpendicular_rotation}. 

The mapping weight is defined as:
\begin{equation}
  w^m(\vecs \xi) = 1 / \sqrt{ \bigl(\Gamma(\vecs \xi) - 1\bigr) \bigl(\Gamma(m(\vecs \xi)) - 1\bigr) + 1}
  \label{eq:mapping_weight}
\end{equation}
This weight function is designed to ensure a decreasing influence as the distance from the obstacle increases, as well as a decreasing influence when the position is opposite to the singularity point relative to the obstacle (green line in Figure~\ref{fig:convergence_direction}). Additionally, it ensures that the weight approaches one when the position lies on the surface of the obstacle.

Finally, the mapping into the original space can be made through a vector rotation
\begin{equation}
\vect c(\vecs \xi) = \vect v_r \left(\beta, \bar{\vect c} \right)
\label{eq:attractor_convergence_direction}
\end{equation}
where the vector rotation $\vect v^r$ and angle $beta$ are obtained according to Eq.~\eqref{eq:vector_basis}, using input vector $\vect v_i = \vect m(\vecs \xi) - \vecs \xi^a$ and output vector $\vect v_o = \vecs \xi - \vecs \xi^a$.

\begin{theorem} \label{theorem:convergence_direction}
  The smoothly defined convergence dynamics $\vect c(\vecs \xi)$ as given in \eqref{eq:attractor_convergence_direction} for obstacles in a vector field with a directional singularity point $\vecs \xi^a$, i.e., $\left\{ \vecs \xi: \| \vect f(\vecs \xi) \| = 0, \nabla \vect f(\vecs \xi) \neq 0  \right\} = \{ \vecs \xi^a \}$ is ensured to be straight according to Definition~\ref{def:straight_dynamics} on the surface and inside the obstacle, i.e., $\vect c(\vecs \xi) \in \mathcal{F}^s, \; \vecs \xi \in \mathcal{X}^b \cup \mathcal{X}^i$.
\end{theorem}

\begin{proof}
 Lemma~\ref{lem:folding_mapping} states that the unfolding has no effect at $\vecs \xi^r$ in the shrunk space, which is the surface of the obstacle in the original space as stated in Definition~\ref{def:shrinking_mapping}. 
 Further, the inflating mapping is the inverse of the shrinking as stated in Definition~\ref{def:inflating_mapping}. 
  Thus, when approaching the surface, i.e., $\Gamma(\vecs \xi) \rightarrow 1$, Lemma~\ref{lemma:unbounded_dynamics} transfers to Theorem~\ref{theorem:convergence_direction}.
\end{proof}

\subsubsection{Dynamical Systems with Multiple Attractors}
A dynamical system can exhibit multiple attractors \cite{luo2007dynamical}. However, if the system is continuous across space, there must exist a region known as the \textit{saddle boundary}. Points starting on the \textit{saddle boundary} do not converge to any specific attractor, and this boundary divides the space into regions that contain only a single attractor each. By interpreting the \textit{saddle boundary} as an inverted obstacle, as described in Section~\ref{sec:inverted_obstacles}, the method presented in this section can be applied to the local regions. This allows for effective obstacle avoidance in the presence of multiple attractors, ensuring that the system remains within specific regions corresponding to individual attractors.

\subsubsection{Obstacles Across the Attractor}
When an obstacle spans across the attractor $\vecs \xi^a$, i.e., $\Gamma({\vecs \xi}^a) \leq 1$, the unfolding mapping is not defined. Hence, the convergence dynamics $\vect c (\vecs \xi)$ must approach globally straight dynamics according to Definition~\ref{def:straight_dynamics}. To account for this, the initial dynamics are updated as follows:
\begin{equation}
\vect f(\vecs \xi^r) \gets w^\Gamma \text{sign}\left(\nabla \vect f(\vecs \xi) |_{\vecs \xi = \vecs \xi^a} \right)  ( \vect \xi^r - \vect \xi^a) +  (1 - w^\Gamma) \vect f(\vecs \xi^r)
\end{equation}
The weight is designed to reach one when the obstacle reaches the attractor, e.g., $w^\Gamma = 1 / \Gamma(\vecs \xi^a)$.

\section{Multi-Obstacle Environments} \label{sec:multi_obstacle_environment}
In the presence of multiple obstacles, the rotational obstacle avoidance method for a single obstacle can be extended using a weighted rotational summing.
First, the weighted convergence dynamics $\vect c(\vecs \xi)$ as introduced in \eqref{eq:unbounded_convergence_dynamics} is averaged as follows:
\begin{equation}
\vect c(\vecs \xi) = \vect f(\vecs \xi) \hat{+} {\sum_{o=1}^{N^{obs}}} w_o \vect c_o(\vecs \xi)
= \vect f(\vecs \xi) \hat{+} {\sum_{o=1}^{N^{obs}}} w_o w^c_o \vect f(\vecs \xi^r_o)
\label{eq:multi_obstacle_convergence}
\end{equation}
where $\hat{+}$ denotes the rotational summing as defined in Appendix~\ref{sec:perpendicular_rotation}.

The obstacle weights have been proposed in \cite{huber2022avoiding} as:
\begin{equation}
w_o(\vecs \xi) = \frac{\tilde w_o(\vecs \xi)}{\sum_{i=1}^{N^{\mathrm{obs}}} \tilde w_i(\vecs \xi)}
\;\; \text{with} \;\;
w_o(\vecs \xi) = \frac{1}{\Gamma_o(\vecs \xi)} \quad \forall \, \vecs \xi \in \mathcal{X}^e
\label{eq:multi_obstacle_weights}
\end{equation}
The weights associated with each obstacle ensure that their sum is at most one and converge to zero as the distance from the respective obstacle increases. Furthermore, on the surface of an obstacle $o$, the corresponding weight equals 1. Moreover, the weighting allows for overlapping of the influence regions of the obstacles.

These convergence dynamics are used to evaluate the preferred pseudo tangent for each obstacle $\vect e_o(\xi)$ as defined in \eqref{eq:tangent_calculation}. Finally, the rotation of the initial dynamics from \eqref{eq:pulling_weight} can be restated for multi-obstacle scenarios as:
\begin{equation}
    \dot{\vecs \xi} = \vect f(\vecs \xi) \hat{+} {\sum_{o=1}^{N^{obs}}} w_o \dot{\vecs \xi}_o
= \vect f(\vecs \xi) \hat{+} {\sum_{o=1}^{N^{obs}}} \lambda_o w_o m(\vecs \xi) \vect e_o(\vecs \xi)
\label{eq:multi_obstacle_dynamics}
\end{equation}

The method is summarized in Algorithm~\ref{alg:rotational_avoidance} and handles star-shaped obstacles and boundaries, as shown in Figure~\ref{fig:multi_obstacle_avoidance}. 

\begin{figure}[tbh]
    \centering
    \includegraphics[width=0.8\columnwidth]{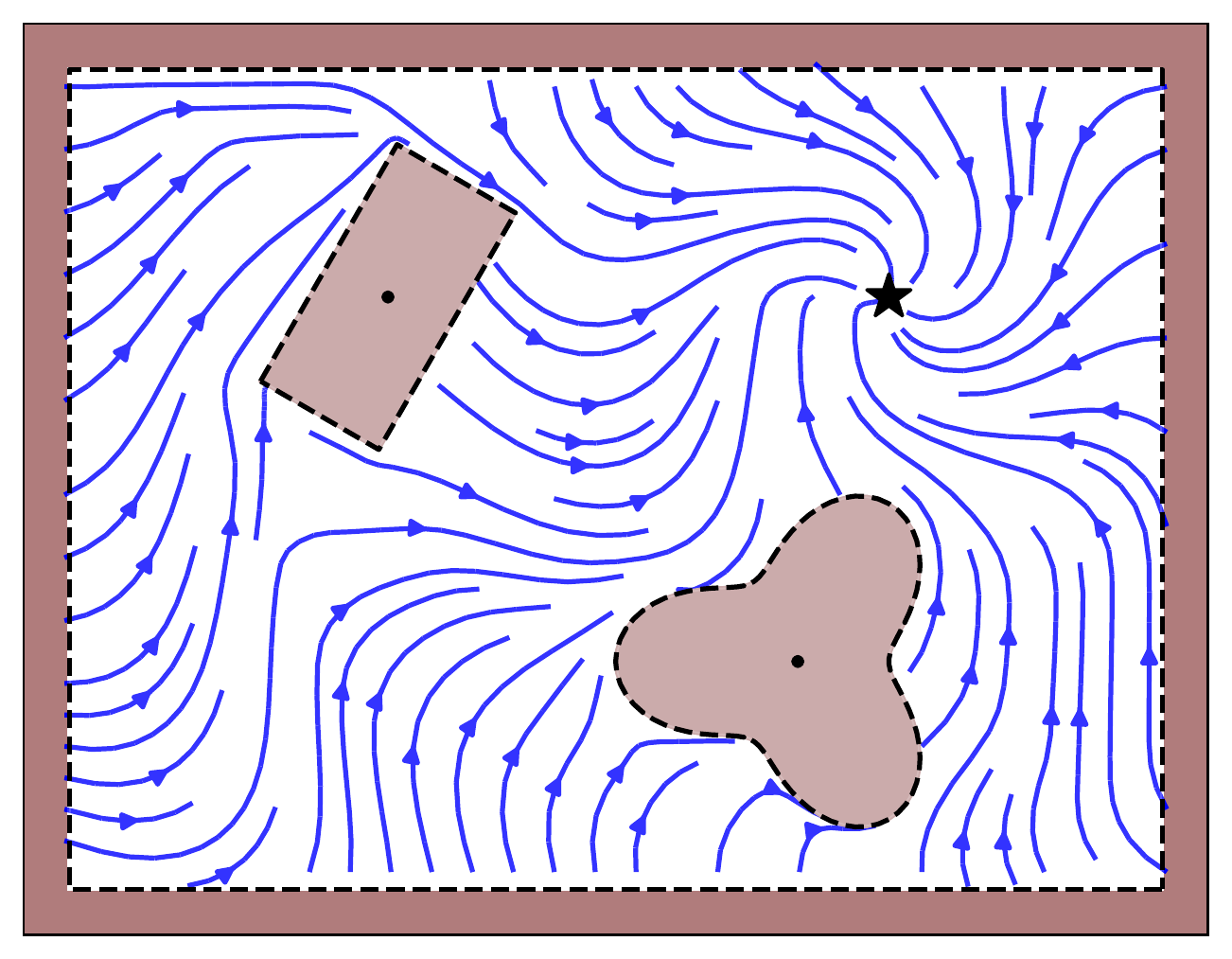}
    \caption{The rotational obstacle avoidance with respect to wavy dynamics as used in Figure~\ref{fig:starshaped_wall} moves towards the attractor (black star) while avoiding collisions inside a rectangular hull with two obstacles.}
    \label{fig:multi_obstacle_avoidance}
\end{figure}

\begin{lemma} \label{lemma:multi_obstacle_weighting}
The weighted rotation of the convergence dynamics $\vect c(\vecs \xi)$ from \eqref{eq:multi_obstacle_convergence} and the final dynamics $\dot{\vecs \xi}$ from \eqref{eq:multi_obstacle_dynamics} conserves impenetrability and saddle-point properties introduced in Theorem~\ref{theorem:collision_avoidance}.
\end{lemma}

\begin{proof}
When approaching an obstacle $o$, the weight defined in \eqref{eq:multi_obstacle_weights} results in $\lim_{\Gamma_o(\vecs \xi) \rightarrow 1} w_o(\vecs \xi) = 1$. Hence the evaluation of the multi-obstacle avoidance converges to the single obstacle scenario of obstacle $o$. It follows that the properties of the single obstacle case are conserved locally.
\end{proof}

\begin{algorithm}[ht]
\caption{Rotational Obstacle Avoidance Method (ROAM)} \label{alg:rotational_avoidance}
\begin{algorithmic}[1]
  \renewcommand{\algorithmicrequire}{\textbf{Input:}}
  \renewcommand{\algorithmicensure}{\textbf{Output:}}
  \REQUIRE $\vect f ({\vecs \xi})$, $N^{\mathrm{obs}}$ obstacles
  \ENSURE $\dot {\vecs \xi}$
  \FOR{$o = 1$ \textbf{to} $N^{\mathrm{obs}}$}
    \STATE $\tilde w_o(\vecs \xi) =  1 / \Gamma_o(\vecs \xi)$ \COMMENT{Evaluate weights \eqref{eq:multi_obstacle_weights}}
  \ENDFOR
  \STATE $ w_o(\vecs \xi) = \tilde w_b(\vecs \xi) / \sum_{i=1}^{N^{\mathrm{obs}}} \tilde w_i(\vecs \xi)$ \COMMENT{Normalize weights \eqref{eq:multi_obstacle_weights}}
  \FOR{$o = 1$ \textbf{to} $N^{\mathrm{obs}}$ \textbf{if} $w_o(\vecs \xi) > 0$}
      \STATE $\vect f(\vecs \xi^r_o)$  \COMMENT{Compute DS at reference point}
  \ENDFOR
  \STATE $\vect c(\vecs \xi) = \vect f(\vecs \xi) \hat{+} {\sum_{o=1}^{N^{obs}}} w_o w^c_o \vect f(\vecs \xi^r_o)$  \COMMENT{Rotational average \ref{eq:multi_obstacle_convergence}}
  \FOR{$o = 1$ \textbf{to} $N^{\mathrm{obs}}$ \textbf{if} $w_o(\vecs \xi) > 0$}
    \STATE ${\vect e_o}(\vecs \xi)$ \COMMENT{Compute pseudo tangent \eqref{eq:tangent_calculation}}
    \STATE $h(\vecs \xi)$ \COMMENT{Compute magnitude \eqref{eq:magnitude_scaling}}
    \ENDFOR
 \STATE $\dot{\vecs \xi} = \vect f(\vecs \xi) \hat{+} {\sum_{o=1}^{N^{obs}}} \lambda_o w_o m(\vecs \xi) \vect e_o(\vecs \xi)$  \COMMENT{Rot. average \eqref{eq:multi_obstacle_dynamics}}
\end{algorithmic}
\end{algorithm}

\subsection{Dynamic Environments}v
The closed-form description of the algorithm enables short computation time without the need for offline trajectory planning. Dynamic obstacles can be considered by transposing the avoided direction into the moving reference frame:
\begin{equation}
\dot{\vecs \xi} = \matr M(\vecs \xi) \vect f(\vecs \xi)  + \dot{\tilde{\vecs \xi}} 
\quad \text{with} \quad
\dot{\tilde{\vecs \xi}} = \sum_{o = 1}^{N^{\mathrm{obs}}} w_o \dot{\tilde{\vecs \xi}}_o
\end{equation}
The method's application to dynamics obstacles is experimentally evaluated in Section~\ref{sec:robot_implementation}.

\subsection{Motion within Multiple Enclosing Hulls}
The outer boundary  might not be star-shaped, for example, when a robot moves through a curvy corridor.
However, there such a general space can be divided into multiple stars-shapes \cite{lindemann2009simple}, which can be used by ROAM to guide a motion to stay within the boundary. 
This requires convergence dynamics, which transition between the obstacles as:
\begin{equation}
\vect c(\vecs \xi) = \vect f(\vecs \xi) \hat{+} {\sum_{b=1}^{N^{bnd}}} w_b(\vecs \xi) \vect c_b(\vecs \xi)
\label{eq:multi_boundary_convergence}
\end{equation}
where $N^{bnd} \in \mathbb{N}_{>0}$ is the number of boundaries.
Furthermore, the weights of the multi-boundary environment are set to:
\begin{equation}
w_b(\vecs \xi) = \frac{\max\left(\Gamma_b(\vecs \xi), 1) \right) - 1}{\sum_{i=1}^{N^{bnd}} \max\left(\Gamma_i(\vecs \xi), 1\right) - 1}
\label{eq:multi_boundary_weight}
\end{equation}
Where the local dynamics  $\vecs c_b(\vect \xi)$ are locally straight according to Definition~\ref{def:straight_dynamics} in the subdomain of the surface of each boundary. The  attractor of each hull $\vecs \xi^a_b$ is placed such that it lies outside of the corresponding boundary $b$ for all boundaries that do not contain the global attractor $\vecs \xi^a$, i.e.,
\begin{equation}
\begin{cases}
\vecs \xi^a_b = \vecs \xi^a & \text{if} \;\; \Gamma_b(\vecs \xi^a) > 1 \\
\vecs \xi^a_b \in \mathcal{X}^i_b \cup \mathcal{X}^b_b  & \text{otherwise}
\end{cases}
\end{equation}
The boundary encapsulation can be seen in Figure~\ref{fig:multihull_dynamics}.

\begin{figure}[tbh]\centering
  \begin{subfigure}{.9\columnwidth}
   \centering
   \includegraphics[width=\textwidth]{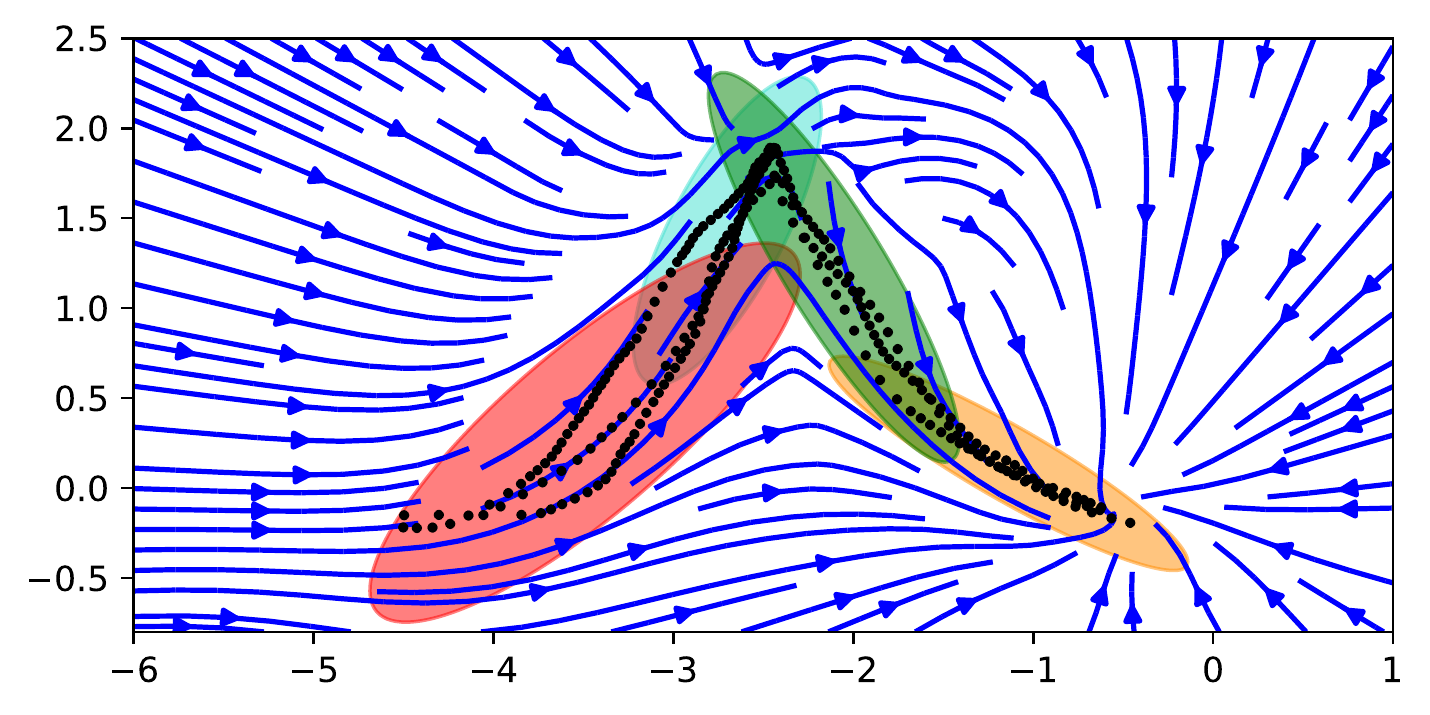}
   \caption{The data (position and velocity) are fit using Gaussian Mixture Model (GMM) with four components \cite{scikit-learn}, which is used to predict the velocity.}
   \label{fig:multihull_avoidance_2D_Ashape_learned}
  \end{subfigure}
  \begin{subfigure}{.9\columnwidth}
    \centering
  \includegraphics[width=\textwidth]{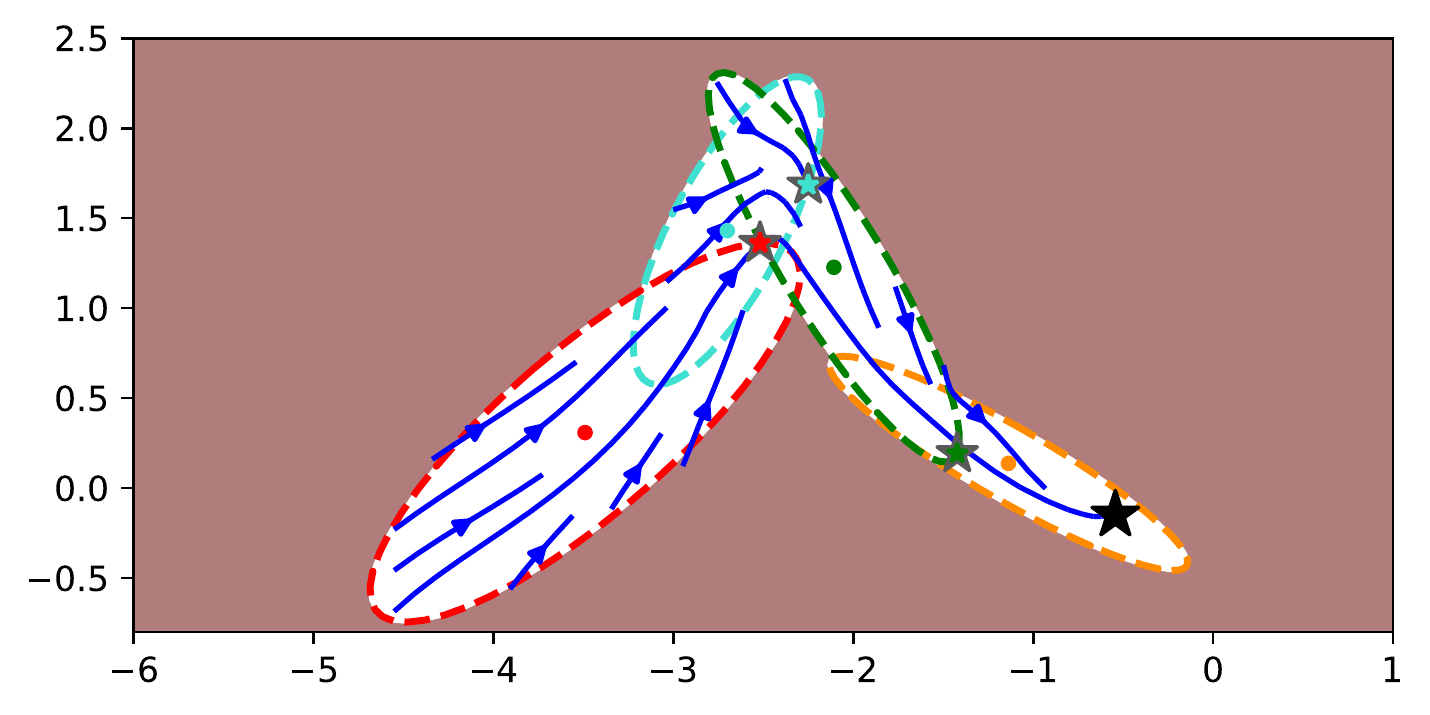}
  \caption{The GMM components are used as a multi-hull environment to bind the dynamics and enforce leaving the boundary in the direction of the local attractors $\vecs \xi^a_b$ (colored stars).}
  \label{fig:multihull_avoidance_2D_Ashape_enclosed}
  \end{subfigure}
\caption{
  Approaches for learning nonlinear motion from demonstration can guarantee stability by ensuring that the system asymptotically converges towards the attractor $\vecs \xi^a$ \cite{perrin2016fast}. However, they do not prohibit the dynamics of taking an undesired shortcut or moving away from the data, which is the known region (a).
  Convex hulls can be obtained through trajectory learning methods, by interpreting the Gaussian Mixture Model applied to the data as ellipse-shaped obstacles \cite{figueroa2018physically}. ROAM can enforce the final dynamics to stay close to the data (b).}
  \label{fig:multihull_dynamics}
\end{figure}

\begin{lemma}
  A motion starting within multiple boundaries $\{ \vect \xi \}_0 \in \bigcup_{b \in N^{bnd}} \mathcal{X}^e_b$ and evolving according to $\dot {\vecs \xi}$ as described in \eqref{eq:multi_obstacle_dynamics} and with convergence dynamics $\vect c(\vecs \xi)$ defined in \eqref{eq:multi_boundary_convergence} will stay within the boundaries, i.e, $\{ \vect \xi \}_t \in \bigcup_{b \in N^{bnd}} \mathcal{X}^e_b \;\; \forall \; t \in \mathbb{N}_{>0}$.
\end{lemma}

\begin{proof}
If $\vecs \xi$ is within multiple boundaries, and it approaches boundary $b$, the weight given in in \eqref{eq:multi_boundary_weight} evaluates to:
\begin{equation}
    \lim_{\Gamma_b(\xi) \rightarrow 1}  w_b(\vecs \xi) = \frac{ 0 }{0  + \sum_{i\neq b}^{N^{bnd}} \max \left(\Gamma_b(\vecs \xi), 1) \right) - 1} = 0
\end{equation}
Hence, boundary $b$ has no effect, and the motion can traverse any boundary $b = 1 .. N^{\mathrm{bnd}}$.
Until, $\vecs \xi$ is within only one boundary $b$. In this case, the corresponding boundary weights simplify to:
\begin{equation}
w_b(\vecs \xi) = \frac{\max \left(\Gamma_b(\vecs \xi), 1) \right) - 1}{\max\left(\Gamma_b(\vecs \xi), 1) \right) - 1 + \sum_{i\neq b}^{N^{bnd}} 0} = 1
\end{equation}
given that $\Gamma_b(\vecs \xi) > 1$. Thus, the algorithm evolves according to the single-boundary case, i.e., collision avoidance with the boundary $b$ is ensured.
\end{proof}


\section{General Concave Obstacles}\label{sec:concave_obstacles}
A general obstacle can be described as a union of multiple star-shaped obstacles \cite{rimon1991construction}, also referred to as trees of stars. 
Extending the algorithm to such shapes enables navigating in many more environments.
Let us for this introduce a general obstacle using nomenclature from graph theory \cite{knuth1997art}:

\begin{definition}[Tree of Obstacles]
    A tree of star-shaped obstacles represents a shape without holes. Each obstacle (node) in the tree can have multiple children (successors), but have exactly one single parent (predecessor), except for the root obstacle which does not have a parent. 
    All obstacles have a non-zero intersection with their parent and children. Obstacles are assigned a level $l$ in the tree, starting from $l=0$ at the root.
\end{definition}


\subsection{Velocity Propagation through Obstacle Tree}
The avoidance velocity $\dot{\vecs \xi}$ in the presence of trees-of-obstacles is obtained through the summed average of a rotation-tree as described in Appendix~\ref{sec:perpendicular_rotation}. The tree is constructed as described in Algorithm~\ref{alg:tree_of_star_avoidance}, and the individual steps are detailed below.

\begin{algorithm}[ht]
\caption{Avoidance of Tree-of-Obstacles} \label{alg:tree_of_star_avoidance}
\begin{algorithmic}[1]
  \renewcommand{\algorithmicrequire}{\textbf{Input:}}
  \renewcommand{\algorithmicensure}{\textbf{Output:}}
  \REQUIRE $\vect f (\cdot)$, $N^{\mathrm{com}}$ obstacle-components
  \ENSURE $\dot {\vecs \xi}$
  \STATE $ \vect t^r(\cdot)$ \COMMENT{Create rotation tree, see Algo.~\ref{alg:rotation_summing_tree}}
  \FOR{$o = 1 \; \textbf{to} \; N^{\mathrm{com}}$}
  \STATE $\vect s_{0} = \vect r(\vecs \xi)$ \COMMENT{Set initial surface point}
  \STATE $w^h_o \gets \Gamma_o(\vect s_{0}) $ \COMMENT{Compute hiding-weight \eqref{eq:hiding_weight}}
  \STATE $w_o \gets \Gamma_o(\vecs \xi)$ \COMMENT{Compute obstacle weights}
  \STATE $c = o$ \COMMENT{Initialize $c$ to obstacle $o$}
  \FOR [From node to root] {$l = l(n) \; \textbf{to} \; 0 $}
  \STATE $\vect s_{p(c)} = b ( \vecs \xi^r_c - \vect s_c) + \vect s_c$ \COMMENT{Propagate reference  \eqref{eq:surface_point_propagation}}
  \STATE $c = p(c)$ \COMMENT{Set iterator $c$ to parent $p$}
  \ENDFOR
  \STATE $\vect f_0 = \vect f(\vecs \xi^r)$ \COMMENT{Set initial tangent}
  \FOR [From root to node] {$c = 1 \; \textbf{to} \; l(n)$}
  \STATE $\vecs v^{r}(\cdot), \beta \gets (\vecs s_{c} - \vecs \xi^r_c), (\vecs s_{p(c)} - \vecs \xi^r_{p(c)})$ \COMMENT{Get rot. \eqref{eq:vector_basis}}
  \STATE $\vect f_{c} = \vecs v^{r}(\vect f_{p(c)}, \beta)$ \COMMENT{Propagate rotated velocity \eqref{eq:vector_rotation_application}}
  \STATE $\vecs f_{c} \Rightarrow \vecs t^{r}(\cdot)$ \COMMENT{Append to tree}
  \ENDFOR
  \STATE $\dot{\vecs \xi}_o \gets \vect f_o$ \COMMENT{Avoidance with propagated velocity \eqref{eq:tangent_evaluation_for_tree}}
  \ENDFOR
  \STATE $w_f = 1 - \sum_o w^h_o w_o$  \COMMENT{Weight of initial velocity $\vecs f(\vecs \xi)$}
  \STATE $\dot{\vecs \xi} \gets \vecs t^r(w_f, w^h_o(\vecs \xi))$ \COMMENT{Evaluate rotation tree, Algo.~\ref{alg:rotation_summing_tree}}
\end{algorithmic}
\end{algorithm}

\subsubsection{Surface Point Propagation} \label{sec:surface_propagation}
The rotational avoidance of each obstacle $o = 1 .. N^{\mathrm{com}}$ of the tree is obtained at the corresponding surface point $\vect s_o$.
The surface points are obtained by propagating the position $\vecs \xi$ through the obstacle tree, starting from the an obstacle $o$ down to the root $r$:
\begin{equation}
  \vect s_{p} = b ( \vecs \xi^r_c - \vect s_c) + \vect s_c
  \quad \text{such that} \quad \Gamma_p(\vect s_{p}) = 1 \label{eq:surface_point_propagation}
\end{equation}
where $p$ is the parent of the component $c$.
The factor $b \in \mathbb{R}_{>0}$ is evaluated such that $\vect s_{p}$ lies on the parent's surface.
Note, that the obstacles are intersecting, hence we have $\Gamma(\vect s_{p}) < 1$. See the surface points of a three-component tree in Figure~\ref{fig:triple_ellipses_obstacle}.

\subsubsection{Velocity Propagation} \label{sec:velocity_propagation}
We can now propagate the velocity $\vect f(\vecs \xi^r)$ iteratively from parent $p(c)$ to the component $c$ until we reach the respective obstacle $o$. The iteration starts at the root $r$ and is done for each obstacle, except the root.
The propagated velocity $\vect f_c$ is obtained as follows:
\begin{equation}
  \vect f_{c} = \vecs v^{r}(\vect f_{p(c)}, \phi_c) 
  \qquad
  o = 1.. N^{\mathrm{obs}} \setminus r
  \label{eq:velocity_propagation}
\end{equation}
where $\vect v^r(\cdot)$ is the vector rotation as described in \eqref{eq:vector_rotation_application}.
The vector rotation is obtained with respect to the vectors $\vect v_i = (\vecs s_{c} - \vecs \xi^r_c)$ and $\vect v_o = (\vecs s_{p(c)} - \vecs \xi^r_{p(c)})$ as described in \eqref{eq:vector_basis}.
When constructing the obstacle tree, the parent of an obstacle needs to be chosen such that $\vect v_i$ and $\vect v_o$ are not anti-collinear. This is ensured if the direction from the component to the parent is never opposing the direction from the parent to the \textit{grand-parent} (parent of the parent): 
\begin{equation}
 \normdotprod{\vecs \xi^r_{p(c)}- \vecs \xi^r_c}{\vecs \xi^r_{p(p(c))}- \vecs \xi^r_{p(c)}} \neq - 1 \label{eq:parent_opposing}
\end{equation}

\begin{figure}[tbh]\centering
\includegraphics[width=0.9\columnwidth]{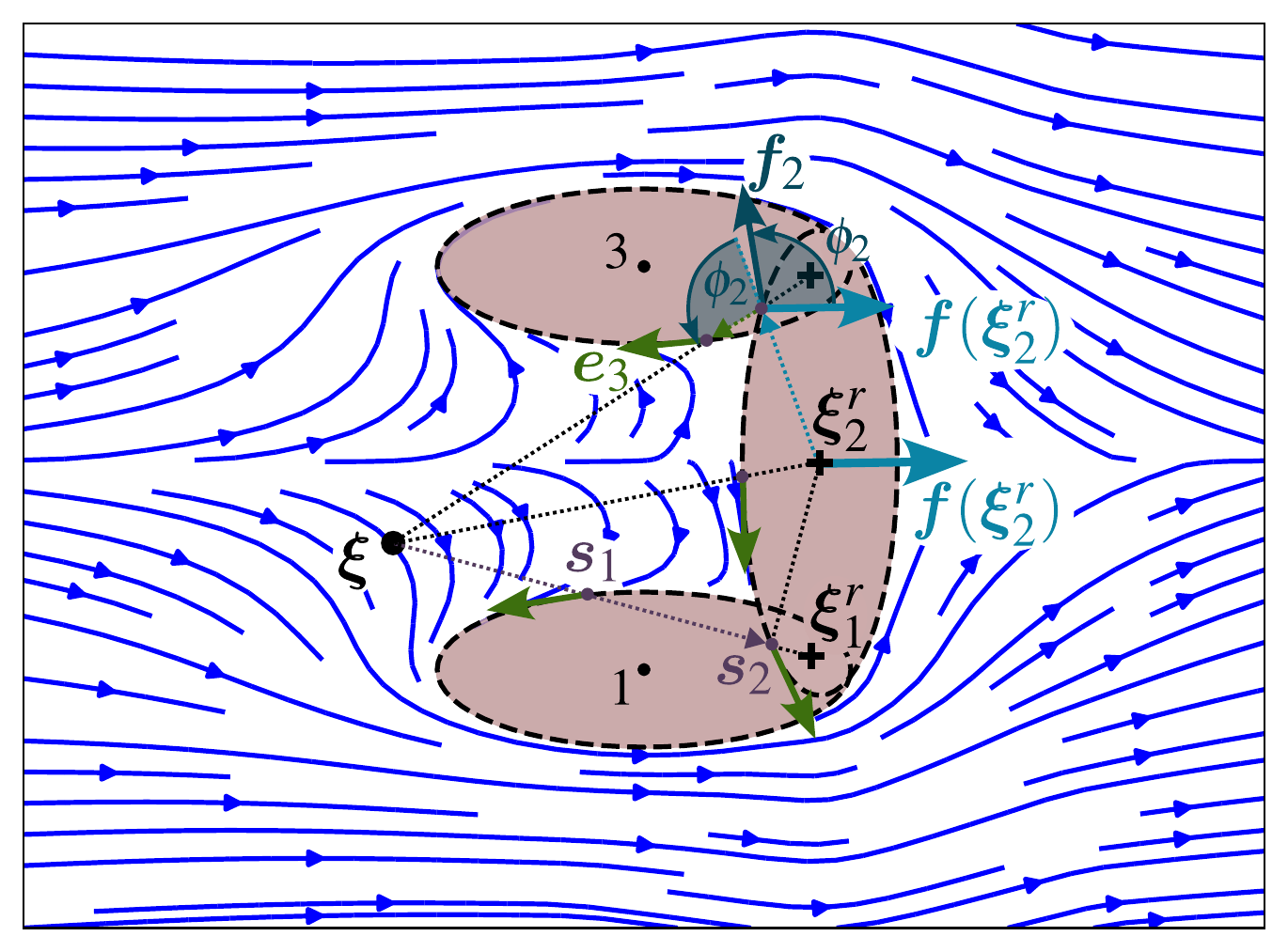}
  \caption{The surface points $\vect s_{(\cdot)}$ (purple) are evaluated for each obstacle up to the root, to then propagate the desired tangent velocity $\vect e$ (green) from the root to each obstacle.}
  \label{fig:triple_ellipses_obstacle}
\end{figure}

\subsubsection{Tangent Evaluation} \label{sec:tangent_evaluation}
Finally, after propagating the velocity to the obstacle $o$, the pseudo tangent is evaluated. For trees-of-stars, we want to enforce the pseudo tangent to be parallel to the surface at all times, hence we adopt \eqref{eq:pulling_weight} as follows:
\begin{equation}
  \begin{split}
  & \text{if} \quad  \normdotprod{- \vect n}{\vect f} > \cos(R^e):  \\
  & \qquad \vect k \left(\sminus \vect n, {\vect e_o} \right) = \vect k \left( \sminus \vect n, {\vect r} \right)+ b \left( \vect k \left(\sminus \vect n,   \vect e_p \right) - \vect k \left( \sminus \vect n, \vect r \right) \right) \\
  & \qquad \text{such that} \quad \| \vect k \left(\sminus \vect n, \vect e_p \right) \| = R^e, \quad b \in \mathbb{R}_{>0} \\
  & \text{otherwise}  \\
  & \qquad \vect k \left(\vect n, {\vect e}_o \right) = \vect k \left( \vect n, - {\vect r} \right)+ b \left( \vect k \left(\vect n,   \vect e_p \right) - \vect k \left( \vect n, - \vect r \right) \right)  \\
  & \qquad \text{such that} \quad \| \vect k \left(\vect n, \vect e_p \right) \| = \pi - R^e, \quad b \in \mathbb{R}_{>0}
  \end{split}
  \label{eq:tangent_evaluation_for_tree}
\end{equation}

\subsubsection{Hiding Weights}
An obstacle $o$ which is occluded by its parent $p(o)$ should not influence the avoidance velocity. For this reason, we introduce the \textit{hiding weight} which reaches zero at full occlusion:
\begin{equation}
  \begin{split}
  w^h_o =
  \begin{cases}
    1 & \text{if} \;\; \Gamma(\vect{s}_{o}) > 1 \\
    \Gamma(\vect{s}_o) ^{\frac{1}{1 - b}} & \text{if} \;\; b < 1\\
    0 & \text{otherwise}
  \end{cases}
  b = \normdotprod{\vecs \xi - \vecs \xi^r_o}{\vecs \xi^r_{p(o)} - \vecs \xi^r_o}
  \end{split}
  \label{eq:hiding_weight}
\end{equation}

\begin{lemma} \label{lemma:rotation_tree}
  Let us assume a tree with obstacle components $o = 1 ..N^{\mathrm{com}}$ and the corresponding reference points $\vecs \xi^r_o$, which all lie within obstacle $o$ and the corresponding parent $p$, i.e., $\vecs \xi^r_o \in \mathcal{X}^i_o \cap \mathcal{X}^i_{p(o)}$. Conversely, the reference of each parent lies outside of the obstacle, i.e.,  $\vecs \xi^r_{p(o)} \in \mathcal{X}^i_{p(o)} \setminus \mathcal{X}^i_{o}$. Moreover, the direction from obstacle to parent is never opposing the direction from the parent to the grandparent as described in \eqref{eq:parent_opposing}.
  Let us assume the dynamics $\vect f(\vecs \xi)$ are locally straight in the surrounding of the obstacle according to Definition~\ref{def:straight_dynamics}. The vector field $\dot{\vecs \xi}$ obtained through the propagation of the velocity $\vect f(\vecs \xi)$ as described in \eqref{eq:velocity_propagation}, with the rotation of the final velocity given by \eqref{eq:tangent_evaluation_for_tree}, and vector tree summing using the weights $w^h_o$ given in \eqref{eq:hiding_weight} ensures collision avoidance according to the boundary condition \eqref{eq:boundary_condition} with the absence of local minima in free-space. 
\end{lemma}

\begin{proof}
  As shown in Lemma~\ref{lemma:pseudo_tangent}, the tangent is defined everywhere as long as the reference direction $\vect r_o(\vecs \xi)$ is not parallel to the propagated velocity $\vect f_o(\vecs \xi)$.
  Moreover, smoothness is ensured due to the adaptable rotation weight $\lambda(\vecs \xi)$ as shown in Theorem~\ref{theorem:collision_avoidance}.
  
  Furthermore, the rotation defined in \eqref{eq:velocity_propagation} is smoothly defined for obstacle-parent-pair, due to the fact that the reference point of the parent is required to lay outside of the child obstacles, and the parent-opposing inequality from \eqref{eq:parent_opposing}. 
  
  This is also ensured for the last surface point $s_o$, i.e., the projection of the position on the obstacle's surface since the hiding weight $w^h_o$ goes to zero if the two vectors are opposing:
  \begin{equation}
    \normdotprod{\vecs s_{c} - \vecs \xi^r_c}{\vecs s_{p(c)} - \vecs \xi^r_{p(c)}}
    \;\; \Rightarrow \;\;
    b = 1
    \;\; \Rightarrow \;\;
    w^h = 0
  \end{equation}
  Hence, the corresponding tangent and the vector can be omitted.

Since all weights and corresponding vectors are smoothly defined, according to Lemma~\ref{lemma:direction_tree} the vector tree evaluation leads to continuous and minima-free dynamics.

Furthermore, from Lemma~\ref{lemma:multi_obstacle_weighting} we know that the impenetrability of the obstacles is preserved as a single component dominates when approaching the corresponding surface. 
  \end{proof}


\begin{figure}[tbh]\centering
\includegraphics[width=0.99\columnwidth]{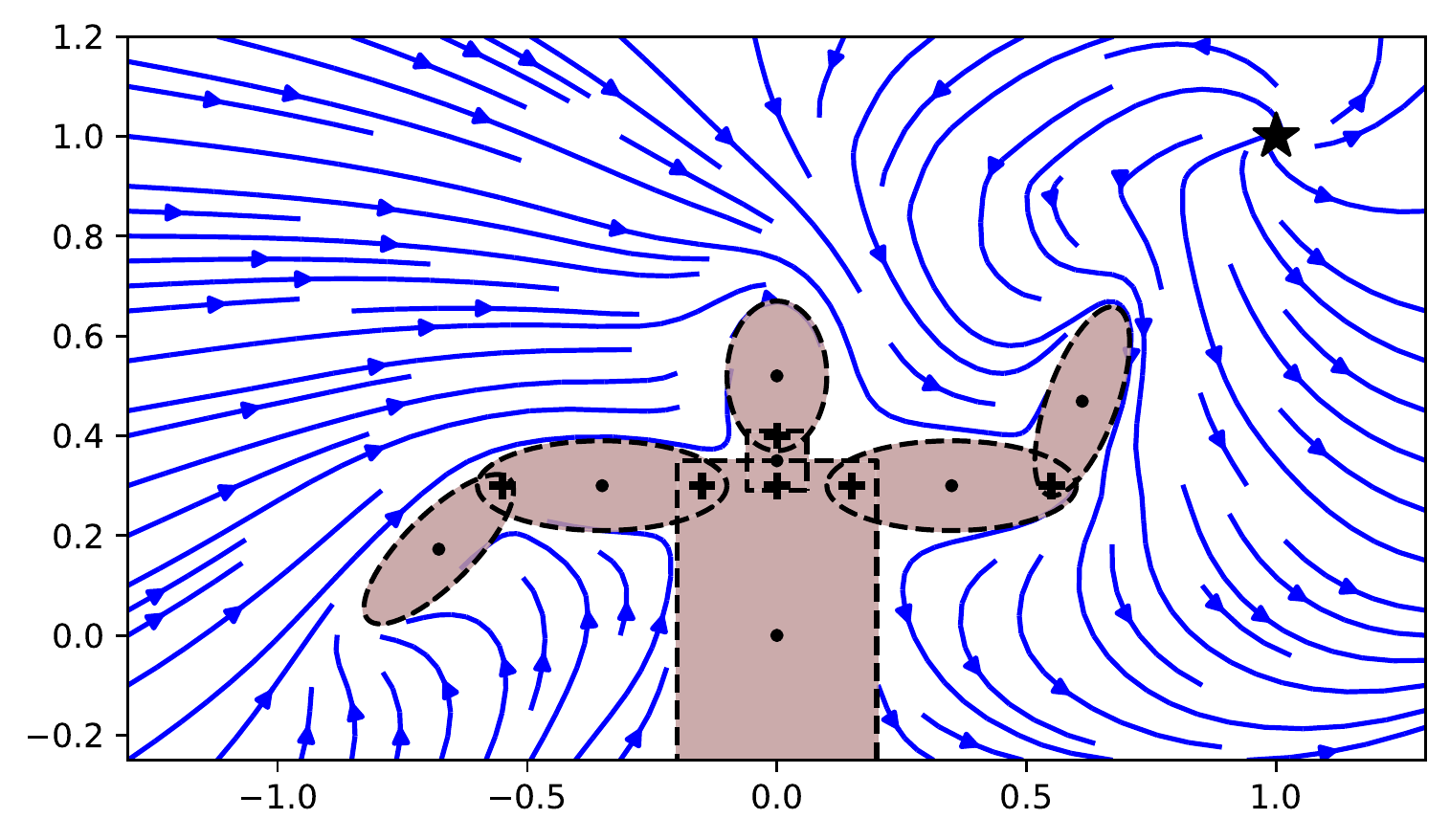}
  \caption{Obstacle avoidance of a two-dimensional human with a total of seven sub-obstacles (corresponding reference points as black crosses) of a circular motion with a single stationary point (black star).}
  \label{fig:rotated_dynamics_multibody_human}
\end{figure}

\subsection{Convergence Sequence}
The obstacle avoidance algorithm is dependent on locally straight dynamics on the surface of the tree-of-obstacles.
Since the general shapes described can be encapsulating directional singularity points, such as attractors, the design of convergence direction requires special care. The computation has to ensure smoothness when getting closer to the directional singularity point $\vecs \xi^a$. It is detailed in Algorithm~\ref{alg:convergence_tree_of_star_avoidance}.

\begin{algorithm}[ht]
\caption{Convergence Direction for Tree-of-Obstacles. The abbreviation pred. refers to the predecessor of the direction tree.} \label{alg:convergence_tree_of_star_avoidance}
\begin{algorithmic}[1]
  \renewcommand{\algorithmicrequire}{\textbf{Input:}}
  \renewcommand{\algorithmicensure}{\textbf{Output:}}
  \REQUIRE $\vect f ({\vecs \xi})$, $N^{\mathrm{com}}$ obstacle-components
  \ENSURE $\vect c(\vecs \xi)$ {Locally straight dynamics}
  \STATE $w^m_1(\vecs \xi)$ \COMMENT{Mapping weight for root component \eqref{eq:mapping_weight}}
  \STATE $\vecs f(\vecs \xi), (\vecs \xi - \vecs \xi^a) \Rightarrow \vecs v^{r, i}(\cdot),  \beta_i$ \COMMENT{Initial rotation \eqref{eq:vector_basis}}
  \STATE $\vecs v^{r, c}(\cdot),  \beta_c \gets \vecs f(\vecs \xi^r), (\vecs \xi^r - \vecs \xi^a)$ \COMMENT{Convergence rot. \eqref{eq:vector_basis}}
  \STATE $\vecs v^{r, a}(\cdot),  \beta_a \gets (\vecs \xi - \vecs \xi^a), (\vecs \xi^r - \vecs \xi^a)$ \COMMENT{Get rotation \eqref{eq:vector_basis}}
  \STATE $\vect t^r_1(\cdot) \gets \vecs v^{r, i}(\cdot), \vecs v^{r, a}(\cdot), \vecs v^{r, c}(\cdot)$ \COMMENT{Create rotation tree}
  \STATE \COMMENT{Tree reduction with Algo.~\ref{alg:rotational_avoidance} using vector-weight pairs:} $\vect t_1^r \left([\vecs f(\vecs \xi^r), w^m_1], [\vecs f(\vecs \xi), (1 -w^m_1)]\right) \Rightarrow \vect t^r(\cdot)$  
  \FOR[Iteration over components]{$n = 1 \; \mathbf{to} \; N^{\mathrm{com}}, i \neq r$}
  \STATE $ c \gets n$ \COMMENT{Set initial node}
  \STATE $w^m_i(\vecs \xi)$ \COMMENT{Mapping weight \eqref{eq:mapping_weight}}
  \FOR[From node to root]{$l=l(n) \; \mathbf{to} \; 0$}
  \STATE $(\vecs \xi^r_{p(c)} - \vecs \xi^r_c) \Rightarrow \vect t^r(\cdot)$ \COMMENT{Append to tree with pred. $c$}
  \STATE $ c \gets p(c)$ \COMMENT{Set current node iterator to its parent}
  \ENDFOR
  \STATE $\vect v^{r, c}(\cdot) \Rightarrow \vect t^r(\cdot)$ \COMMENT{Append to tree with pred. $c=r$}
  \ENDFOR
  \STATE $w^n_o \gets w^m_o$ \COMMENT{Weight normalization \eqref{eq:weight_normalization}}
  \STATE $\vecs c(\vecs \xi) \gets \vect t^r(w^n_o)$ \COMMENT{Tree reduction, Algo.~\ref{alg:rotation_summing_tree}}
\end{algorithmic}
\end{algorithm}

\subsection{Mapping Weight Normalization}
A mapping weight $w^m_o$ as introduced in \eqref{eq:weight_normalization} approaching one indicates that the $\vecs \xi$ is on the surface of the corresponding obstacle $o$. To ensure that this weight remains high while taking into account other obstacles, the normalized weights are defined as follows:
\begin{equation}
  w^n_o = 
  \begin{cases}
    \hat w^n_o / \sum_i^{N^{\mathrm{com}}} \hat w^n_i & \text{if} \;\; \sum_i^{N^{\mathrm{com}}} \hat w^n_i > 1 \\
    \hat w^n_o & \text{otherwise} 
  \end{cases}
  \; , \;\;
    \hat w^n_o =  1 / (1 - w_o) \label{eq:weight_normalization}
\end{equation}
for all obstacle $o = 1 ... N^{\mathrm{obs}}$.



\subsection{Practical Considerations}
As for each obstacle, the algorithm needs to propagate the whole obstacle tree. The presence of multiple obstacles and long trees can lead to many computations.
Therefore, it is advised to adapt the distance function $\Gamma_o$ to approach infinity after a certain distance for outer leaves, as proposed in \eqref{eq:infinity_distance}. However, the distance function of the root $\Gamma_r$ should decrease slower, i.e.,
\begin{equation}
  \left\| \frac{d}{d \vecs \xi} \Gamma_o(\vecs \xi) \right\| \geq  \left\| \frac{d}{d \vecs \xi} \Gamma_r(\vecs \xi)  \right\|
\end{equation}
This has the effect that the algorithm only considers the obstacle core far away, while when approaching the obstacle, the more detailed leaf structure is considered.
This leads to a sparse evaluation tree and enables real-time applicability.

\section{Evaluation} \label{sec:results}
\subsection{Computational Cost}
The most computationally intensive part of the ROAM algorithm is the matrix-vector multiplication in $N$ dimensions for the stereographic projections and unfolding mappings. Therefore, the algorithm's complexity, given $O$ obstacles, $K$ components, and an obstacle tree of level $L$, can be expressed as $\mathcal{O}(N^{2} O K L)$.

\subsection{Obstacle Avoidance While Following a Stable Limit Cycle}
\subsubsection{Setup}
We compare the three algorithms MuMo \cite{huber2022avoiding}, VF-CAPF \cite{yao2022guiding}, and ROAM (proposed approach). The chosen scenario uses initial dynamics and the obstacle distribution as proposed in \cite{yao2022guiding}. The initial dynamics represent a circular limit cycle of the form:
\begin{equation}
  \vect f(\vecs \xi)  =
  \left(
  \begin{bmatrix}
    0  & 1 \\
    -1  & 0
  \end{bmatrix}
  + 2 (R_0 - \| \vecs \xi \|) {\matr I}
  \right) \vecs{\xi} \label{eq:circular_motion_vf}
\end{equation}
where the circle radius is $R_0 =2$, and $\matr I$ is the two-dimensional identity matrix.
The environment contains six convex obstacles, and the agent is aware of their location at all times (see Fig.~\ref{fig:quntative_comparison_roam}).
A grid of 10x10 evenly distributed starting points was constructed, of which 93 were in free space. The trajectory is evaluated for these starting points through Euler integration with a time step $dt=0.01 s$ and a maximum of 500 iterations.\footnote{Source code on \url{https://github.com/hubernikus/nonlinear_obstacle_avoidance.git}, 2022/02/31}
  
\begin{figure}[tbh]
  \centering
  \includegraphics[width=0.8\columnwidth]{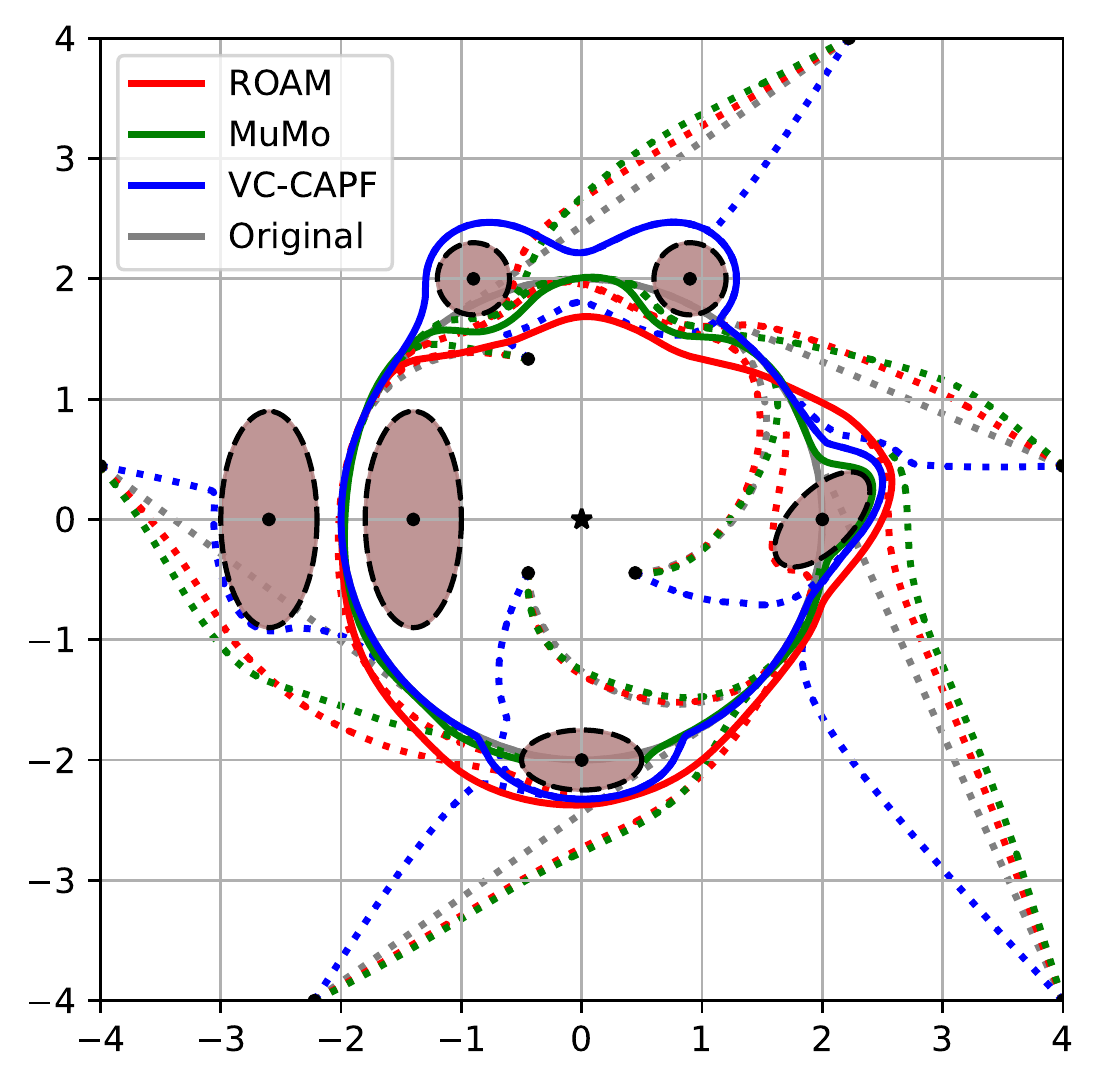}
  \caption{MuMo \cite{huber2022avoiding},  VF-CAPF \cite{yao2022guiding} and ROAM are used to guide a nonlinear, limit-cycle-following vector field around six concave obstacles.
    The resulting limit cycles are visualized as a solid line in the respective color. Additionally, the trajectories from specific starting points (black circles) are visualized with a dashed line.}
    \label{fig:quntative_comparison_roam}
\end{figure}

\subsubsection{Metrics}
The trajectories are compared by observing the local minima, the distance to the desired limit cycle, and the similarity between the velocity after obstacle avoidance and the initial velocity. Furthermore, we look at the evolution of the velocity over time, i.e., the discrete acceleration.

The root mean square error (RMSE), i.e., $\text{RMSE} = \sum^M_{m=1} \| \vect v_i - \vect u_i \|$ is used, as well as the
normalized inverted cosine similarity (NICS)
\begin{equation}
  \mathrm{NICS} = \frac{1}{2} \left(1 -
  \frac{1}{M} \sum_{m=1}^M \normdotprod{\vect v_i}{\vect u_i} \right)
\end{equation}
with $\text{NICS} \in [0, 1]$. Note that the smaller NICS, the higher the similarity of the two vectors.



\subsubsection{Results}
In Table~\ref{tab:comparison_algorithms_limit_cycle}, it can be observed that MuMo is the only method among the three that results in trajectories ending up in local minima on the surface of an obstacle in almost half of the cases. These local minima lie on the limit cycle, suggesting that by increasing the step number, all MuMo trajectories would eventually converge to these local minima. Due to this behavior, MuMo is deemed inappropriate for the proposed scenario. Its limited applicability to nonlinear initial dynamics as pointed out in Table~\ref{tab:comparison_convergence}. 

The focus of the comparison is thus between ROAM and VF-CAPF. ROAM demonstrates better path following, with trajectories maintaining a shorter distance to the reference circle throughout the motion. Moreover, ROAM exhibits lower acceleration along the path, indicating smoother motion with fewer abrupt changes. Additionally, ROAM shows higher similarity to the initial dynamics compared to VF-CAPF.

These findings can be qualitatively observed in Figure~\ref{fig:quntative_comparison_roam}, where VF-CAPF closely follows the initial path (gray) outside the obstacle region and deviates only when in close proximity to the obstacle. The metrics indicate that this behavior leads to higher accelerations and more significant deviations from the initial trajectory overall. On the other hand, MuMO moves more directly towards the limit cycle compared to ROAM, but at the expense of creating local minima on the surface of the bottom obstacle.

\begin{table}[tbh]
  \centering
  \scalebox{0.85}{
    \begin{tabular}{|l|c|c|c||c|}\hline 
      & ROAM & MuMo & VF-CAPF & Original \\
      & (proposed) & &  &  dynamics \\ \hline
      $N^m$ & \textbf{0\%} & 48\% & \bf{0\%} & 0\% \\ \hline
      $\text{RMSE}(\vecs \xi, R_0)$ & 1.46 $\pm$ 1.11 & \bf{1.24 $\pm$ 1.15} & 1.48 $\pm$ 0.87 & 0.93 $\pm$ 0.96 \\ \hline
      $\text{RMSE}(\dot{\vecs \xi}, \vecs f_0)$ & 0.15 $\pm$ 0.07 & \bf{0.04 $\pm$ 0.09} & 0.47 $\pm$ 0.43 & 0.00 $\pm$ 0.00 \\ \hline
      $\text{NICS}(\dot{\vecs \xi}, \vecs f_0)$ & 0.04 $\pm$ 0.02 & \bf{0.01 $\pm$ 0.02} & 0.12 $\pm$ 0.11 & 0.00 $\pm$ 0.00 \\ \hline
      $\text{RMSE}(\dot{\vecs \xi}_t, \dot{\vecs \xi}_{t+1})$ & 2.51 $\pm$ 2.66 & \bf{2.04 $\pm$ 5.98} & 3.47 $\pm$ 2.14 & 0.27 $\pm$ 0.17 \\ \hline
      $\text{NICS}(\dot{\vecs \xi}_t, \dot{\vecs \xi}_{t+1})[10^{-4}]$ & 0.63 $\pm$ 0.66 & \bf{0.51 $\pm$ 1.49} & 1.08 $\pm$ 0.82 & 0.07 $\pm$ 0.04 \\ \hline
    \end{tabular}
    }
    \caption{The different trajectories are compared in (1) the ratio of trajectories which end up in a local minimum on the obstacle $N^m$, (2) the distance to the desired trajectory, i.e., the difference to radius $R_0$. Furthermore, (3) RMSE and NICS between all approaches to the original DS are evaluated, as well as (4) the change of the dynamics over time (corresponds to acceleration). The mean and standard deviation are evaluated over the 93 trajectories.
    }
    \label{tab:comparison_algorithms_limit_cycle}
\end{table}

\subsection{Nonlinear Path Following with Autonomous Wheelchair}

\subsubsection{Experimental Setup}
\def\plotwidth{0.45\columnwidth}
\begin{figure}[t]
  \centering
  \begin{subfigure}{0.90\columnwidth}
    \centering
    \includegraphics[width=\textwidth]{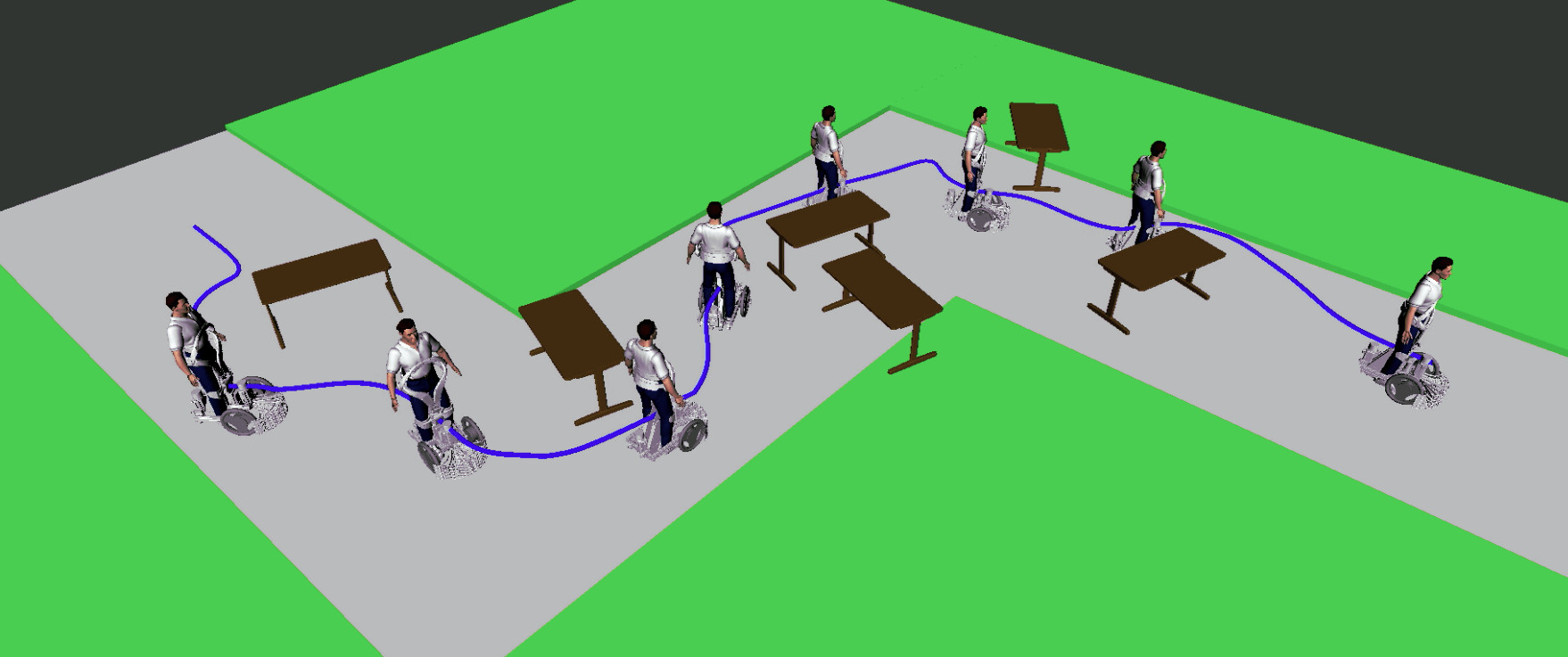}
    \caption{QOLO-robot navigating along a wavy road using global PF}
    \label{fig:qolo_among_tables}
  \end{subfigure} 
  \begin{subfigure}{\plotwidth}
    \centering
    \includegraphics[width=\textwidth]{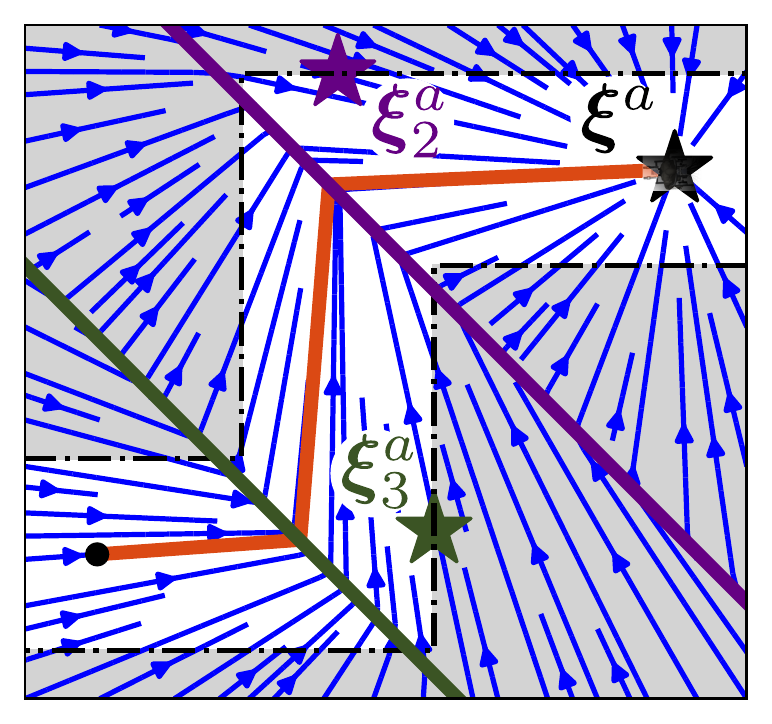}
    \caption{Local straight - initial}
    \label{fig:qolo_wavy_switching_straight_initial}
  \end{subfigure}%
  \begin{subfigure}{\plotwidth}
    \centering
    \includegraphics[width=\textwidth]{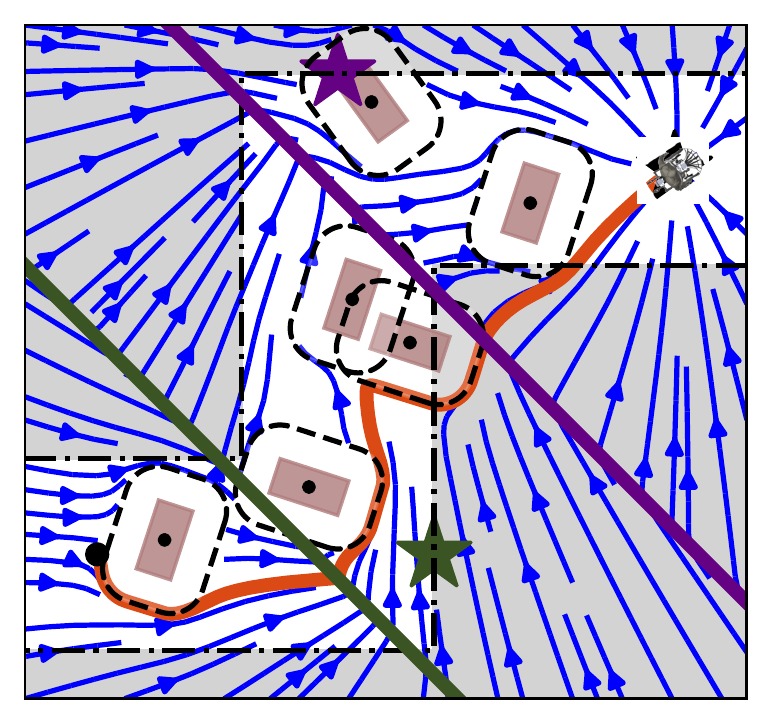}
    \caption{Local straight - avoiding}
    \label{fig:qolo_wavy_switching_straight_avoiding}
  \end{subfigure}
    \begin{subfigure}{\plotwidth}
    \centering
    \includegraphics[width=\textwidth]{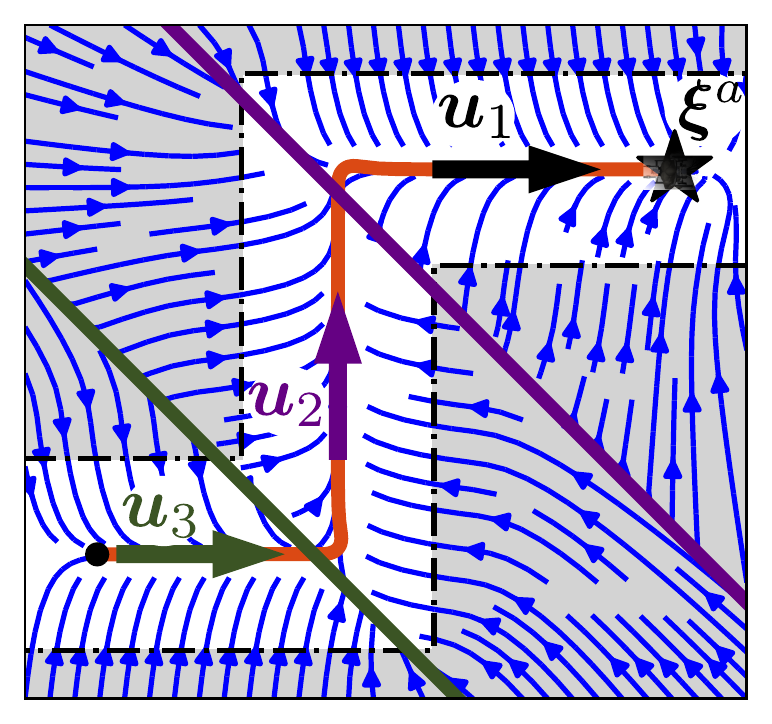}
    \caption{Local PF - initial}
    \label{fig:qolo_wavy_switching_path_initial}
    \end{subfigure}%
  \begin{subfigure}{\plotwidth}
    \centering
    \includegraphics[width=\textwidth]{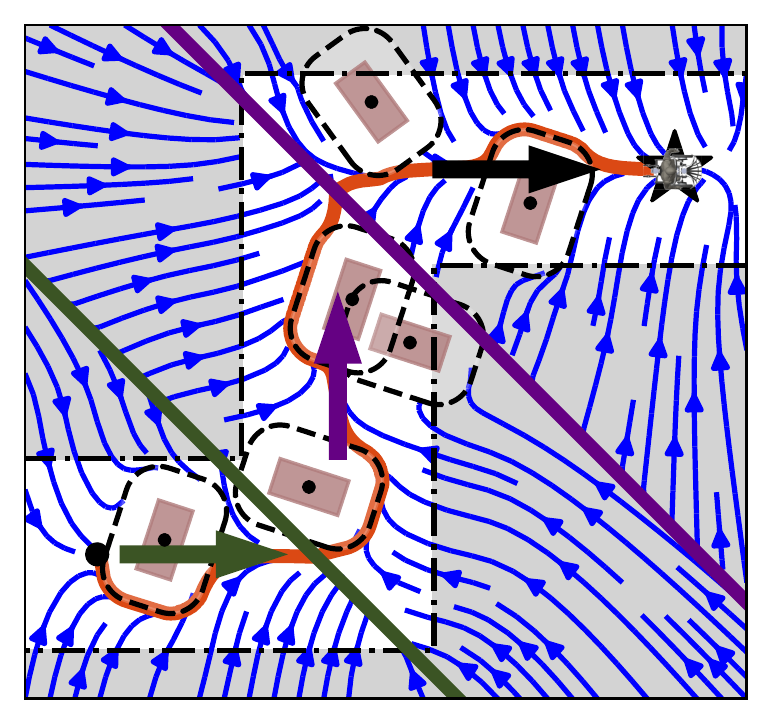}
    \caption{Local PF - avoiding}
    \label{fig:qolo_wavy_switching_path_avoiding}
  \end{subfigure}
  \begin{subfigure}{\plotwidth}
    \centering
    \includegraphics[width=\textwidth]{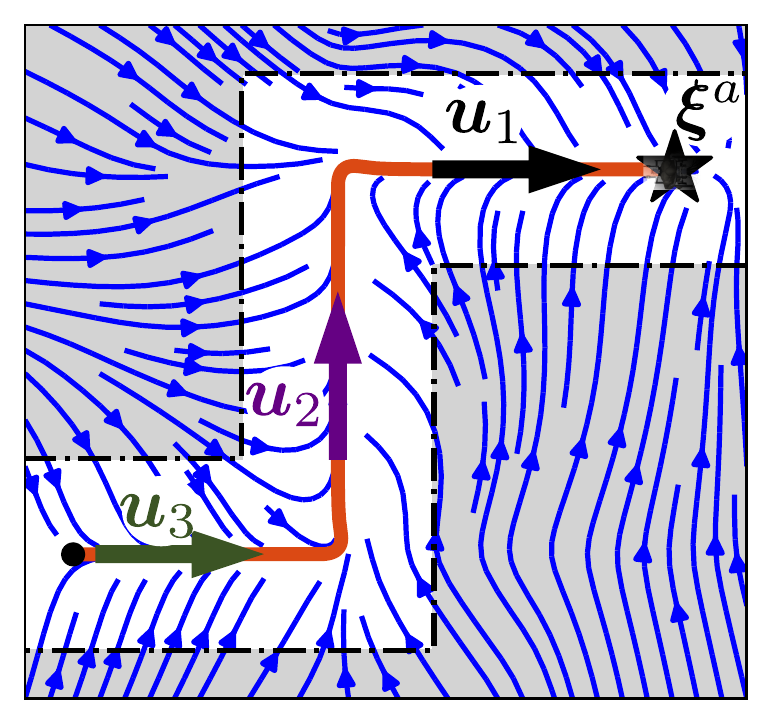}
    \caption{Global PF - initial}
    \label{fig:qolo_wavy_nonlinear_global_initial}
  \end{subfigure}%
  \begin{subfigure}{\plotwidth}
    \centering
    \includegraphics[width=\textwidth]{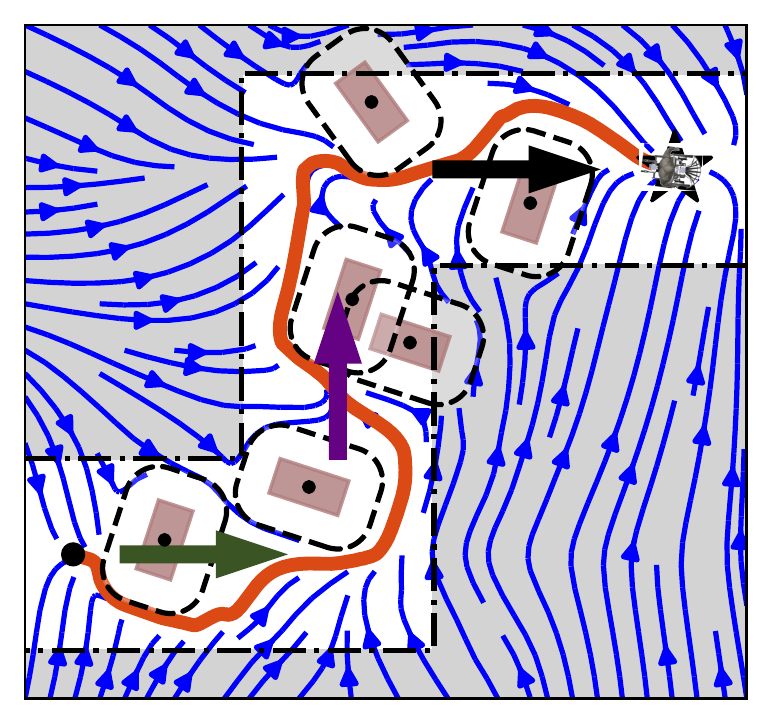}
    \caption{Global PF - avoiding}
    \label{fig:qolo_wavy_nonlinear_global_avoiding}
\end{subfigure}
  \caption{
  The robot is navigating between static tables on a wavy road (gray) through a grass field (green). The initial dynamics and the reference trajectory (orange) are on the left, with the corresponding obstacle dynamics on the right.
  The local attractors  (colored starts) and switching regions (colored lines) are used to create global dynamics (a, c). 
  The global path following (f) uses a single attractor only.
}
    \label{fig:comparison_algorithms}
\end{figure}

In the second evaluation, the autonomous wheelchair QOLO was employed, and various initial vector fields combined with avoidance methods were tested. The control point of the wheelchair was positioned in front of the wheel axis to ensure effective maneuverability. To account for the shape and size of the QOLO wheelchair, a margin of 0.7 meters was added to the obstacles during the evaluation \cite{huber2022avoiding}.

The path followed by the wheelchair consisted of three road segments, denoted as $s=1$, the segment closest to the goal, up to $s=3$. Although there was no strict constraint to remain within the road boundaries, the initial dynamics were specifically designed to guide the wheelchair towards the center of the road, promoting adherence to the desired path (see Fig.~\ref{fig:qolo_among_tables}).

Seven tables were randomly placed along the path, with their centers positioned on the path itself. The tables were positioned away from the starting point and the attractor. Additionally, no more than two tables intersected, including the margin, allowing all table shapes to be represented as star-shapes, as required by MuMo \cite{huber2022avoiding}.

\subsubsection{Navigation Algorithms}
Three approaches are used to navigate in this environment. 

\textbf{Local straight} dynamics \eqref{eq:straight_system} are combined with MuMo \cite{huber2022avoiding}. The initial dynamics consist of three separate vector fields with corresponding local attractors (star), see Fig.~\ref{fig:qolo_wavy_switching_straight_initial}. The attractor switches when transitioning from one region to the next (crossing the line).

\textbf{Local PF} (path following) is combined with ROAM. The local PF dynamics are given as:
\begin{equation}
  \vect f_s (\vecs \xi) = \vect u_s +  \dotprod{\Delta \vecs \xi}{\vect u_s} \vect u_s  - \Delta \vecs \xi
  \;\;\; \text{with} \;\;\;
  \Delta \vecs \xi  = (\vecs \xi - \vecs \xi^a_s)
\end{equation}
for all segments $s = 1 .. 3$, where $\vect u_s \in \mathbb{R}^2$ is the (local) nominal direction pointing along the road segment, and $\vecs \xi^a_s \in \mathbb{R}^2$ is the local attractor. 

\textbf{Global PF} dynamics is evaluated by using directional-tree averaging as described in Appendix~\ref{sec:perpendicular_rotation}.
The root of the direction-tree is given as $\vect v_{0, 1} = \vecs \xi - \vecs \xi^a$. The direction tree is populated iteratively:
\begin{itemize}
    \item $\vect v_{s, 1}(\vecs \xi) = \vect u_s$ with respective parent direction $\vect v_{s-1, 1}$
    \item $\vect v_{s, 2}(\vecs \xi) = \vect f_s (\vecs \xi)$ with respective parent direction $\vect v_{s, 1}$
\end{itemize}
for all segments $s = 1 ..3$.
The final dynamics are obtained through the weighted evaluation described Algorithm~\ref{alg:rotation_summing_tree}, using the following weights
\begin{equation}
  w_{s, 1}(\vecs \xi)= 0 , \;\;  w_{s, 2}(\vecs \xi) = \frac{1}{d_s} (1 + \min(\dotprod{\vecs v_s}{\vecs \xi^a_s - \vecs \xi}, 0))
\end{equation}
where $d_s \in \mathbb{R}_{\geq 0}$ the distance to the line segment $s$. The segment weights $w_s(\vecs \xi) \in [0, 1]$ are normalized if their sum exceeds one.


\subsubsection{Results}
The combination of the global PF with ROAM ensures the convergence of all trajectories, as shown in Table~\ref{tab:comparison_algorithms_path_following}.
The other two methods achieve a convergence rate of approximately 92 \%. This can be attributed to the use of a high-level planner, specifically switching between the local dynamics. Since this conflicts with the guarantees of absences of local minima.
While more sophisticated switching or transitioning methods may exist, to the best of our knowledge, there is no global path sequencer that can guarantee these convergence properties within a finite time.
Furthermore, when using the local potential field (PF) with ROAM, the robot spends less time on the desired path and has a greater average distance to the path boundaries compared to the local straight algorithm combined with MuMo.
However, the average distance traveled remains approximately the same across all methods.


\begin{table}[tbh]
  \centering
    \begin{tabular}{|l|c|c|c|}\hline 
    & Local straight \cite{huber2022avoiding} & Local PF & Global PF  \\ \hline
	Converged [] & 92\% & 92\% & \textbf{100\%} \\ \hline
	Off-track [\%] & 1.49 $\pm$ 0.00 & \bf{0.79 $\pm$ 0.00} & 1.89 $\pm$ 0.00 \\ \hline
    $\Delta d$ [m] & 1.54 $\pm$ 0.04 & \textbf{1.73 $\pm$ 0.02} & 1.65 $\pm$ 0.04 \\ \hline
	Distance [m] & \bf{20.2 $\pm$ 3.4} & 21.0 $\pm$ 3.8 & 20.8 $\pm$ 2.8 \\ \hline
    \end{tabular}
    \caption{
      The three approaches for following the local path are compared based on the following metrics: the convergence ratio to the attractor, the ratio of trajectories deviating from the path, the distance to the road border $\Delta d$, and the total length of the trajectory. The reported values represent the mean and standard deviation calculated from 100 runs with randomly distributed furniture, while keeping the start and endpoints consistent.
    }
    \label{tab:comparison_algorithms_path_following}
\end{table}

\subsection{Obstacle Avoidance in Three Dimensions}

\subsubsection{Spiraling Motion Around Human in Simulation}
Inspired by \eqref{eq:circular_motion_vf}, we propose spiraling dynamics as:
\begin{equation}
	\begin{split}
  \vect f(\vecs \xi)  =
  \matr{B}^T \left(
  \begin{bmatrix}
    0  & 1 \\
    -1  & 0
  \end{bmatrix}
  + 2 \left(R_0 - \| \matr B \tilde{\vecs \xi} \|\right) {\matr I}
\right) \matr{B} \tilde{\vecs{\xi}} + \vecs{p}( {\vecs \xi}) \\
\text{with} \;\; 
\matr B = \begin{bmatrix}
	1 & 0  & 0 \\
 	0 & 0 & 1
 \end{bmatrix} 
 \quad \text{and} \quad 
	\vecs{p}(\vecs \xi) = \begin{bmatrix} 0 &  1 & 0 \end{bmatrix}^T
\end{split}
  \label{eq:spiral_motion_vf}
\end{equation}
with the spiraling radius of $R_0 =$~\SI{0.1}{m} and $\tilde{\vecs \xi} = \vecs \xi - \vecs \xi^a$ the relative position with respect to the center $\vecs \xi^a$.

The obstacle tree representing the human in ROAM consists of components corresponding to the limbs and main body, with the core serving as the root of the tree. When applying ROAM for collision avoidance, see Figure \ref{fig:spiraling_motion}, it can be observed that all trajectories successfully avoid the human without becoming trapped in local minima. Furthermore, the dynamics of the system return to the initial state, far away from the obstacle, both at the beginning and after successfully avoiding the collision. This behavior highlights the effectiveness of ROAM in generating smooth and convergent trajectories while maintaining the desired dynamics of the system.

\begin{figure}[tbh]\centering
\begin{subfigure}{0.49\columnwidth}
\centering
\includegraphics[width=\textwidth, trim={0cm 6cm 0 1cm}, clip]{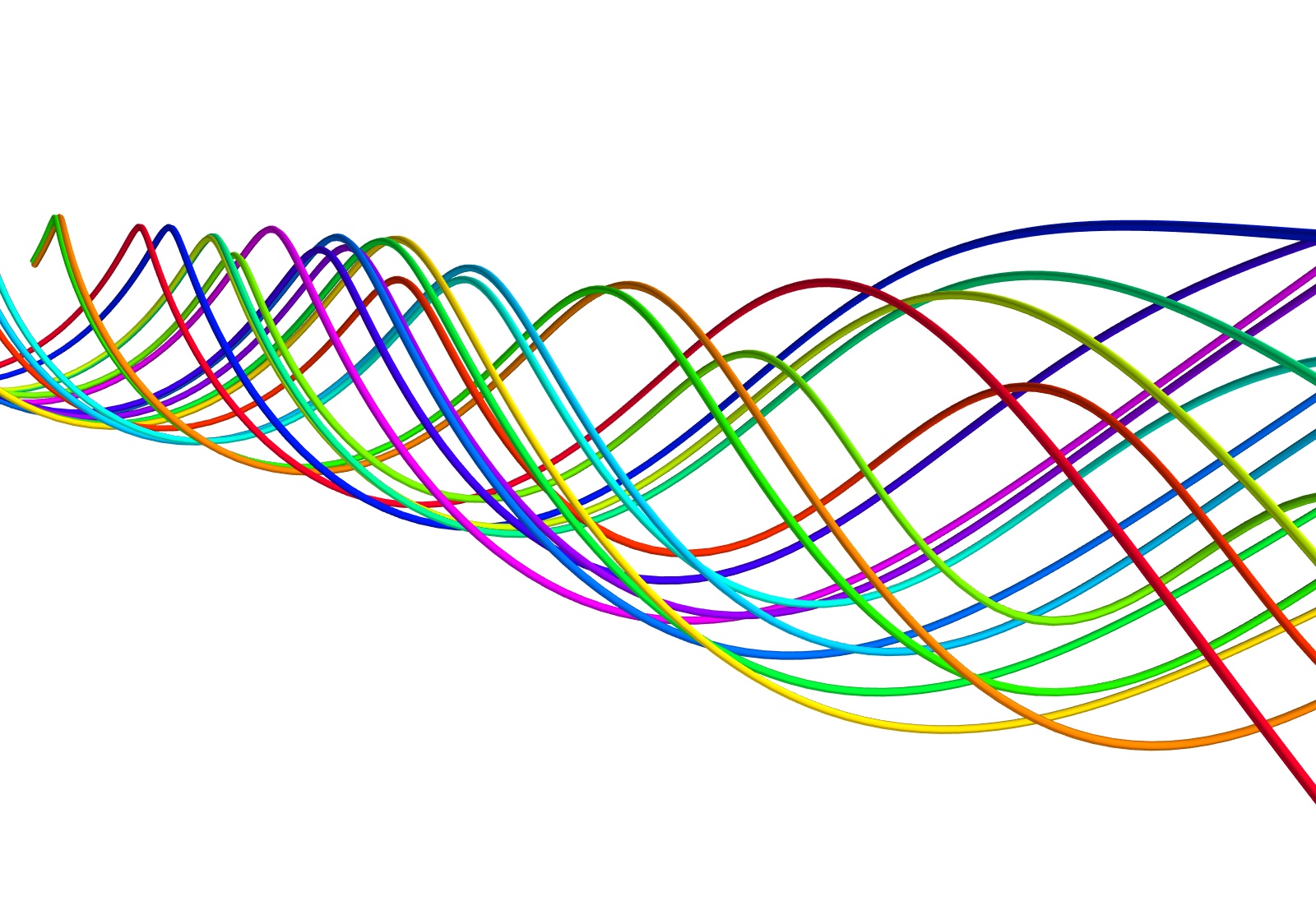}
\caption{Initial spiraling motion $\vect f(\vecs \xi)$}
\label{fig:spiraling_motion_inital}
\end{subfigure} 
\begin{subfigure}{0.49\columnwidth}
  \centering
  \includegraphics[width=\textwidth, trim={0cm 6cm 0 1cm}, clip]{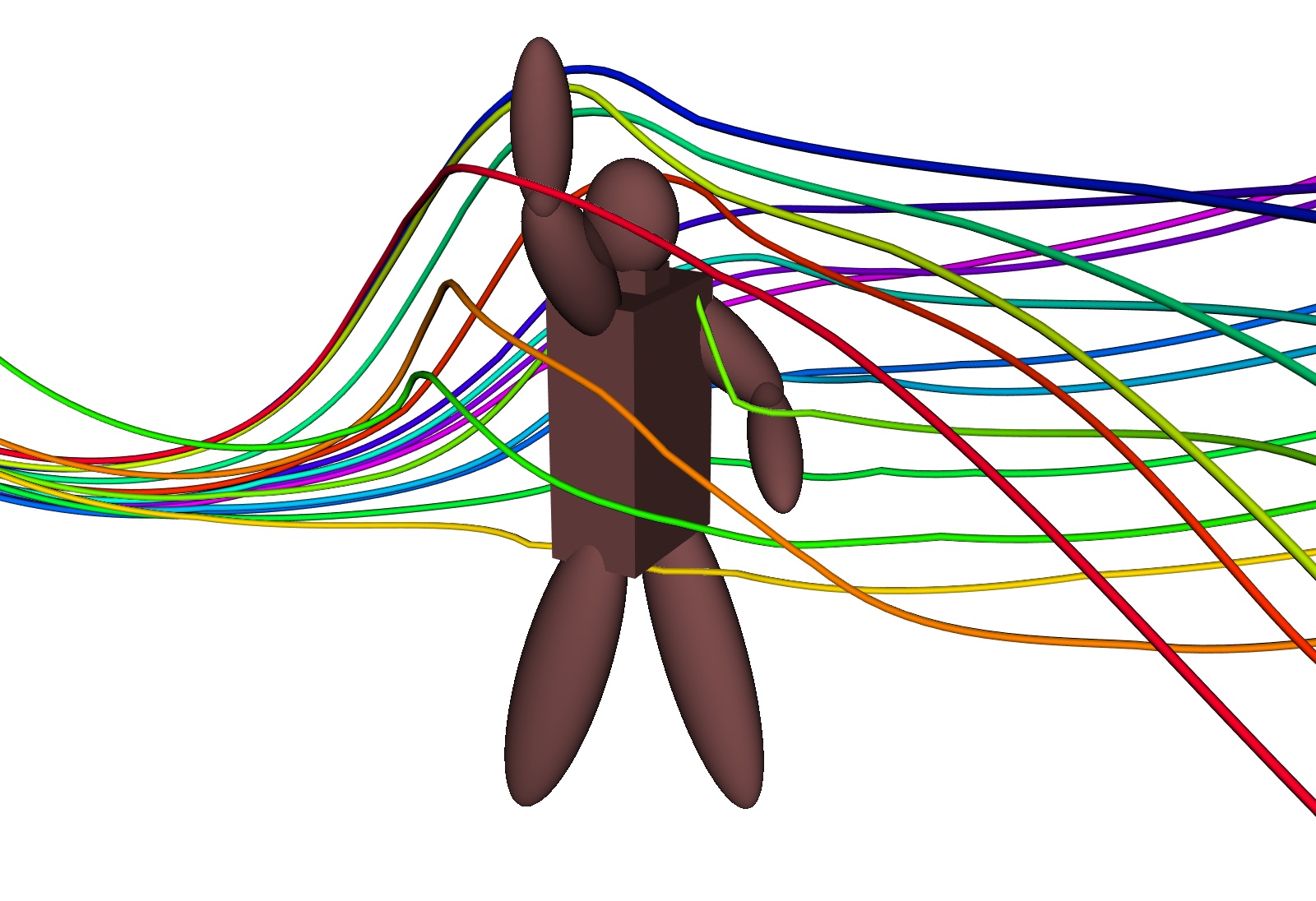}
\caption{Rotated motion around human  $\vecs \xi$
}
  \label{fig:spiraling_motion_avoidance}
\end{subfigure}%
\caption{ROAM is used to guide trajectories from 16 different initial positions and ensures that all trajectories successfully avoid the static human in simulation.}
\label{fig:spiraling_motion}
\end{figure}

\subsubsection{Qualitative Evaluation on Robot Arm} \label{sec:robot_implementation}
Experiments with the real robot were performed using the 7 DoF Panda robot by Franka Emika on a fixed base (see Fig.~\ref{fig:fanka_conveyor_avoidance}).
The scenario chosen is the automated disinfecting of a running conveyor belt, which transports various parcels. The initial dynamics $\vecs f(\vecs \xi)$ are similar to the spiral motion in \eqref{eq:spiral_motion_vf}, but the basis $\matr{B}$ and perpendicular dynamics $\vect p(\vecs \xi)$ given by:
\begin{equation}
\matr B = \begin{bmatrix}
	1 & 0  & 0 \\
 	0 & 1 & 0
 \end{bmatrix} 
 \quad \text{and} \quad 
	\matr{p}(\vecs \xi) = \begin{bmatrix} 0 &  0 & \left(\vecs \xi^a(t) - \vecs \xi\right) \end{bmatrix}^T
\end{equation}

Moreover, the attractor is dynamic and moves back and forth the conveyer belt:
\begin{equation}
	\vecs \xi^a(t) = \begin{bmatrix} 0.5 & 0.6 + 0.1 \sin \left(\frac{\pi}{10} t \right) & 0.3 \end{bmatrix}^T 
\end{equation}
As a result, the initial dynamics $\vecs f(\vecs \xi, t)$ are time-varying, too.

It is assumed that the robot has complete knowledge of the relative position and shape of the conveyor belt. The position and velocity of the parcel are determined using reflective markers (Optitrack). The analytical shape of the objects is known analytically.
Additionally, an operator is handling parcels on the conveyer belt, but the robot is not directly aware of its presence. However, since an impedance controller is employed \cite{kronander2015passive}, the robot is capable of adapting to physical interactions.

\begin{figure}[tbh]\centering
\begin{subfigure}{1.0\columnwidth}
  \centering
  \includegraphics[width=\textwidth]{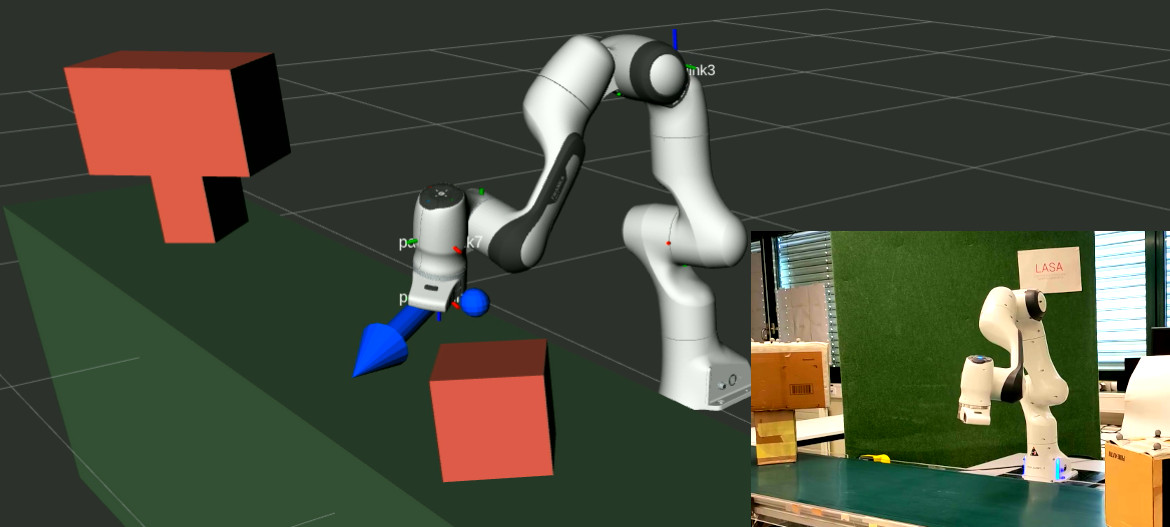}
\end{subfigure}%
\caption{The robot is aware of the two obstacles (brown shapes) as well as the conveyor belt (green block) to obtain the avoidance dynamics (blue arrow). The center of the initial dynamics (blue dot) is moving across the conveyor belt.}
\label{fig:fanka_conveyor_avoidance}
\end{figure}

The robot successfully avoids both the parcels and the conveyor belt while maintaining adherence to the initial dynamics whenever feasible. By utilizing trees-of-stars to represent the concave obstacle and positioning the reference point of the root-component on the conveyor belt, the robot effectively avoids the parcel from above (see Fig.~\ref{fig:franka_avoidance_on_conveyor_belt}). Notably, comparable methods such as MuMo \cite{huber2022avoiding} or VF-CAPF \cite{yao2022guiding} do theoretically not permit the placement of a reference point in trees of obstacles that facilitates collision avoidance in such environments.

\def\figwidth{0.164\textwidth}
\begin{figure*}[bth]\centering
\begin{subfigure}{\figwidth}
\centering
\includegraphics[width=\textwidth, trim={4cm 0cm 10cm 0cm}, clip]{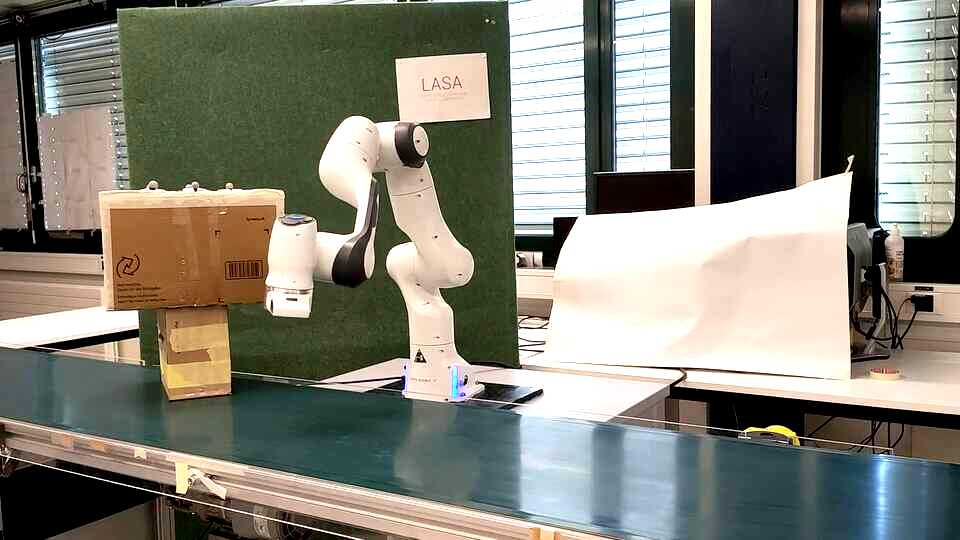}
\end{subfigure}\hfill%
\begin{subfigure}{\figwidth}
\centering
\includegraphics[width=\textwidth, trim={4cm 0cm 10cm 0cm}, clip]{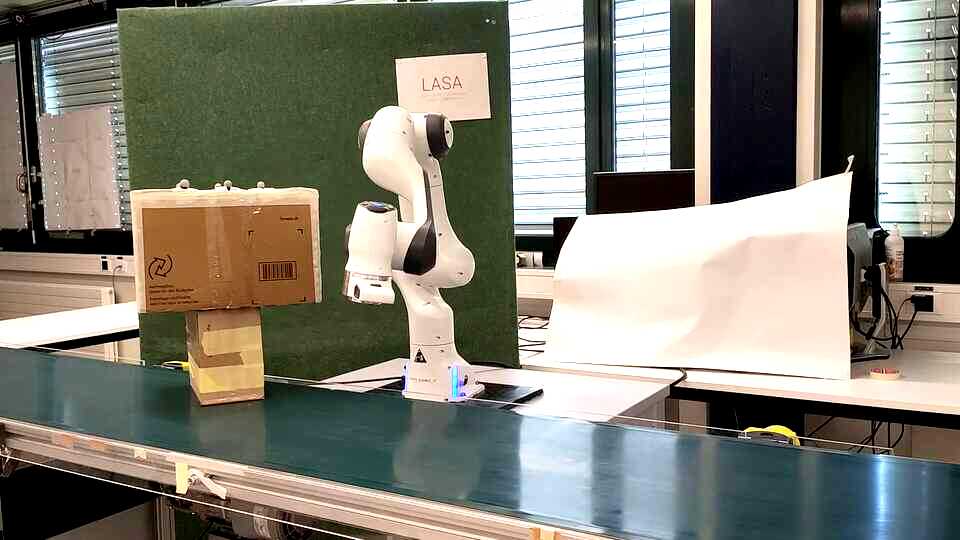}
\end{subfigure}\hfill%
\begin{subfigure}{\figwidth}
\centering
\includegraphics[width=\textwidth, trim={4cm 0cm 10cm 0cm}, clip]{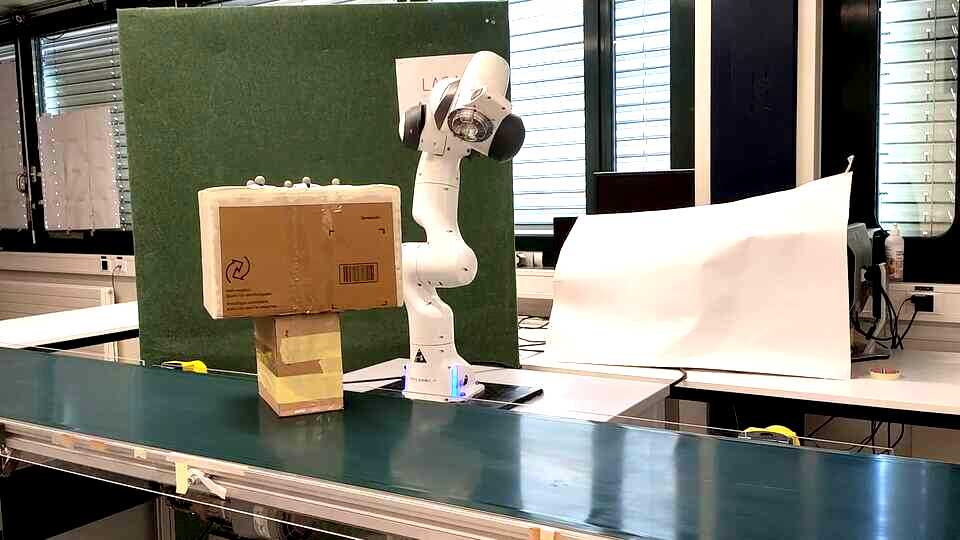}
\end{subfigure}\hfill%
\begin{subfigure}{\figwidth}
\centering
\includegraphics[width=\textwidth, trim={4cm 0cm 10cm 0cm}, clip]{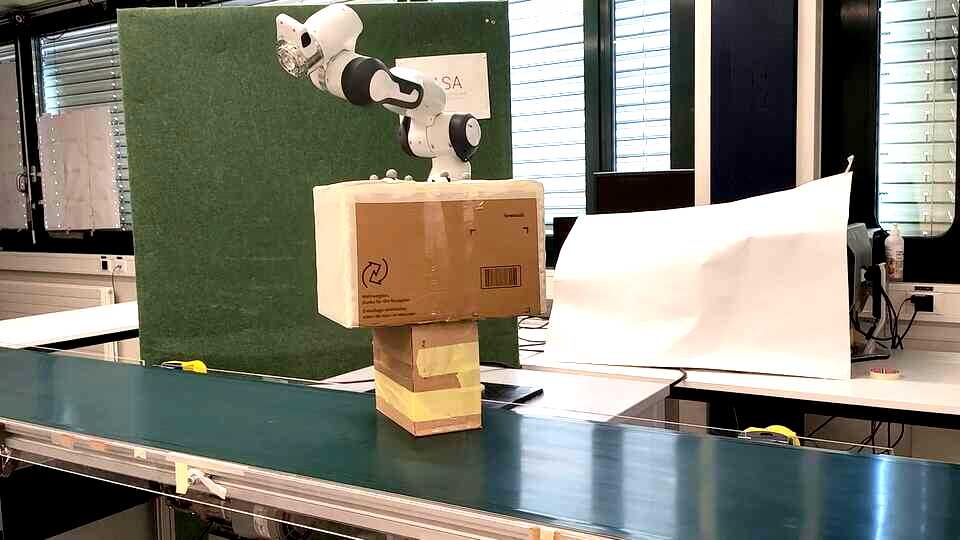}
\end{subfigure}\hfill%
\begin{subfigure}{\figwidth}
\centering
\includegraphics[width=\textwidth, trim={4cm 0cm 10cm 0cm}, clip]{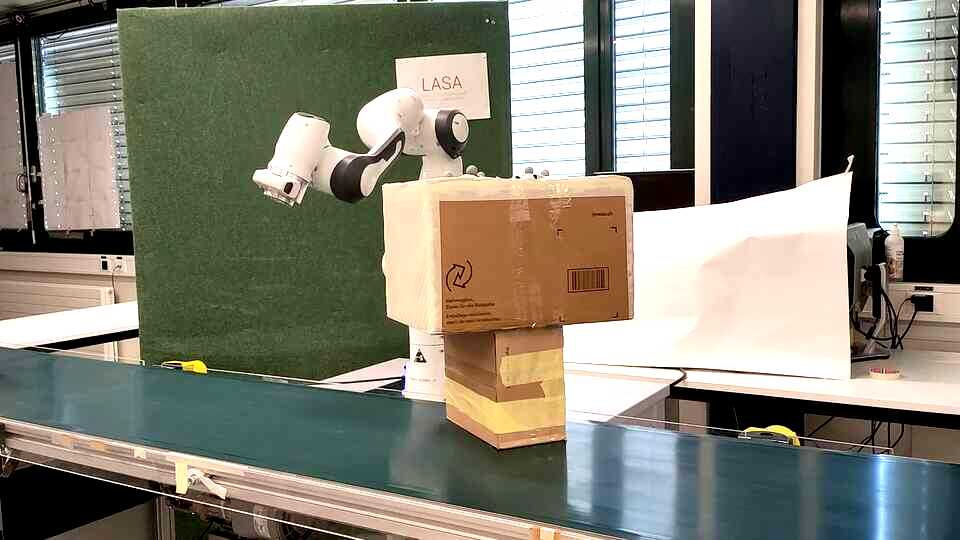}
\end{subfigure}\hfill%
\begin{subfigure}{\figwidth}
\centering
\includegraphics[width=\textwidth, trim={4cm 0cm 10cm 0cm}, clip]{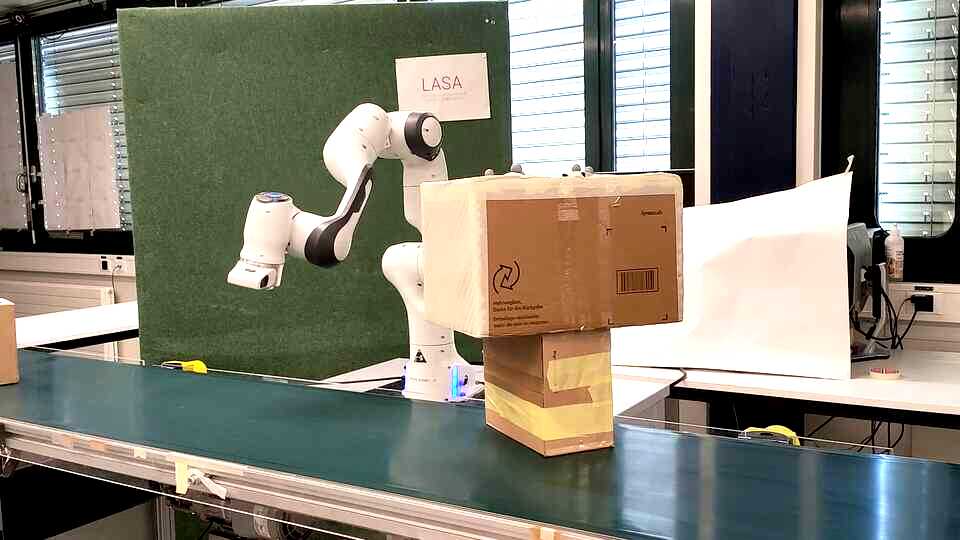}
\end{subfigure}\hfill%
\caption{The 7DoF robot arm adeptly avoids the star-shaped parcels that are being transported on the conveyor belt.}
\label{fig:franka_avoidance_on_conveyor_belt}
\end{figure*}

\section{Discussion}
The proposed rotational obstacle avoidance method (ROAM) has successfully addressed the challenge of avoiding collisions with initially nonlinear dynamics around general concave obstacles without holes. It stands out as the first state-dependent solution for trees-of-stars obstacles that is free from local minima without the need for hyperparameter tuning. Furthermore, ROAM enables obstacle avoidance while attempting to maintain nonlinear dynamics, which is a significant advantage. The algorithm has demonstrated improved convergence and motion similarity compared to the baseline methods in experimental evaluations. Moreover, its low computational cost has allowed its application to dynamic obstacle avoidance scenarios involving a robotic arm.

\subsection{Stationary Points}
ROAM introduces a new stationary point for each obstacle (or obstacle tree). However, due to the topological properties of smooth vector fields, at least one fixed point is created for each obstacle (a hole in space). These fixed points are observed to be saddle points and have no practical significance as they are unstable and the probability of reaching them is effectively zero. Additionally, any noise or perturbation in the system pushes the motion away from these unstable points. It is worth noting that while ROAM is used for dynamic scenarios, the trajectories of these saddle points should be preserved as they reflect the smoothness of the velocity. Smoothness is crucial as it ensures that even with uncertainties in perception or unexpected disturbance, there will be no discontinuity in the desired velocity command $\dot{\vec{\xi}}$. Nevertheless, the removal of saddle points could be achieved by setting the smoothness factor $q$ to 1 in Equation \eqref{eq:pulling_weight} and selecting any tangent direction for $\vec{e}(\vec{\xi})$ when $\vec{c}(\vec{\xi})$ and $\vec{\xi}^r$ are collinear.


\subsection{Trees of Stars}
ROAM relies on a star-shaped (or trees-of-stars) obstacle environment. While in certain real-world scenarios, such division can be achieved based on the rigid-body features of the surroundings (e.g., dividing a human into limbs and core or a table into plate and legs), it is often challenging to determine such divisions for more complex obstacles or in higher-dimensional spaces, such as joint space collision avoidance. Some propositions for extracting star-shaped environments exist \cite{dahlin2022creating}; however, further research is needed to extend these approaches for real-time applications.

However, there exist algorithms to simplify general shapes by approximating them as a union of overlapping spheres \cite{hubbard1996approximating}, which could serve as the basis for constructing the obstacle tree in future work. Additionally, the circular shape of these components reduces the computational complexity of tasks such as checking for component intersections and evaluating the normal direction and distance to the obstacles.

\subsection{Tangent Following on the Surface}
Each tree of stars has exactly two saddle points, which lie on the convergence direction line passing through the root of the tree. At any other surface position, the velocity is tangent to the surface. This can lead to extensive wall-following behaviors for obstacles with multiple levels. To address this, future work should focus on optimizing the selection of saddle points and allowing the rotated velocity to deviate from the obstacle's tangent plane. However, this would require a combination with high-level planning.

\subsection{Higher Dimensional Applications: Joint Limits}
During the experiments with the robotic arm, there were instances where obstacle-free motions led the arm to reach its joint limits or encounter singularities. Future research will investigate a unified framework that combines task-space collision avoidance with joint-space constraint avoidance. ROAM is well-suited for evaluating the desired motion in both representations due to its low computational cost and applicability in higher dimensions.

\subsection{Region of Influence}
Navigation functions for obstacle avoidance \cite{koditschek1990robot, paternain2017navigation} and vector field avoidance algorithms \cite{panagou2014motion, yao2022guiding} often require parameter tuning based on the obstacle distribution to ensure the absence of local minima. In contrast, the introduced method, ROAM, can handle intersecting obstacle regions, and its absence of local minima is guaranteed as long as each distance function $\Gamma_o(\vec{\xi})$ for each obstacle $o$ individually satisfies the limitations outlined in Section~\ref{sec:distance_function}. The distance function proposed in Equation \eqref{eq:distance_function}, which was used throughout this work, is globally active. However, when dealing with dense environments, the linear scaling of ROAM with the number of obstacles can result in high computational costs. To address this, the influence region can be reduced by introducing a distance threshold $d^{\text{max}} \in \mathbb{R}$, such that:
  \begin{equation}
    \lim_{\|\vecs \xi - {\vecs \xi}^b\| \rightarrow d^{\mathrm{max}}}\Gamma(\vecs {\vecs \xi}) \rightarrow \infty
    \quad \text{and} \quad
    \lim_{\|{\vecs \xi} - {\vecs \xi}^b\| \rightarrow 0}\Gamma(\vecs {\vecs \xi}) = 1
  \end{equation}
For example, the distance function can be set as follows:
  \begin{equation}
    \Gamma({\vecs \xi}) =
    \begin{cases}
      d^{\mathrm{max}} / (d^{\mathrm{max}} - \|{\vecs \xi} - {\vecs \xi}^b\|) & \text{if} \;\; \|{\vecs \xi} - {\vecs \xi}^b\| < d^{\mathrm{max}}  \\
      \infty & \text{otherwise}
    \end{cases}
    \label{eq:infinity_distance}
  \end{equation}

\subsection{Application to Robotic Systems}
The proposed algorithm assumes that the robot is a point mass. However, it can be extended to real robots by introducing a margin around the obstacles or using multiple control points, as proposed in \cite{huber2022avoiding}. It is important to note that the latter method does not guarantee full convergence in general scenarios. Alternative methods employ a full-body representation of the agent \cite{mainprice2020interior, koptev2022neural}, which involves sampling or optimization. However, they cannot guarantee finding a feasible path in a finite time. Therefore, future work should focus on extending ROAM to handle analytic avoidance scenarios where trees of stars represent both the agent and the obstacle.


\appendix
\subsection{Vector Rotations}  \label{sec:perpendicular_rotation}
The rotational obstacle avoidance method (ROAM) requires weighted summing between vectors as well as partial (sequential) rotations applied to vector fields.
Herefore, we introduce vector math concepts to simplify these operations.

\subsubsection{Perpendicular Rotation Base} \label{sec:rotation_base}
Let us consider two initial directions $\vect v_i, \vect v_o \in \mathbb{R}^N \setminus \vect 0$. They are used to construct the base vectors $\vect b_i, \vect b_o \in \{\vect b \in \mathbb{R}^N: \|\vect b\| = 1\}$:
\begin{gather}
  {\vect b}_i = \frac{{\vect v}_i }{\| {\vect v}_i \|}
  \, , \;
     {\vect b}_o = \frac{{\hat{\vect b}}_o}{\| {\hat{\vect b}}_o \|}
     \, , \; \quad \beta = \arccos \left( \normdotprod{\vect v_i}{\vect v_o} \right)
     \label{eq:vector_basis} \\
  \;\; \text{with} \quad
  {\hat{\vect b}}_o = \vect v_o - \vect v_i
  \dotprod{\vect v_i}{\vect v_o}
  \;\; , 
  \quad \forall \, \vect v_i, \vect v_o :
  \normdotprod{\vect v_i}{\vect v_o} \neq -1 \nonumber
\end{gather}
The angle $\beta$ and two vectors $\vecs b_i$, $\vecs b_o$ represent the rotation of a vector in $N$ dimensions, this is equivalent to $2 N + 1$ parameters.\footnote{
This compares to $N (N-1) / 2$ parameters required to describe rigid body rotation in $N$ dimensions. Hence, already at $N=4$ the rigid body orientation uses more parameters to describe vector rotation, and stays more efficient for higher dimensions.
}

\subsubsection{Rotating a Vector}
Rotating a vector $\vect v$ entails rotating the component which lies in the plane spanned by $\vect b_i$ and $\vect b_o$, while conserving the part orthogonal to the plane.
\begin{gather}
	{\vect v}_{r} \left(\vect v , \beta, \vect b_{\{i, o\}}\right) 
	 = p_0  \vect b_i \cos(\phi) + \vect b_o \sin(\phi)   + 
	  \vect v - \sum_{j \in \{i, o\}} \vect b_j \dotprod{\vect b_j}{\vect v} \nonumber \\
  \quad \text{with} \quad
  \phi = \phi_0 + \beta
  \;, \;\; \label{eq:vector_rotation_application}
  \tan(\phi_0) =  \frac{\dotprod{\vect b_i}{\vect v}}{\dotprod{\vect b_o}{\vect v}} \;, \\
  \; \; p_0 = \sqrt{\dotprod{\vect b_i}{\vect v}^2 + \dotprod{\vect b_o}{\vect v}^2} \nonumber
\end{gather}
Note that by changing the rotation angle $\beta$, a partial rotation or over-rotation can be applied to a vector, too. 
In this work, the basis is only explicitly stated if not clear from the context.

Moreover, as each rotation basis is defined by the two vectors $\vect b_i$ and $\vect b_o$, a rotation can be rotated, too. For example, a rotation expressed in a relative reference frame can be rotated to the global frame. This allows for the evaluation of weighted rotation sequences, as will be discussed later in this section.

\subsubsection{Weighted Rotation Summing}
Let us assume that we have $N^{\mathrm{vec}} \in \mathbb{N}_{>0}$ vector rotations which are defined with respect to a shared basis vector $\vect b_i$. The rotations have a respective second basis vector $\vect b_{o, v}$ and rotation angle $\beta_v$ with $v \in 1 .. N^{\mathrm{vec}}$. The basis and angle of the weighted average rotation is defined as:
\begin{equation}
  \hat{\vect b}_o = \matr B(\vect b_i) \left( \sum_{v=1}^{N^{\mathrm{vec}}} \matr B^T(\vect b_i) w_v \beta_v \vect b_{o, v} \right)
  \quad \text{and} \quad
   \hat\beta = \| \hat{\vect b}_o \| 
  \label{eq:single_basis_summing}
\end{equation}
where $w_v \in [0, 1]$ are the vector weights, and $\matr B(\vect b_i)$ is an orthonormal basis with respect to the vector $\vect b_i$.

We further define the rotational sum using the symbol $\hat{+}$ as:
\begin{equation}
	\hat{\vect v} = \vect v_0 \hat{+} \sum_{v=1}^{N^{\mathrm{vec}}} w_v \vect v_{v} \label{eq:rotational_summing}
\end{equation}
where the summation is performed according to \eqref{eq:single_basis_summing}, with normalization of the vector pairs $[\vect v_0, \vect v_v]$ to obtain the rotations as described in \eqref{eq:vector_basis}.



\subsubsection{Weighted Rotation Sequence}
Let us assume that we have $N^{\mathrm{rot}}$ sequential rotations such that the output vector $\vect b_{o, n}$ of the element $n$ is equal to the input vector of the next in the sequence $n+1$, i.e., $\vect b_{o, n} = \vect b_{i, n+1}$ with $n = 1 .. N^{\mathrm{rot}} -1$. To evaluate the weighted sequence each basis is adapted with respect to all the parent rotations as:
\begin{equation}
\vect b_{\{i, o\}, c} \gets  \vect v_r((w_n - 1) \beta_n, \, \vect b_{\{i, o\}, c}) \quad n = 1 .. N^{\mathrm{rot}} - 1, \; c > n \label{eq:sequence_invertion}
\end{equation}
The full weighted rotation of the sequence is obtained using the vector rotation as defined in \eqref{eq:vector_basis} with input vector $\vect b_{i, 1}$ and output $\vect v_{o, n} = \vect v_r((1 - w_n) \beta_n, \, \vect b_{o, n})$ with $n=N^{\mathrm{rot}}$.



  

\subsubsection{Direction Tree Evaluation}
Let us define a direction tree with nodes $n = 1 .. N^{\mathrm{node}}$. There exists a single (unique) root node $r$, which has no parent. All other nodes $n$ have single parent $p(n)$, but a node can have multile children, given by the set $\mathcal{C}_n$. The set of all offspring, i.e., including the children of children etc., is referred to as $\mathcal{C}_n^{\mathrm{tot}}$.
Each node has a level $l$, which indicates the discrete distance to the root.
The value of each node corresponds to a unit vector $\vect v_i$, and hence the vector rotation between the parent and the node can be evaluated using \eqref{eq:vector_basis}. 
For given weights $w_n$, the weighted direction is evaluated as described in Algorithm~\ref{alg:rotation_summing_tree}. The rotation tree is interpreted as a combination of weighted summations and sequences, as described earlier in this section.

\begin{algorithm}[ht] 
  \caption{Rotation Tree Summing} \label{alg:rotation_summing_tree}
  \begin{algorithmic}[1]
    \renewcommand{\algorithmicrequire}{\textbf{Input:}}
    \renewcommand{\algorithmicensure}{\textbf{Output:}}
    \REQUIRE $w_n$, $\vect v_n$ for $n = 1 .. N^{\mathrm{node}}$
    \ENSURE Averaged vector-rotation $\vect v^*$
    \FOR [Reverse iteration over levels] {$l = L^{\mathrm{max}} \, .. \, 1$}
    \FOR [For all nodes of the same level] {$n: l_n = l$}
    \STATE $w_{p(n)} \gets w_{p(n)} + w_{n}$ \COMMENT{Update cumulative weight}
    \ENDFOR
    \ENDFOR
	\STATE $\vect v ^*_1 = \vect v_{r}$ \COMMENT{Initialize base of root $r$}
    \FOR[Iteration over levels] {$l = 2 \, .. \, L^{\mathrm{max}}$}
	\STATE $\vect v^*_l = \vect v^*_{l-1} \hat + \sum_{c}^{l_c == l} \vect v_c$ \COMMENT{Weighted average from \eqref{eq:rotational_summing}}
	\STATE $\vect v_{r}(\cdot), \beta \gets \vect v_{l-1}^*, \vect v_{l}^*$ \COMMENT{Rotation from vectors \eqref{eq:vector_basis}}
	\FOR [For all offsprings of $l$] {$c \in \mathcal{C}_l^{\mathrm{tot}}$}
	\STATE $\vect v_{r, c}(\cdot), \beta_c \gets \vect v^*_l, \vect v_c$ \COMMENT{Rotation from vectors \eqref{eq:vector_basis}}
	\STATE $\vect b_{j, c} \gets  \vect v_{r, c}(- \beta_c, \, \vect b_{j, c})$ \COMMENT{Sequence inversion \eqref{eq:sequence_invertion}}
    \STATE $\vect b_{j, c} \gets \vect v_r(\beta, \vect b_{j, c})$ \COMMENT{Apply rotation to $\vect v^*_l$ as \eqref{eq:vector_rotation_application}}
    \ENDFOR
    \ENDFOR
	\STATE $\vect v^* = \vect v^*_{L^{\mathrm{max}}}$
  \end{algorithmic}
\end{algorithm}



\begin{lemma} \label{lemma:direction_tree}
  The vector rotation as described in Algorithm~\ref{alg:rotation_summing_tree} ensures a smooth vector summing for weights $w_n \in [0, 1]$ such that $\sum_{n = 1}^{N^{\mathrm{node}}} w_n = 1$, such that we reach converge to a vector when its corresponding weight is approaching one, i.e., $\lim_{w_i \rightarrow 1}\vect v^* = \vect v_i$. 
\end{lemma}

\begin{proof}
 When a specific direction-node $n$ has weight $w_n = 1$, it follows from Algorithm~\ref{alg:rotation_summing_tree} all the child-nodes weight one and zero otherwise:
\begin{equation}
 w_c = \begin{cases}
 1  & \text{if} \;\; c \in \mathcal{C}_n \\
 0  & \text{otherwise}
 \end{cases}
\end{equation}
Hence, we have a rotation sequence of the form \eqref{eq:sequence_invertion}. Due to the sum of the weights of all sequence elements being equal to one, the summed weight equals $\vect b_{o,e}$  where $e$ is the index of the ending element.
Note successive vectors cannot be anti-parallel, as denoted in \eqref{eq:vector_basis}.
\end{proof}

Note, that the weighted evaluation of the direction tree can be reduced to a single vector $\vect v^*$, or the whole sequence can be kept by storing all $\vect v^*_{(\cdot)}$. If all node-parent pair had an angle $\beta < \pi$, then this will be the case for the resulting sequence, too.






\renewcommand*{\bibfont}{\footnotesize}
\printbibliography

\end{document}